\documentclass[final,3p,twocolumn]{elsarticle}
\usepackage{geometry}
\usepackage{mathrsfs,enumerate,amssymb,amsmath,bm,amsthm,color,lineno}
\usepackage[all]{xy}
\usepackage{latexsym}
\usepackage{amscd,verbatim,cuted}
\usepackage{graphicx}
\usepackage[colorlinks=true]{hyperref}
\usepackage{amssymb}
\usepackage{placeins}
\usepackage{float}
\usepackage{varioref}
\usepackage{booktabs}
\usepackage{pifont}
\usepackage{subfigure}

\theoremstyle{plain}
\newtheorem{theorem}{Theorem}[section]

\newtheorem{lemma}{Lemma}[section]
\newtheorem{proposition}{Proposition}[section]
\theoremstyle{definition}
\newtheorem{definition}{Definition}[section]
\newtheorem{example}{Example}[section]

\newtheorem{remark}{Remark}

\numberwithin{equation}{section}

\begin{document}

\begin{frontmatter}
\title{Picture Fuzzy Interactional Aggregation Operators via Strict Triangular Norms
and Applications to Multi-Criteria Decision Making\tnoteref{mytitlenote}}
\tnotetext[mytitlenote]{This work was supported by the National Natural Science Foundation of China
(No. 11601449), and the Key Natural Science Foundation of Universities in Guangdong Province (No. 2019KZDXM027).}

%% Group authors per affiliation:

\author[a1,a2]{Xinxing Wu}
\address[a1]{School of Sciences, Southwest Petroleum University, Chengdu, Sichuan 610500, China}
\address[a2]{Zhuhai College of Jilin University, Zhuhai, Guangdong 519041, China}
\ead{wuxinxing5201314@163.com}

\author[a1]{Zhiyi Zhu}
\ead{zhuzhiyi2019@163.com}

\author[a3]{G\"{u}l Deniz \c{C}ayl{\i}}
\address[a3]{Department of Mathematics, Faculty of Science, Karadeniz Technical University, 61080 Trabzon, Turkey}
\ead{guldeniz.cayli@ktu.edu.tr}

\author[a4]{Peide Liu}
\address[a4]{School of Management Science and Engineering, Shandong University
of Finance and Economics, Jinan Shandong 250014, China}
\ead{peide.liu@gmail.com}

\author[a5]{Xu Zhang\corref{mycorrespondingauthor}}
\cortext[mycorrespondingauthor]{Corresponding author}
\address[a5]{Department of Mathematics, Shandong University, Weihai, Shandong 264209, China}
\ead{xu$\_$zhang$\_$sdu@mail.sdu.edu.cn}

\author[a6]{Zao-li Yang}
\address[a6]{School of Economics and Management, Beijing University of Technology, Beijing 100124, China}
\ead{yangprofessor@163.com}

\begin{abstract}
The picture fuzzy set, characterized by three membership degrees, is a helpful
tool for multi-criteria decision making (MCDM). This paper investigates the
structure of the closed operational laws in the picture fuzzy numbers (PFNs)
and proposes efficient picture fuzzy MCDM methods. We first introduce an
admissible order for PFNs and prove that all PFNs form a complete lattice
under this order. Then, we give some specific examples to show the
non-closeness of some existing picture fuzzy aggregation operators. To
ensure the closeness of the operational laws in PFNs, we construct a new class
of picture fuzzy operators based on strict triangular norms, which consider
the interaction between the positive degrees (negative degrees) and the
neutral degrees. Based on these new operators, we obtain the picture fuzzy
interactional weighted average (PFIWA) operator and the picture fuzzy
interactional weighted geometric (PFIWG) operator. They are proved to be
monotonous, idempotent, bounded, shift-invariant, and homogeneous. We also
establish a novel MCDM method under the picture fuzzy environment applying
PFIWA and PFIWG operators. Furthermore, we present an illustrative example
for a clear understanding of our method. We also give the comparative
analysis among the operators induced by six classes of famous triangular
norms.
\end{abstract}
\begin{keyword}
Picture fuzzy set (PFS); Aggregation operator; Linear order; Triangular norm; Multi-criteria decision making (MCDM).
\end{keyword}
\end{frontmatter}

\section{Introduction}
For the voting process, the voters present multiple selection opinions as follows:
`yes', `no', `abstain', and `refusal'. This case cannot be exactly
described by fuzzy sets~\cite{Za1965}, intuitionistic fuzzy sets (IFSs)~\cite{Ata1999},
and hesitant fuzzy sets~\cite{T2010}. To deal with such issues,
being a direct extension of fuzzy sets and IFSs, the concept of picture fuzzy sets
(PFSs) as a modern fundamental concept about computational intelligence problems
was introduced by Cuong and Kreinovich~\cite{CK2014,CK2013},
which are characterized by three membership functions: a positive,
a neutral, and a negative, meeting the requirement that the sum
of the positive, neutral, and negative membership degree of every
point is between $0$ and $1$. In particular, every triple of positive, neutral,
and negative membership degrees
for PFSs was called a picture fuzzy number (PFN) by Jana et al.~\cite{JSPY2019}.
Then, some basic operations for PFSs, including conjunctions, disjunctions, negations, and
implications, were defined by Cuong et al.~\cite{CK2014,CP2015}. They also
described the picture fuzzy distance, the picture fuzzy relation, and Zadeh's
extension principle in PFSs. In analogy to IFSs, considering the uncertainty
of three membership degrees for PFSs, Cuong~\cite{CK2014} proposed the
notion of interval-valued PFSs (IvPFSs).

To solve some practical MCDM problems
described by PFSs, the picture fuzzy aggregation operator,
which monotonously transforms a collection of picture fuzzy inputs into
a single PFN, is an effective tool. In addition, picture fuzzy
similarity measure, the method of picture fuzzy TOPSIS,
the method of picture fuzzy AHP, and the method of picture fuzzy TODIM-VIKOR were also applied for picture fuzzy
MCDM problems \cite{AMAK2019,AK2020,GS2021,GDMA2021,LZ2020,SZZA2019,SG2021,Th2020}. In order to get
logical and influential aggregation results, we must work out the following
basic problems: (1) Construct an admissible order for ranking all PFNs
(according to our knowledge, until now there is no general methodology that ranks any two PFNs);
(2) Define a `good' aggregation operator, meeting the monotonicity
and closeness. The
monotonicity ensures that the higher the score of each criterion, the better
the aggregation result. The closeness ensures that the aggregation output is
still a PFN; and so, our evaluation criteria are under the unified framework.

Up to the present, many scholars have devoted themselves
to the investigation of aggregation operators under picture fuzzy
environment (see \cite{AMAK2019,Ga2017,JSPY2019,JJGGW2019,KAA2019,
PWC2021,TPZZW2019,WZTT2017,Wang2015,Wei2017,Wei2018a,ZTLJZ2020})
and their applications (see~\cite{LHX2020,WPW2018}).
For example, Wei et al.~\cite{Wei2017,Wei2018a}
developed some picture fuzzy aggregation operators, including the
weighted average operator, the weighted geometric operator, and
the Hamacher (Heronian mean) aggregation operator for PFSs.
They used these operators to clarify the picture fuzzy MCDM.
Through triangular norms and triangular conorms,
Garg~\cite{Ga2017} and Ashraf et al.~\cite{AMAK2019} developed the
weighted average, the
ordered weighted average, and the hybrid average aggregation
operators to deal with the MCDM for PFSs.
Meanwhile, Jana et al.~\cite{JSPY2019} provided various picture
fuzzy aggregation operators using Dombi triangular norms and triangular
conorms with a view to concerning the picture fuzzy MCDM.
Moreover, Khan et al.~\cite{KAA2019} constructed the picture fuzzy Einstein (ordered) weighted
operators by means of Einstein sum and product.
Considering the relationship among the criterions in MCDM problems, some interactional
aggregation operators were used to model these possible relationships
(\cite{WZTT2017,Wang2015,AA2020-1,LLCFL2021,MKA2018,QQSSJ2020,WLG2018}).
Wei et al.~\cite{WLG2018} introduced
the Heronian mean aggregation operators
in the picture fuzzy situations where each attribute always has a relationship with others.
Furthermore, Lin et al.~\cite{LLCFL2021} proposed the interactional
operational laws for PFNs, which get hold of the interaction between the
positive degrees and the neutral degrees (negative degrees) for any two
PFNs. Ate\c{s} and Akay~\cite{AA2020-1} extended the Bonferroni mean to PFNs.
They also introduced the Bonferroni mean, the normalized weighted
Bonferroni mean, and the ordered weighted Bonferroni mean for PFSs. However, all above
aggregation operators have a common shortcoming (see Section~\ref{S-3}):
They do not need to satisfy the
fact that the sum of the three degrees cannot exceed one.
Thus, such
operators are not closed in PFNs. This trouble mainly arises from the
three-dimensional degree structure of PFNs. The aggregation outcomes can be
larger than one whenever two degrees calculated by triangular conorms exist.
Since there is no linear order for ranking all PFNs, all {\it ordered} aggregation
operators are of little or no value for decision-making.

Inspired by the above results on picture fuzzy aggregation operators,
this paper is devoted to constructing the closed operational laws in PFNs and
establishing effective picture fuzzy MCDM methods. In Sections~\ref{S-2}
and \ref{S-3}, by introducing a total order $\preceq _{_{\mathrm{W}}}$ for
PFNs, we prove that it is an admissible order. We also demonstrate that all
PFNs form a complete lattice under this order.
In Section~\ref{S-4}, we consider some examples to explain the non-closeness of
operational laws in \cite{AMAK2019,AA2020-1,Ga2017,JSPY2019,KAA2019,LLCFL2021,
WWGW2019,Wei2017,Wei2018a,XSWWZLX2019}. Intending to construct the closed
operational laws in PFNs, we introduce in Section~\ref{S-5} a new kind of
picture fuzzy operators via strict triangular norms. These operators
consider the interaction between the positive degrees (negative degrees) and
the neutral degrees for any two PFNs. Note that the positive
degrees and the negative degrees are more significant than the neutral
degrees for evaluation and decision-making. Hence, we only consider the
interactional operations that occur in the position of the neutral degrees.
Meanwhile, we prove that our operators are closed in PFNs and obtain their
basic properties. In Section~\ref{S-6}, by using new picture fuzzy
operators, we obtain the picture fuzzy interactional weighted average
operator $\mathrm{PFIWA}_{T,\omega }$ and the picture fuzzy interactional
weighted geometric operator $\mathrm{PFIWG}_{T,\omega }$ induced by a
triangular norm $T$. In addition, we show that they have the monotonicity,
idempotency, boundedness, shift-invariance, and homogeneity. In Section~\ref%
{S-7}, through these operators, a novel MCDM method under the picture fuzzy
environment is developed. Finally, we present an illustrative example
involving the investment options of a commercial company that shows the
decision procedure of our MCDM method. We also describe the comparative
analysis among the operators induced by some famous triangular norms,
such as the algebraic product $T_{\mathbf{P}}$, Schweizer-Sklar t-norm $%
T_{\gamma }^{\mathbf{SS}}$, Hamacher t-norm $T_{\gamma }^{\mathbf{H}}$,
Frank t-norm $T_{\gamma }^{\mathbf{F}}$, Dombi t-norm $T_{\gamma }^{\mathbf{D%
}}$, and Acz\'{e}l-Alsina t-norm $T_{\gamma }^{\mathbf{AA}}$.

\section{Preliminaries}\label{S-2}
\subsection{Order and lattice}

\begin{definition}[{\textrm{\protect\cite[Definition~1.1.3]{HWW2016}}}]
If a binary relation $\preceq$ on a set $X$ satisfies the properties:
\begin{enumerate}[(1)]
\item  $a\preceq a$ (reflexivity);

\item  From $a\preceq b$ and $b\preceq a$, it follows $a=b$ (antisymmetry);

\item  From $a\preceq b$ and $b\preceq c$, it follows $a\preceq c$ (transitivity);
\end{enumerate}
then, it is called a \textit{partial order} on $X$.
The set $X$ with this partial order is denoted by $(X, \preceq)$, which is called a \textit{poset}.
\end{definition}

For a poset $(L, \preceq)$, an \textit{upper bound} of $A\subset L$ is
an element $u\in L$ satisfying $a\preceq u$ for any $a\in A$. If an
upper bound $u$ of $A$ is not larger than any other upper bound of $A$,
then it is said to be the \textit{smallest upper bound} or
\textit{supremum} of $A$, denoted by $\bigvee A$ or $\sup A$. Similarly, the \textit{greatest
lower bound} or \textit{infimum} of a subset $A$ of $L$ can be defined, denoted by
$\bigwedge A$ or $\inf A$. Given any pair of elements, denote by
\begin{equation}
\label{sup-inf-operation}
a\vee b=\sup\{a, b\} \text{ and } a\wedge b=\inf\{a, b\}.
\end{equation}

\begin{definition}[\textrm{\protect\cite{Bir1967}}]
\label{Lattice-Def-1}
A \textit{lattice} is a poset satisfying that every
pair of two elements have the infimum as well as the supremum.
\end{definition}

\subsection{Triangular norm}

\begin{definition}[{\textrm{\protect\cite[Definition~1.1]{KMP2000},\cite{SS1961}}}]
A binary function $T:[0, 1] ^{2}\rightarrow [0, 1] $ is called a \textit{triangular norm} (\textit{t-norm})
if, for any $z_1$, $z_2$, $z_3\in [0, 1] $, the following are satisfied:
\begin{enumerate}
\item[(T1)] $T(z_1, z_2)=T(z_2, z_1)$ (commutativity);

\item[(T2)] $T(z_1, T(z_2, z_3))=T(T(z_1, z_2), z_3)$ (associativity);

\item[(T3)] $T(z_1, z_3)\leq T(z_2, z_3)$ for $z_1\leq z_2$ (monotonicity);

\item[(T4)] $T(z_1, 1)=z_1$ (neutrality).
\end{enumerate}

A binary function $S:[0,1]^{2}\rightarrow [0,1]$ is called a
\textit{t-conorm} if, for any $z_1$, $z_2$, $z_3\in \lbrack 0,1]$, it satisfies (Tl)--(T3) and
(S4) described by
\begin{enumerate}
\item[(S4)] $S(z_1, 0)=z_1$ (neutrality).
\end{enumerate}
\end{definition}

The associativity of t-norms $T: [0, 1]^{2}\rightarrow [0, 1]$ helps us to
extend them to $n$-ary functions $T^{(n)}: [0,
1]^{n}\rightarrow [0, 1]$ in the following way (see~\cite[Remark~1.10]%
{KMP2000} and \cite[Definition~3.23]{GMMP2009})
\begin{equation*}
T^{(n)}(z_1, \ldots, z_{n-1}, z_n)\triangleq T(T^{(n-1)}(z_1,
\ldots, z_{n-1}), z_{n}).
\end{equation*}
If, in particular, we have $z_1= z_2=\cdots =z_n=z$, we shall briefly write
$z_{T}^{(n)}=T^{(n)}(z, z, \ldots, z).$
Finally we put, by convention, $z_{T}^{(0)}=1$ and $z_{T}^{(1)}=z$ for any $z\in [0, 1]$.

\begin{definition}[{\textrm{\protect\cite[Definition~3.2]{KMP2000}}}]
Let $f:[u_1, u_2]\rightarrow [v_1, v_2]$ be a monotone
function, where $[u_1, u_2]$ and $[v_1, v_2]$ are two closed subintervals of
$[-\infty, +\infty]$. The \textit{pseudo-inverse} $f^{(-1)}: [v_1, v_2]\rightarrow
[u_1, u_2]$ of $f$ is defined by
\begin{equation*}
f^{(-1)}(v)=\sup\{u\in [u_1, u_2]\mid (f(u)-v)(f(u_2)-f(u_1))<0\}.
\end{equation*}
\end{definition}

\begin{definition}[{\textrm{\protect\cite[Definition~3.25]{KMP2000}}}]
A strictly decreasing function $\tau: [0,1]\rightarrow \lbrack 0,+\infty ]$
is called an \textit{additive generator} (AG) of a t-norm $T$ if it is
right-continuous in $0$, $\tau(1)=0$, and, for any $(u,v)\in \lbrack 0,1]^{2}$,
we have $\tau(u)+\tau(v)\in \mathrm{Ran}(\tau)\cup \lbrack \tau(0),+\infty ]$
and
$T(u,v)=\tau^{(-1)}(\tau(u)+\tau(v)).$
\end{definition}

\begin{definition}
[{\textrm{\protect\cite[Definitions~2.9 and 2.13]{KMP2000}}}]
A t-norm $T$ is
\begin{enumerate}[(i)]
\item \textit{strictly monotone}, if
$T(u, v_1)<T(u, v_2)$ whenever $u>0$ and $v_1<v_2$.

\item \textit{Archimedean} if, for any $(u, v)\in (0, 1)^{2}$,
there exists some $n\in \mathbb{N}$ such that $u_{T}^{(n)}<v$.

\item \textit{strict}, if it is strictly monotone and continuous.
\end{enumerate}
\end{definition}

\begin{remark}
\label{Strict-->Archi}
By \cite[Proposition~2.15]{KMP2000}, it follows that every strict
t-norm is Archimedean.
\end{remark}

\begin{lemma}[{\textrm{\protect\cite{KMP2000}}}]
\label{Concinuous-Archi-Representation}
For a mapping $T: [0,
1]^{2}\rightarrow [0, 1]$, the following are equivalent:
\begin{enumerate}[{\rm (i)}]
\item $T$ is a strict t-norm;

\item $T$ has a continuous additive generator $\tau$ with $\tau(0)=+\infty$.
\end{enumerate}
\end{lemma}

\subsection{Picture fuzzy set}

\begin{definition}[{\textrm{\protect\cite[Definition 3.1]{CK2014}}}]
Let $X$ be the universe of discourse. A \textit{picture fuzzy set} (PFS) $P$
in $X$ is defined as an object with the following form
\begin{equation}  \label{eq-IFS-1}
P=\left\{\langle x, \mu_{P}(x), \eta_{P}(x), \nu_{P}(x)\rangle\mid x\in
X\right\},
\end{equation}
where $\mu_{P}(x)$, $\eta_{P}(x)$, $\nu_{P}(x)\in [0, 1]$, $%
\mu_{P}(x)+\eta_{P}(x)+\nu_{P}(x)\leq 1$, and $\pi_{P}(x)=1-(\mu_{P}(x)+%
\eta_{P}(x)+\nu_{P}(x))$ for any $x\in X$. $\mu_{P}(x)$, $\eta_{P}(x)$, and $%
\nu_{P}(x)$ denote the \textit{degree of positive membership}, \textit{%
neutral membership}, and \textit{negative membership} of $x$ in $P$,
respectively. $\pi_{P}(x)$ is the \textit{degree of refusal membership} of $%
x $ in $P$. In particular, the triplet $\langle \mu, \eta,
\nu\rangle$ satisfying that $\mu$, $\eta$, $\nu\in [0, 1]$ and
$\mu+\eta+\nu\leq 1$ is called a \textit{picture
fuzzy number} (PFN). %%For convenience, a PFN $\alpha$ is denoted by $%
%%\alpha=\langle \mu_{\alpha}, \eta_{\alpha}, \nu_{\alpha}\rangle$.
\end{definition}

Let $\mathscr{P}$ denote the set of all PFNs, i.e., $\mathscr{P}%
=\{\langle\mu, \eta, \nu\rangle \mid \mu, \eta, \nu\in [0, 1] \text{ and }
\mu+\eta+\nu\leq 1\}$. If $\eta_{P}(x)=0$, then the PFS reduces to
the Atanassov's IFS. Thus, the Atanassov's IFSs are special form of PFSs.

The following inclusion relation was introduced by Cuong and Kreinovich \cite%
{CK2014}.

\begin{definition}[\textrm{\protect\cite{CK2014}}]
\label{sub-order}
For two PFNs $\alpha$ and $\beta$, $\alpha\subseteq \beta$
if and only if $\mu_{\alpha}\leq \mu_{\beta}$, $\eta_{\alpha}\leq
\eta_{\beta}$, and $\nu_{\alpha}\geq \nu_{\beta}$.
\end{definition}

It can be verified that `$\subseteq$' is a partial order on $\mathscr{P}$.
However, the poset $(\mathscr{P}, \subseteq)$ does not form a lattice, since
the two PFNs $\langle 1, 0, 0\rangle$ and $\langle 0, 1, 0\rangle$ does not
have a supremum in $\mathscr{P}$.

With a view to comparing two PFNs, Wang et al.~\cite{WZTT2017} introduced the
following comparison laws.

\begin{definition}[{\textrm{\protect\cite[Definition 3.3]{WZTT2017},
\cite[Definition~4.2.2]{Wang2015}}}]
\label{WZTT-Order} Given a PFN $\beta=\langle \mu
_{\beta},\eta _{\beta},\nu _{\beta}\rangle $, let the \textit{score
function} $S(\beta)$ and the \textit{accuracy function} $H(\beta)$ of $%
\beta$ be defined by $S(\beta)=\mu _{\beta}-\nu _{\beta}$ and $%
H(\beta)=\mu _{\beta}+\eta _{\beta}+\nu _{\beta}$. In this case,
for $\beta$, $\gamma \in \mathscr{P}$,
\begin{itemize}
\item if $S(\beta)<S(\gamma)$, then %$\alpha$ is smaller than $\beta$, denoted by
$\beta\prec \gamma$;

\item if $S(\beta)=S(\gamma)$, then
\begin{itemize}
\item if $H(\beta)=H(\gamma)$, then % $\alpha$ is equivalent to $\beta$, denoted by
$\beta\sim \gamma$;

\item if $H(\beta)<H(\gamma)$, then $\beta \prec \gamma$.
\end{itemize}
\end{itemize}
If $\beta \prec \gamma$ or $\beta \sim \gamma$, we denote it by $\beta
\preceq \gamma$.
\end{definition}

Notice that the order $\preceq $ given by Definition~\ref{WZTT-Order} is not
a linear order on $\mathscr{P}$. Moreover, it is not a partial order on
$\mathscr{P}$ since the antisymmetry property is violated. We observe that
the degrees of the positive membership and the negative membership are more
important than the degree of the neutral membership for PFNs. Thus, the
larger (as the linear order of Atanassov's intuitionistic fuzzy numbers in
\cite[Definition 3.1]{Xu2007}) the degrees of the positive membership and
the negative membership, the better the PFN. In the case of that the degrees
of the positive membership and the negative membership are respectively equal,
the higher the degree of the neutral membership, the better the PFN
since more comprehensive information is obtained. Based on this idea, we
introduce the following linear order for PFNs.

\begin{definition}
\label{Order-Wu} Given a PFN $\beta=\langle \mu _{\beta
},\eta _{\beta},\nu _{\beta}\rangle $, let the \textit{score function} $%
S(\beta),$ the \textit{first accuracy function} $H_{1}(\beta),$ and the
\textit{second accuracy function} $H_{2}(\beta)$ of $\beta$ be defined
by $S(\beta)=\mu _{\beta}-\nu _{\beta},$ $H_{1}(\beta)=\mu _{\beta
}+\nu _{\beta},$ and $H_{2}(\beta)=\mu _{\beta}+\eta _{\beta}+\nu
_{\beta}$, respectively. In this case, for $\beta$, $\gamma\in \mathscr{P}$,
\begin{itemize}
\item if $S(\beta)<S(\gamma)$, then %%$\alpha$ is smaller than $\beta$, denoted by
$\beta\prec_{_{\mathrm{W}}} \gamma$;

\item if $S(\beta)=S(\gamma)$, then

\begin{itemize}
\item if $H_{1}(\beta)<H_{1}(\gamma)$, then $\beta \prec _{_{\mathrm{W}%
}}\gamma$;

\item if $H_1(\beta)=H_1(\gamma)$, then

\begin{itemize}
\item if $H_{2}(\beta)=H_{2}(\gamma)$, then %%$\alpha $ is equivalent to $ \beta ,$ denoted by
$\beta =\gamma$;

\item if $H_{2}(\beta)<H_{2}(\gamma)$, then $\beta \prec _{_{\mathrm{W}%
}}\gamma$.
\end{itemize}
\end{itemize}
\end{itemize}
If $\alpha \prec _{_{\mathrm{W}}}\beta $ or $\alpha =\beta $, we denote it
by $\alpha \preceq _{_{\mathrm{W}}}\beta $.
\end{definition}

\section{Some basic properties of linear order $\preceq_{_{\mathrm{W}}}$}
\label{S-3}

\begin{example}
Consider $\beta=\langle0.2,0.2,0.1\rangle$ and $\gamma=%
\langle 0.3, 0, 0.2\rangle$. (1) If we use the order $\preceq $ defined
in Definition~\ref{WZTT-Order}, then $S(\beta)=S(\gamma)=0.1$ and
$H(\beta)=H(\gamma)=0.5$, and thus $\beta\sim \gamma$, which means that we cannot distinguish between
$\beta$ and $\gamma$ by using the order $\preceq $. (2) If we use
the order $\preceq _{_{\mathrm{W}}}$ defined in Definition~\ref{Order-Wu},
then $S(\beta)=S(\gamma)=0.1 $ and $H_{1}(\beta)=0.3<0.5=H_{1}(\gamma)$,
and thus $\beta\prec _{_{\mathrm{W}}}\gamma$. Clearly, this is reasonable.
\end{example}

\begin{theorem}
\label{Total-Order-Thm}
The order $\preceq_{_{\mathrm{W}}}$ defined in Definition~\ref{Order-Wu} is
a total order.
\end{theorem}

\begin{remark}
(1) The bottom and the top elements of $\mathscr{P}$
{are, respectively,} $\langle0,0,1\rangle$ and $\langle1,0,0%
\rangle$ under the order $\preceq _{_{\mathrm{W}}}$.

(2) According to our best knowledge, there is no general methodology that
ranks any two PFNs.
\end{remark}

Let $L([0,1])=\{[u, v]\mid 0\leq u\leq v\leq1\}$.
If a binary relation $\preceq $ on $L([0,1])$ is a total order that refines the usual
order on $L([0,1])$, then is is called an
\textit{admissible order}; that is, for any $[u_1, v_1],$ $[u_2, v_2]\in L([0,1]),$ $%
[u_1, v_1]\preceq \lbrack u_2, v_2]$ whenever $u_1\leq u_2$ and $v_1\leq v_2$, this was introduced by Bustince
et al. (see~\cite[%
Definition~3.1]{BFKM2013}). Motivated by this fact, we introduce the
following admissible order for PFNs.

\begin{definition}
An order $\preceq $ on $\mathscr{P}$ is said to be an \textit{admissible
order} if it is a total order that refines the order $\subseteq $ defined in
Definition~\ref{sub-order}; i.e., for any $\beta$, $\gamma \in \mathscr{P}$%
, $\beta \subseteq \gamma$ implies that $\beta \preceq \gamma$.
\end{definition}

\begin{theorem}
\label{Adm-Order-Thm}
The order $\preceq_{_{\mathrm{W}}}$ defined in Definition~\ref{Order-Wu} is
an admissible order.
\end{theorem}

\begin{theorem}
$(\mathscr{P}, \preceq_{_{\mathrm{W}}})$ is a complete lattice.
\end{theorem}

\begin{proof}
Analogically to the proof of \cite[Theorem~4.1]{WWLCZ2021}, we conclude that
this is true. We omit its proof here.
\end{proof}

\section{Limitations in existing picture fuzzy aggregation operators}\label{S-4}

After PFSs were introduced by Cuong and Kreinovich~\cite{CK2013}, various
operational laws for PFSs and PFNs were proposed. However, many of these
laws violate the condition that the sum of the three degrees should be
in the interval $[0, 1]$. Thus, they are not closed in $\mathscr{P}$.

\subsection{Operations in \protect\cite{AMAK2019,Ga2017}}

By applying the additive generators, Garg~\cite{Ga2017} introduced the
following operational laws for PFNs and obtained some basic properties.

\begin{definition}[{\textrm{\protect\cite[Definition 4]{Ga2017}}}]
\label{Ga-Operations} Let $\tau$ be an AG of a t-norm $T$
and $\zeta $ be an AG of the dual t-conorm of $T;$ i.e., $\zeta (u)=\tau (1-u)$.
For $\alpha _{1}=\left\langle \mu _{1},\eta _{1},\nu
_{1}\right\rangle $, $\alpha _{2}=\left\langle \mu _{2},\eta _{2},\nu
_{2}\right\rangle $, $\alpha =\langle \mu ,\eta ,\nu \rangle\in \mathscr{P}$,
define
\begin{enumerate}[(i)]
\item $\alpha _{1}\oplus \alpha _{2}=\big\langle \zeta ^{-1}(\zeta (\mu
_{1})+\zeta (\mu _{2})),\tau ^{-1}(\tau (\eta _{1})+\tau (\eta _{2})),
\tau^{-1}(\tau (\nu _{1})+\tau (\nu _{2}))\big\rangle $;

\item $\alpha _{1}\otimes \alpha _{2}=\big\langle \tau ^{-1}(\tau (\mu
_{1})+\tau (\mu _{2})),\zeta ^{-1}(\zeta (\eta _{1})+\zeta (\eta
_{2})),
\zeta ^{-1}(\zeta (\nu _{1})+\zeta (\nu _{2}))\big\rangle $;

\item $\lambda \cdot \alpha =\left\langle \zeta ^{-1}(\lambda \zeta (\mu
)),\tau ^{-1}(\lambda \tau (\eta )),\tau ^{-1}(\lambda \tau (\nu
))\right\rangle $;

\item $\alpha ^{\lambda }=\left\langle \tau ^{-1}(\lambda \tau (\mu )),\zeta
^{-1}(\lambda \zeta (\eta )),\zeta ^{-1}(\lambda \zeta (\nu ))\right\rangle $%
; $\lambda >0$.
\end{enumerate}
\end{definition}

\begin{theorem}[{\textrm{\protect\cite[Theorem~1]{Ga2017}}}]
\label{Garg-Thm} Let $\alpha_{1}$, $\alpha_{2}$, and $\alpha$ be three
PFNs. Then, $\alpha _{1}\oplus \alpha _{2}$, $\alpha _{1}\otimes \alpha _{2}$%
, $\lambda \cdot \alpha $, and $\alpha ^{\lambda }$ ($\lambda >0$) are
PFNs.
\end{theorem}

The following example shows that none of the operational laws defined in
Definition~\ref{Ga-Operations} (\cite[Definition 4]{Ga2017}) is closed in $%
\mathscr{P}$. Thus, Theorem~\ref{Garg-Thm} (\cite[Theorem~1]{Ga2017}) does
not hold.

\begin{example}
\label{Exm-Wu} Define $\tau: [0, 1]\rightarrow [0, +\infty]$ as follows:
\begin{equation*}
\tau(x)=
\begin{cases}
+\infty, & x=0, \\
-\log_{8}x-\frac{1}{3}, & x\in (0, \frac{1}{8}], \\
-\frac{8}{3}x+1, & x\in [\frac{1}{8}, \frac{1}{4}], \\
-\frac{1}{3}x+\frac{5}{12}, & x\in [\frac{1}{4}, \frac{1}{2}], \\
-\frac{1}{2}(x-1), & x\in [\frac{1}{2}, 1],%
\end{cases}%
\end{equation*}
$\zeta(x)=\tau(1-x)$, and take $\alpha_1=\alpha_2=\langle\frac{1}{2}, \frac{1%
}{4}, \frac{1}{4}\rangle \in \mathscr{P}$. Clearly, $\tau$ is an AG of some strict t-norm
by Lemma~\ref{Concinuous-Archi-Representation}.
\begin{figure}[H]
\begin{center}
\scalebox{0.45}{\includegraphics{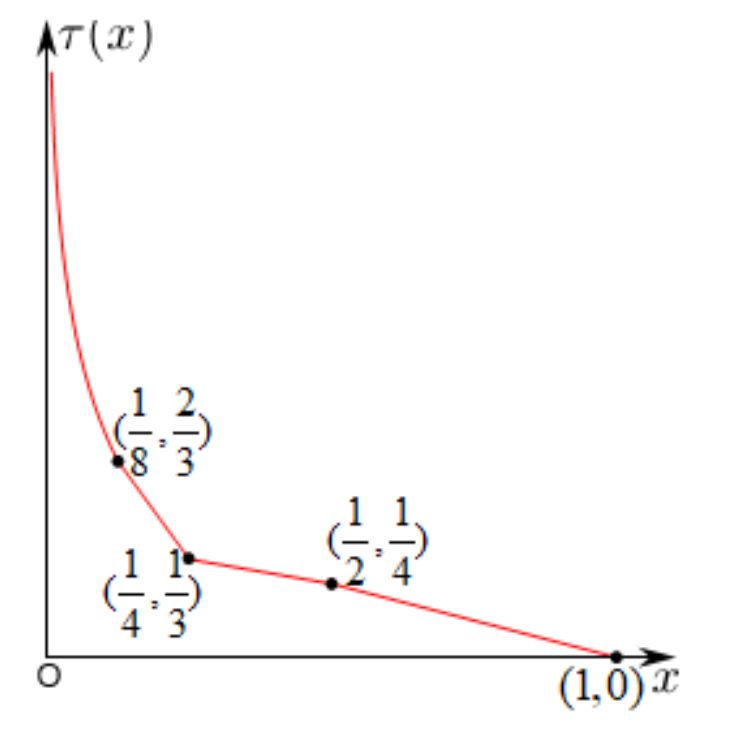}}
\end{center}
\caption{Graph of $\protect\tau$}
\label{Fig-1}
\end{figure}

Applying Definition~\ref{Ga-Operations} and by direct calculation, we have
\begin{small}\begin{align*}
&\quad 2\alpha_1=\alpha_1\oplus \alpha_2\\
& =\left\langle \zeta^{-1}\left(2\zeta\left(\frac{1}{2}\right)\right),
\tau^{-1}\left(2\tau\left(\frac{1}{4}\right)\right),
\tau^{-1}\left(2\tau\left(\frac{1}{4}\right)\right)\right\rangle \\
& =\left\langle\zeta^{-1}\left(\frac{1}{2}\right), \frac{1}{8}, \frac{1}{8}%
\right\rangle =\left\langle \frac{13}{16}, \frac{1}{8}, \frac{1}{8}%
\right\rangle,
\end{align*}
\end{small}
and
\begin{small}
\begin{align*}
&\quad (\alpha_1)^{2}=\alpha_1\otimes \alpha_2\\
& =\left\langle \tau^{-1}\left(2\tau\left(\frac{1}{2}\right)\right),
\zeta^{-1}\left(2\zeta\left(\frac{1}{4}\right)\right),
\zeta^{-1}\left(2\zeta\left(\frac{1}{4}\right)\right)\right\rangle \\
%%& =\left\langle\tau^{-1}\left(\frac{1}{2}\right), \zeta^{-1}\left(\frac{1}{4}%
%%\right), \zeta^{-1}\left(\frac{1}{4}\right)\right\rangle \\
& =\left\langle \frac{3}{16}, \frac{1}{2}, \frac{1}{2}\right\rangle.
\end{align*}
\end{small}
Clearly, $2\alpha_1=\alpha_1\oplus \alpha_2=\left\langle \frac{13}{16},
\frac{1}{8}, \frac{1}{8}\right\rangle \notin \mathscr{P}$ and $%
(\alpha_1)^{2}=\alpha_1\otimes \alpha_2=\left\langle \frac{3}{16}, \frac{1}{2%
}, \frac{1}{2}\right\rangle \notin \mathscr{P}$ since $\frac{13}{16}+\frac{1%
}{8}+\frac{1}{8}=\frac{17}{16}>1$ and $\frac{3}{16}+\frac{1}{2}+\frac{1}{2}=%
\frac{19}{16}>1$.
\end{example}

The above example demonstrates that $\lambda \cdot \alpha $ and $\alpha
^{\lambda }$ are not closed in $\mathscr{P}$ for $\lambda >1$. The following
example illustrates that they are also not closed in $\mathscr{P}$ for $%
\lambda \in (0,1)$. This observation means that we cannot use the operations
in Definition~\ref{Ga-Operations} to construct aggregation operators for
PFNs.

\begin{example}[\textrm{Continuation of Example~\protect\ref{Exm-Wu}}]
Given $\tau $ and $\zeta $ as in Example~\ref{Exm-Wu}, we choose $\alpha
=\langle \frac{1}{2},\frac{1}{4},\frac{1}{4}\rangle \in \mathscr{P},$ and $%
\beta =\langle \frac{1}{4},\frac{1}{2},\frac{1}{4}\rangle \in \mathscr{P}$.
By direct calculation, we have
$\frac{1}{2}\alpha=\langle\zeta^{-1}(\frac{1}{2}\zeta(\frac{1%
}{2})), \tau^{-1}(\frac{1}{2}\tau(\frac{1}{4}%
)), \tau^{-1}(\frac{1}{2}\tau(\frac{1}{4}%
))\rangle=\langle \frac{1}{4}, \frac{2}{3}, \frac{2}{3}\rangle$,
and
$(\beta)^{\frac{1}{2}}=\langle \tau^{-1}(\frac{1}{2}\tau(%
\frac{1}{4})), \zeta^{-1}(\frac{1}{2}\zeta(\frac{1}{2}%
)), \zeta^{-1}(\frac{1}{2}\zeta(\frac{1}{4}%
))\rangle=\langle \frac{2}{3}, \frac{1}{4}, \frac{1}{8}\rangle$.
Clearly, $\frac{1}{2}\alpha= \left\langle \frac{1}{4}, \frac{2}{3}, \frac{2}{%
3}\right\rangle\notin \mathscr{P}$ and $(\beta)^{\frac{1}{2}}=\left\langle
\frac{2}{3}, \frac{1}{4}, \frac{1}{8}\right\rangle \notin \mathscr{P}$,
since $\frac{1}{4}+\frac{2}{3}+\frac{2}{3}=\frac{19}{12}>1$ and $\frac{2}{3}+%
\frac{1}{4}+\frac{1}{8}=\frac{25}{24}>1$.
\end{example}

Ashraf et al.~\cite{AMAK2019} introduced the similar operational laws for
PFNs by using the Archimedean t-(co)norm.

\begin{definition}[{\textrm{\protect\cite[Definition~9]{AMAK2019}}}]
\label{AMAK-Operations} Let $\tau$ be an AG of an
Archimedean t-norm, and $\zeta$ be an AG of the dual t-conorm of $T;$ i.e., $\zeta
(u)=\tau (1-u)$. For $\alpha _{1}=\left\langle \mu _{1},\eta _{1},\nu
_{1}\right\rangle $, $\alpha _{2}=\left\langle \mu _{2},\eta _{2},\nu
_{2}\right\rangle $, $\alpha =\langle \mu ,\eta ,\nu \rangle\in \mathscr{P}$,
and $\lambda>0$,
define
\begin{enumerate}[(i)]
\item $\alpha _{1}\otimes _{_{\mathrm{As}}}\alpha _{2}=\big\langle \tau
^{-1}(\tau (\mu _{1})+\tau (\mu _{2})),\tau ^{-1}(\tau (\eta _{1})+\tau
(\eta _{2})),
\zeta ^{-1}(\zeta (\nu _{1})+\zeta (\nu _{2}))\big\rangle $;

\item $\alpha ^{\lambda _{_{\mathrm{As}}}}=\left\langle \tau ^{-1}(\lambda
\tau (\mu )),\tau ^{-1}(\lambda \tau (\eta )),\zeta ^{-1}(\lambda \zeta (\nu
))\right\rangle $.
\end{enumerate}
\end{definition}

\begin{example}
Given $\tau $ and $\zeta $ as in Example~\ref{Exm-Wu}, we choose $\alpha
=\langle \frac{1}{4},\frac{1}{4},\frac{1}{2}\rangle \in \mathscr{P}$.
Clearly, the t-norm with the AG $\tau $ is strict; and thus,
it is Archimedean by Lemma~\ref{Strict-->Archi}.
By direct calculation, we have that
$\alpha\otimes_{_{\mathrm{As}}} \alpha =
\langle \tau^{-1}(2\tau(\frac{1}{4})),
\tau^{-1}(2\tau(\frac{1}{4})),
\zeta^{-1}(2\zeta(\frac{1}{2}))\rangle =
\langle \frac{1}{8}, \frac{1}{8}, \zeta^{-1}(\frac{1}{2}
)\rangle =\langle \frac{1}{8}, \frac{1}{8}, \frac{13}{16}
\rangle,
$
and
$\alpha^{(\frac{1}{2})_{_{\mathrm{As}}}}
=\langle \tau^{-1}(\frac{1}{2}\tau(\frac{1}{4})), \tau^{-1}(\frac{1}{2}
\tau(\frac{1}{4})), \zeta^{-1}(\frac{1}{2}\zeta(\frac{1}{2}))\rangle
=\langle \frac{2}{3}, \frac{2}{3}, \frac{1}{4}\rangle.
$
Clearly, $\alpha\otimes_{_{\mathrm{As}}} \alpha=\left\langle \frac{1}{8},
\frac{1}{8}, \frac{13}{16} \right\rangle \notin \mathscr{P}$ and $\alpha^{(
\frac{1}{2})_{_{\mathrm{As}}}}= \left\langle \frac{2}{3}, \frac{2}{3}, \frac{
1}{4}\right\rangle\notin \mathscr{P}$ since $\frac{1}{8}+\frac{1}{8}+\frac{
13}{16}=\frac{17}{16}>1$ and $\frac{2}{3}+\frac{2}{3}+\frac{1}{4}=\frac{19}{
12}>1$. This example implies that neither of the operational laws defined in
Definition~\ref{AMAK-Operations} (\cite[Definition~9]{AMAK2019}) is closed
in $\mathscr{P}$.
\end{example}

\subsection{Operations in \protect\cite{JSPY2019}}

\begin{definition}[{\textrm{\protect\cite[Definition 12]{JSPY2019}}}]
\label{JSPY-Def} Let $\alpha _{1}=\left\langle \mu _{1},\eta _{1},\nu
_{1}\right\rangle $ and $\alpha _{2}=\left\langle \mu _{2},\eta _{2},\nu
_{2}\right\rangle $ be two PFNs, $\gamma \geq 1$, and $\lambda >0$.
Define
\begin{enumerate}[(i)]
\item $\alpha _{1}\oplus \alpha _{2}
=\left\langle 1-\frac{1}{1+\left[\left(\frac{\mu _{1}}{1-\mu _{1}}\right) ^{\gamma }
+\left( \frac{\mu _{2}}{1-\mu
_{2}}\right) ^{\gamma }\right] ^{1/\gamma }},\right.
\newline
\left.\frac{1}{1+\left[\left( \frac{%
1-\eta _{1}}{\eta _{1}}\right) ^{\gamma }+\left( \frac{1-\eta _{2}}{\eta _{2}%
}\right) ^{\gamma }\right] ^{1/\gamma }},
\frac{1}{1+\left[\left( \frac{%
1-\nu _{1}}{\nu _{1}}\right) ^{\gamma }+\left( \frac{1-\nu _{2}}{\nu _{2}}%
\right) ^{\gamma }\right] ^{1/\gamma }}\right\rangle $;

\item $\alpha _{1}\otimes \alpha _{2}
=\left\langle \frac{1}{1+\left[ \left(
\frac{1-\mu _{1}}{\mu _{1}}\right) ^{\gamma }+\left( \frac{1-\mu _{2}}{\mu
_{2}}\right) ^{\gamma }\right] ^{1/\gamma }}, \right.\newline
~~~~~~~~~~~~~~~~~ 1-\frac{1}{1+\left[ \left(
\frac{\eta _{1}}{1-\eta _{1}}\right) ^{\gamma }+\left( \frac{\eta _{2}}{%
1-\eta _{2}}\right) ^{\gamma }\right] ^{1/\gamma }},\newline
~~~~~~~~~~~~~~~~\left.1-\frac{1}{1+\left[
\left( \frac{\nu _{1}}{1-\nu _{1}}\right) ^{\gamma }+\left( \frac{\nu _{2}}{%
1-\nu _{2}}\right) ^{\gamma }\right] ^{1/\gamma }}\right\rangle $;

\item $\lambda \cdot \alpha _{1}=\left\langle 1-\frac{1}{1+\left[ \lambda
\left( \frac{\mu _{1}}{1-\mu _{1}}\right) ^{\gamma }\right] ^{1/\gamma }},%
\frac{1}{1+\left[ \lambda \left( \frac{1-\eta _{1}}{\eta _{1}}\right)
^{\gamma }\right] ^{1/\gamma }},\right.\newline
\left.\frac{1}{1+\left[ \lambda \left( \frac{1-\nu
_{1}}{\nu _{1}}\right) ^{\gamma }\right] ^{1/\gamma }}\right\rangle $;

\item $(\alpha _{1})^{\lambda }=\left\langle \frac{1}{1+\left[ \lambda
\left( \frac{1-\mu _{1}}{\mu _{1}}\right) ^{\gamma }\right] ^{1/\gamma }},1-%
\frac{1}{1+\left[ \lambda \left( \frac{\eta _{1}}{1-\eta _{1}}\right)
^{\gamma }\right] ^{1/\gamma }},\right.\newline
\left.1-\frac{1}{1+\left[ \lambda \left( \frac{\nu
_{1}}{1-\nu _{1}}\right) ^{\gamma }\right] ^{1/\gamma }}\right\rangle $.
\end{enumerate}
\end{definition}

\begin{remark}
\label{Remark-Dombi}
(i) Notice that Dombi t-norm $T_{\gamma }^{\textbf{D}}(x,y)=\frac{1}{1+\left( \left(
\frac{1-x}{x}\right) ^{\gamma }\left( \frac{1-y}{y}\right) ^{\gamma }\right)
^{1/\gamma }}$ ($\gamma\in (0, +\infty)$) is a strict t-norm, its dual t-conorm
$S_{\gamma }^{\textbf{D}}(x,y)=1-\frac{1}{%
1+\left( \left( \frac{x}{1-x}\right) ^{\gamma }\left( \frac{y}{1-y}\right)
^{\gamma }\right) ^{1/\gamma }}$, and $\alpha _{1}\otimes \alpha _{2}
=\langle T_{\gamma }^{\textbf{D}}(\mu _{1},\mu _{2}),S_{\gamma }^{\textbf{D}}(\eta _{1},\eta
_{2}),S_{\gamma }^{\textbf{D}}(\nu _{1},\nu _{2})\rangle$. Then, we have
$T_{\gamma }^{\textbf{D}}(\mu _{1},\mu _{2})+S_{\gamma }^{\textbf{D}}(\eta
_{1},\eta _{2})+S_{\gamma }^{\textbf{D}}(\nu _{1},\nu _{2})\geq S_{\gamma }^{%
\textbf{D}}(\eta _{1},\eta _{2})+S_{\gamma }^{\textbf{D}}(\nu _{1},\nu
_{2})\geq \max \{\eta _{1},\eta _{2}\}+\max \{\nu _{1},\nu _{2}\},$
which can exceed one (see Example~\ref{Dombi>1}). Thus, the operation $%
\otimes $ defined in Definition~\ref{JSPY-Def} (\cite[Definition 12]%
{JSPY2019}) is not closed in $\mathscr{P}$. %%\item

(ii) For any $\alpha =\langle \mu_{\alpha}, \eta_{\alpha}, \nu_{\alpha} \rangle
\in \mathscr{P}$ with $0<\mu_{\alpha}, \eta_{\alpha}, \nu_{\alpha}<1$, by Definition~\ref{JSPY-Def}, we have
$
\lim_{\lambda \rightarrow 0^{+}}\lambda \cdot \alpha =\left\langle
0,1,1\right\rangle \notin \mathscr{P}$
and
$
\lim_{\lambda \rightarrow +\infty }\alpha ^{\lambda }=\left\langle
0,1,1\right\rangle \notin \mathscr{P},
$
implying that there are $\lambda _{1}$, $\lambda _{2}>0$ such that (1) for
any $0<\lambda <\lambda _{1}$, $\lambda \cdot \alpha \notin \mathscr{P}$;
(2) for any $\lambda >\lambda _{2}$, $\alpha ^{\lambda }\notin \mathscr{P}$.
Thus, neither of the operations $\lambda \cdot \alpha $ and $\alpha
^{\lambda }$ in Definition~\ref{JSPY-Def} (\cite[Definition 12]{JSPY2019})
is closed in $\mathscr{P}$.
\end{remark}

\begin{example}
\label{Dombi>1} Let $\alpha_1=\langle\frac{1}{4}, \frac{3}{4},
0\rangle$ and $\alpha_2=\langle\frac{1}{4}, 0, \frac{3}{4}%
\rangle$. By Definition~\ref{JSPY-Def}, we have
$\alpha_1\otimes \alpha_2=\langle\frac{1}{1+3\sqrt[\gamma]{2}}, \frac{3}{%
4}, \frac{3}{4}\rangle$.
 Clearly, $\alpha_1\otimes \alpha_2$ is not a PFN since $\frac{1}{1+3\sqrt[%
\gamma]{2}}+\frac{3}{4}+\frac{3}{4}\geq\frac{3}{2}>1$ holds for any $%
\gamma>0 $. Thus, the operation $\otimes$ in Definition~\ref{JSPY-Def} is
not closed in $\mathscr{P}$.
\end{example}

\subsection{Operations in \protect\cite{KAA2019}}

The same problem as Remark~\ref{Remark-Dombi} may occur in \cite[Definition~8%
]{KAA2019}.

\begin{definition}[{\textrm{\protect\cite[Definition~8]{KAA2019}}}]
\label{KAA-Def}
For $\alpha_1=\langle\mu_1, \eta_1, \nu_1\rangle$,
$\alpha_2=\langle \mu_2, \eta_2, \nu_2\rangle \in \mathscr{P}$, define
\begin{enumerate}[(i)]
\item $\alpha _{1}\oplus _{\varepsilon }\alpha
_{2}=\left\langle \frac{\mu _{1}+\mu _{2}}{1+\mu _{1}\mu _{2}},\frac{\eta
_{1}\eta _{2}}{1+(1-\eta _{1})(1-\eta _{2})},\frac{\nu _{1}\nu _{2}}{%
1+(1-\nu _{1})(1-\nu _{2})}\right\rangle $;

\item $\alpha _{1}\otimes _{\varepsilon }\alpha
_{2}=\left\langle \frac{\mu _{1}\mu _{2}}{1+(1-\mu _{1})(1-\mu _{2})},\frac{%
\eta _{1}+\eta _{2}}{1+\eta _{1}\eta _{2}},\frac{\nu _{1}+\nu _{2}}{1+\nu
_{1}\nu _{2}}\right\rangle $;

\item $\lambda _{\cdot \varepsilon }\alpha _{1}=\newline
\left\langle
\frac{(1+\mu _{1})^{\lambda }-(1-\mu _{1})^{\lambda }}{(1+\mu _{1})^{\lambda
}+(1-\mu _{1})^{\lambda }},\frac{2(\eta _{1})^{\lambda }}{(2-\eta
_{1})^{\lambda }+(\eta _{1})^{\lambda }},\frac{2(\nu _{1})^{\lambda }}{%
(2-\nu _{1})^{\lambda }+(\nu _{1})^{\lambda }}\right\rangle $;

\item $(\alpha _{1})^{\lambda }=\newline
\left\langle \frac{2(\mu
_{1})^{\lambda }}{(2-\mu _{1})^{\lambda }+(\mu _{1})^{\lambda }},\frac{%
(1+\eta _{1})^{\lambda }-(1-\eta _{1})^{\lambda }}{(1+\eta _{1})^{\lambda
}+(1-\eta _{1})^{\lambda }},\frac{(1+\nu _{1})^{\lambda }-(1-\nu
_{1})^{\lambda }}{(1+\nu _{1})^{\lambda }+(1-\nu _{1})^{\lambda }}%
\right\rangle $.
\end{enumerate}
\end{definition}

\begin{remark}
\label{Remark-Einstein}
(i)  Notice that Einstein product $T^{\textbf{E}}(x,y)=\frac{xy}{%
1+(1-x)(1-y)}$ is a strict t-norm, its dual t-conorm
$S^{\textbf{E}}(x,y)=\frac{x+y}{1+xy}$, and $\alpha _{1}\otimes _{\varepsilon
}\alpha _{2}=\langle T^{\textbf{E}}(\mu _{1},\mu _{2}),S^{\textbf{E}%
}(\eta _{1},\eta _{2}),S^{\textbf{E}}(\nu _{1},\nu _{2})\rangle$. Then,
we have
$
T^{\textbf{E}}(\mu _{1},\mu _{2})+S^{\textbf{E}}(\eta _{1},\eta _{2})+S^{%
\textbf{E}}(\nu _{1},\nu _{2})\geq S^{\textbf{E}}(\eta _{1},\eta _{2})+S^{%
\textbf{E}}(\nu _{1},\nu _{2})\geq \max \{\eta _{1},\eta _{2}\}+\max \{\nu
_{1},\nu _{2}\},
$
which can exceed one (see Example~\ref{Einstein>1}). Thus, the operation $%
\otimes _{\varepsilon }$ defined in Definition~\ref{KAA-Def} is not closed
in $\mathscr{P}$. %%\item

(ii) Consider an arbitrary $\alpha =\langle \mu_{\alpha}, \eta_{\alpha},
\nu_{\alpha} \rangle \in \mathscr{P}$ with $0<\mu_{\alpha},\eta_{\alpha},
\nu_{\alpha} <1$. Then, by Definition~\ref{KAA-Def},
we have
$
\lim_{\lambda \rightarrow 0^{+}}\lambda _{\cdot \varepsilon }\alpha
=\left\langle 0,1,1\right\rangle \notin \mathscr{P}
$
and
$
\lim_{\lambda \rightarrow +\infty }\alpha ^{\lambda }=\left\langle
0,1,1\right\rangle \notin \mathscr{P}.
$
These facts imply that there are $\lambda _{1}$, $\lambda _{2}>0$ such that
 (1) for any $0<\lambda <\lambda _{1}$, $\lambda _{\cdot \varepsilon }\alpha
\notin \mathscr{P}$;
(2) for any $\lambda >\lambda _{2}$, $\alpha ^{\lambda }\notin \mathscr{P}$.
Hence, neither of the operations $\lambda \cdot \alpha $ and $\alpha
^{\lambda }$ in Definition~\ref{KAA-Def} (\cite[Definition~8]{KAA2019}) is
closed in $\mathscr{P}$.
\end{remark}

\begin{example}
\label{Einstein>1} Let $\alpha_1=\left\langle\frac{1}{4}, \frac{3}{4},
0\right\rangle$ and $\alpha_2=\left\langle\frac{1}{4}, 0, \frac{3}{4}%
\right\rangle$. By Definition~\ref{KAA-Def}, we have
$
\alpha_1\otimes_{\varepsilon} \alpha_2=\left\langle\frac{1}{25}, \frac{3}{4}%
, \frac{3}{4}\right\rangle.
$
Clearly, $\alpha_1\otimes \alpha_2$ is not a PFN duing to $\frac{1}{25}+\frac{3%
}{4}+\frac{3}{4}\geq\frac{3}{2}>1$. Thus, the operation $\otimes_{%
\varepsilon}$ in Definition~\ref{KAA-Def} is not closed in $\mathscr{P}$.
\end{example}

\begin{proposition}[{\textrm{\protect\cite[Corollary~1]{KAA2019}}}]
\label{KAA-Corollary} Let $\alpha$ be a PFN. Then, for any $\lambda\in \mathbb{R}^{+}$,
$\lambda_{\cdot \varepsilon}\alpha$ is a PFN.
\end{proposition}

By Remark~\ref{Remark-Einstein} (ii), we observe that \cite[Corollary~1]%
{KAA2019} does not hold. So, \cite[Theorem~2 and Theorem~3]{KAA2019} are
also not valid. Considering $\alpha =\big\langle0,\frac{1}{2},\frac{1}{2}%
\big\rangle$, it follows from Definition~\ref{KAA-Def} that, for any $%
\lambda >0$,
$
\lambda _{\cdot \varepsilon }\alpha =\left\langle 0,\frac{2\cdot \left(
\frac{1}{2}\right) ^{\lambda }}{\left( \frac{3}{2}\right) ^{\lambda }+\left(
\frac{1}{2}\right) ^{\lambda }},\frac{2\cdot \left( \frac{1}{2}\right)
^{\lambda }}{\left( \frac{3}{2}\right) ^{\lambda }+\left( \frac{1}{2}\right)
^{\lambda }}\right\rangle.
$
In particular, when taking $\lambda =\frac{1}{2}$, we obtain $\left( \frac{1%
}{2}\right) _{\cdot \varepsilon }\alpha =\langle 0,{\frac{2}{\sqrt{3}+1}%
,\frac{2}{\sqrt{3}+1}}\rangle $. Since $0+{\frac{2}{\sqrt{3}+1}+\frac{2%
}{\sqrt{3}+1}}=\frac{4}{\sqrt{3}+1}>1,$ it is not a PFN.

\subsection{Operations in \protect\cite{Wei2017,Wei2018a}}

Following the operations for the IFSs (\cite%
{Xu2007,XC2012}), Wei \cite{Wei2017} introduced the following operational
laws of PFNs.

\begin{definition}[{\textrm{\protect\cite[Definition 8]{Wei2017}}}]
\label{Wei-Def}
For $\alpha_1=\langle\mu_1, \eta_1, \nu_1\rangle$, $\alpha_2=\langle\mu_2,
\eta_2, \nu_2\rangle\in \mathscr{P}$, define
\begin{enumerate}[(1)]
\item $\alpha _{1}\wedge \alpha_{2}=\langle \mu_{1}\wedge \mu_{2},
\eta_{1}\vee \eta _{2}, \nu_{1}\vee\nu_{2}\rangle$;

\item $\alpha _{1}\vee \alpha_{2}=\langle \mu_{1}\vee \mu_{2},
\eta_{1}\wedge \eta_{2}, \nu_{1}\wedge \nu_{2}\rangle$;

\item $\alpha _{1}\oplus \alpha _{2}=\langle\mu _{1}+\mu _{2}-\mu
_{1}\mu _{2},\eta _{1}\eta _{2},\nu _{1}\nu _{2}\rangle$;

\item $\alpha _{1}\otimes \alpha _{2}=\langle\mu _{1}\mu _{2},\eta
_{1}+\eta _{2}-\eta _{1}\eta _{2},\nu _{1}+\nu _{2}-\nu _{1}\nu _{2}%
\rangle$;

\item $\lambda \cdot \alpha _{1}=\langle1-(1-\mu _{1})^{\lambda },\eta
_{1}^{\lambda },\nu_{1}^{\lambda}\rangle$, $\lambda >0$;

\item $(\alpha _{1})^{\lambda }=\langle\mu _{1}^{\lambda },1-(1-\eta
_{1})^{\lambda },1-(1-\nu _{1})^{\lambda }\rangle$, $\lambda >0$.
\end{enumerate}
\end{definition}

Tian et al.~\cite{TPZZW2019} showed that the operation $\otimes $ defined in
Definition~\ref{Wei-Def} is not closed in $\mathscr{P}$ (also see \cite[%
Example~2]{PWC2021}). Clearly, the operation $\wedge $ in Definition~\ref%
{Wei-Def} is also not closed in $\mathscr{P}$. More precisely, for two PFNs $%
\langle0, 1, 0\rangle,$ $\langle0,0,1\rangle,$ we get that $%
\langle0,1,0\rangle\wedge \langle0,0,1\rangle=\langle%
0,1,1\rangle$. It is easy to see that $\langle0,1,1\rangle$ is not
a PFN. The following example reveals that neither of the operations $\lambda
\cdot \alpha $ and $\alpha ^{\lambda }$ is closed in $\mathscr{P}$. So, they
are not a PFN. Therefore, many aggregation operators induced by the
operations in Definition~\ref{Wei-Def} do not work (see Examples~\ref%
{Exm-3.1}, \ref{Exm-3.7} and \ref{Exm-3.8}).

\begin{example}
Consider a PFN $\alpha= \langle 0, 0.5, 0.5 \rangle$. By Definition~\ref%
{Wei-Def}, it can be verified that
$\frac{1}{2} \alpha=\langle 0, \frac{1}{\sqrt{2}}, \frac{1}{\sqrt{2}}
\rangle$ and $\alpha^{2}=\langle 0, \frac{3}{4}, \frac{3}{4}\rangle.$
Clearly, $\frac{1}{2} \alpha\notin \mathscr{P}$ and $\alpha^{2}\notin %
\mathscr{P}$, since $0+\frac{1}{\sqrt{2}}+\frac{1}{\sqrt{2}}=\sqrt{2}>1$ and
$0+\frac{3}{4}+\frac{3}{4}=\frac{3}{2}>1$.
\end{example}

\begin{definition}[{\textrm{\protect\cite[Definition 9]{Wei2017}}}]
\label{PFWA-Wei} Let $\omega=(\omega_1, \ldots, \omega_n)^{\top}$
be the weight vector such that $\omega_{i}\in (0, 1]$ and $%
\sum_{j=1}^{n}\omega_j=1$ and $\alpha_{j}\in \mathscr{P}$ ($j=1,\ldots, n$).
Define the \textit{picture fuzzy weighted average operator}
(PFWA) $\mathrm{PFWA}_{\omega}: \mathscr{P}^{n}\rightarrow \mathscr{P}$
by
$
\mathrm{PFWA}_{\omega}(\alpha_1, \ldots, \alpha_n)=
\bigoplus_{j=1}^{n}(\omega_{j}\cdot \alpha_j).
$
\end{definition}

It is easy check that the operator described in Definition~\ref{PFWA-Wei} is
not well-defined since $\omega _{j}\cdot \alpha _{j}$ does not need to be a
PFN for $\omega _{j}\in (0,1)$.

\begin{definition}[{\textrm{\protect\cite[Definition 13]{Wei2017}}}]
\label{PFWG-Wei} Let $\omega=(\omega_1, \ldots, \omega_n)^{\top}$
be the weight vector such that $\omega_{i}\in (0, 1]$ and $%
\sum_{j=1}^{n}\omega_j=1$ and $\alpha_{j}\in \mathscr{P}$ ($j=1, \ldots, n$).
Define the \textit{picture fuzzy weighted geometric operator}
(PFWG) $\mathrm{PFWG}_{\omega}: \mathscr{P}^{n}\rightarrow \mathscr{P}$
by
$
\mathrm{PFWG}_{\omega}(\alpha_1, \ldots, \alpha_n)=
\bigotimes_{j=1}^{n}(\alpha_j)^{\omega_{j}}.
$
\end{definition}

Based on Definition~\ref{PFWG-Wei}, Wei~\cite{Wei2017} obtained the
following formula for the operator $\mathrm{PFWG}_{\omega}$.

\begin{theorem}[{\textrm{\protect\cite[Theorem~11]{Wei2017}}}]
\label{PFWG-Thm-Wei} Let $\omega=(\omega_1, \ldots, \omega_n)^{\top}$
be the weight vector such that $\omega_{i}\in (0, 1]$ and $%
\sum_{j=1}^{n}\omega_j=1$ and $\alpha_{j}\in \mathscr{P}$ ($j=1, \ldots, n$).
Then, $\mathrm{PFWG}_{\omega}(\alpha_1,\ldots, \alpha_n)\in \mathscr{P}$
and
$\mathrm{PFWG}_{\omega}(\alpha_1, \ldots, \alpha_n)
=\langle
\prod_{j=1}^{n}(\mu_j)^{\omega_j}, 1-\prod_{j=1}^{n}(1-\eta_j)^{\omega_j},
1-\prod_{j=1}^{n}(1-\nu_j)^{\omega_j}\rangle.
$
\end{theorem}

\begin{example}
\label{Exm-3.1} Take $\alpha_1=\langle 0, 0.9, 0.1 \rangle$, $%
\alpha_2=\langle 0, 0.1, 0.9\rangle
\in \mathscr{P}$, and $\omega=(0.5, 0.5)^{\top}$. By Theorem~\ref{PFWG-Thm-Wei}, we have
$\mathrm{PFWG}_{\omega}(\alpha_1, \alpha_2) =\langle 0, 1-\sqrt{1-0.9}%
\sqrt{1-0.1}, 1-\sqrt{1-0.1}\sqrt{1-0.9}\rangle =\langle 0, 0.7,
0.7\rangle\notin \mathscr{P}.$
This means that \cite[Theorem~11]{Wei2017} does not hold.
\end{example}

Based on the Hamacher product and the Hamacher sum, Wei~\cite{Wei2018a} introduced
Hamacher operations for PFNs.

\begin{definition}[{\textrm{\protect\cite[Definition 8]{Wei2018a}}}]
\label{Wei-Def} Let $\alpha_1=\langle\mu_1, \eta_1, \nu_1\rangle$, $\alpha_2=\langle\mu_2,
\eta_2, \nu_2\rangle \in \mathscr{P}$. For $\gamma>0$, define
\begin{enumerate}[(1)]
\item $\alpha _{1}\oplus \alpha _{2}=\Big\langle \frac{\mu _{1}+\mu
_{2}-\mu _{1}\mu _{2}-(1-\gamma )\mu _{1}\mu _{2}}{1-(1-\gamma )\mu _{1}\mu
_{2}},\frac{\eta _{1}\eta _{2}}{\gamma +(1-\gamma )(\eta _{1}+\eta _{2}-\eta
_{1}\eta _{2})},\newline
\frac{\nu _{1}\nu _{2}}{\gamma +(1-\gamma )(\nu _{1}+\nu
_{2}-\nu _{1}\nu _{2})}\Big\rangle $;

\item $\alpha _{1}\otimes \alpha _{2}=\left\langle \frac{\mu _{1}\mu _{2}}{%
\gamma +(1-\gamma )(\mu _{1}+\mu _{2}-\mu _{1}\mu _{2})},\frac{\eta
_{1}+\eta _{2}-\eta _{1}\eta _{2}-(1-\gamma )\eta _{1}\eta _{2}}{1-(1-\gamma
)\eta _{1}\eta _{2}},\right.\newline
\left.\frac{\nu _{1}+\nu _{2}-\nu _{1}\nu _{2}-(1-\gamma )\nu
_{1}\nu _{2}}{1-(1-\gamma )\nu _{1}\nu _{2}}\right\rangle ;$

\item $\lambda \cdot \alpha _{1}=\left\langle \frac{[1+(\gamma -1)\mu
_{1}]^{\lambda }-(1-\mu _{1})^{\lambda }}{[1+(\gamma -1)\mu _{1}]^{\lambda
}+(\gamma -1)(1-\mu _{1})^{\lambda }},\right.\newline
\left.\frac{\gamma \eta _{1}^{\lambda }}{%
[1+(\gamma -1)(1-\eta _{1})]^{\lambda }+(\gamma -1)\eta _{1}^{\lambda }},%
\frac{\gamma \nu _{1}^{\lambda }}{[1+(\gamma -1)(1-\nu _{1})]^{\lambda
}+(\gamma -1)\nu _{1}^{\lambda }}\right\rangle $, $\lambda >0$;

\item $(\alpha _{1})^{\lambda }=\left\langle \frac{\gamma \mu _{1}^{\lambda }%
}{[1+(\gamma -1)(1-\mu _{1})]^{\lambda }+(\gamma -1)\mu _{1}^{\lambda }},\right.\newline
\left.%
\frac{[1+(\gamma -1)\eta _{1}]^{\lambda }-(1-\eta _{1})^{\lambda }}{%
[1+(\gamma -1)\eta _{1}]^{\lambda }+(\gamma -1)(1-\eta _{1})^{\lambda }},%
\frac{[1+(\gamma -1)\nu _{1}]^{\lambda }-(1-\nu _{1})^{\lambda }}{[1+(\gamma
-1)\nu _{1}]^{\lambda }+(\gamma -1)(1-\nu _{1})^{\lambda }}\right\rangle $, $%
\lambda >0$.
\end{enumerate}
\end{definition}

\begin{remark}
\label{Remark-Wei}
(i) Notice that Hamacher t-norm $T_{\gamma }^{\textbf{H}}(x,y)=\frac{xy}{%
\gamma +(1-\gamma )(x+y-xy)}$ ($\gamma \in (0,+\infty )$) is a strict
t-norm, its dual t-conorm $S_{\gamma }^{\textbf{H}}(x,y)=\frac{x+y-xy-(1-\gamma )xy}{%
1-(1-\gamma )xy}$, and $\alpha
_{1}\otimes \alpha _{2}=\big\langle T_{\gamma }^{\textbf{H}}(\mu _{1},\mu
_{2}),S_{\gamma }^{\textbf{H}}(\eta _{1},\eta _{2}),S_{\gamma }^{\textbf{H}%
}(\nu _{1},\nu _{2})\big\rangle$ by Definition~\ref{Wei-Def}. Then, we have
$
T_{\gamma }^{\textbf{H}}(\mu _{1},\mu _{2})+S_{\gamma }^{\textbf{H}}(\eta
_{1},\eta _{2})+S_{\gamma }^{\textbf{H}}(\nu _{1},\nu _{2})\geq S_{\gamma }^{%
\textbf{H}}(\eta _{1},\eta _{2})+S_{\gamma }^{\textbf{H}}(\nu _{1},\nu
_{2})\geq \max \{\eta _{1},\eta _{2}\}+\max \{\nu _{1},\nu _{2}\},
$
which can exceed one. Thus, the operation $\otimes $ defined in Definition~%
\ref{Wei-Def} is not closed in $\mathscr{P}$.

(ii) Consider an arbitrary $\alpha =\langle \mu_{\alpha},\eta_{\alpha},\nu_{\alpha} \rangle \in %
\mathscr{P}$ with $0<\mu_{\alpha},\eta_{\alpha},\nu_{\alpha} <1$. Then, by Definition~\ref{Wei-Def},
we have
$
\lim_{\lambda \rightarrow 0^{+}}\lambda \cdot \alpha =\left\langle
0,1,1\right\rangle \notin \mathscr{P}
$
and
$
\lim_{\lambda \rightarrow +\infty }\alpha ^{\lambda }=\left\langle
0,1,1\right\rangle \notin \mathscr{P}.
$
These facts imply that there are $\lambda _{1}$, $\lambda _{2}>0$ such that
(1) for any $0<\lambda <\lambda _{1}$, $\lambda \cdot \alpha \notin %
\mathscr{P}$;
(2) for any $\lambda >\lambda _{2}$, $\alpha ^{\lambda }\notin \mathscr{P}$.
Thus, neither of the operations $\lambda \cdot \alpha $ and $\alpha
^{\lambda }$ in Definition~\ref{Wei-Def} (\cite[Definition 8]{Wei2018a}) is
closed in $\mathscr{P}$.
\end{remark}

In view of Definition~\ref{Wei-Def}, Wei~\cite{Wei2018a} defined the picture
fuzzy Hamacher weighted averaging (PFHWA) operator, the picture fuzzy
Hamacher ordered weighted averaging (PFHOWA) operator, and the picture fuzzy
Hamacher hybrid average (PFHHA) operator. By Remark~\ref{Remark-Wei}, we
observe that these operators are not well-defined.

\subsection{Operations in \protect\cite{LLCFL2021}}

To overcome shortcomings of operational laws in \cite%
{Wei2017,Wei2018a,WLG2018}, Lin et al.~\cite{LLCFL2021} introduced the
following interaction operational laws (IOLs) for PFNs.

\begin{definition}[{\textrm{\protect\cite[Definition~10]{LLCFL2021}}}]
\label{LLCFL-Def}
For $\alpha _{1}=\langle\mu _{1},\eta _{1},\nu _{1}\rangle$, $\alpha
_{2}=\langle\mu _{2},\eta _{2},\nu _{2}\rangle$, $\alpha =\langle \mu ,\eta ,\nu
\rangle \in \mathscr{P}$, define the IOLs as
follows:
\begin{enumerate}[(i)]
\item $\alpha _{1}\oplus _{_{\mathrm{Lin}}}\alpha _{2}=\big\langle \mu
_{1}+\mu _{2}-\mu _{1}\mu _{2},\eta _{1}+\eta _{2}-\eta _{1}\eta _{2}-\mu
_{1}\eta _{2}-\eta _{1}\mu _{2},\nu _{1}+\nu _{2}-\nu _{1}\nu _{2}-\mu
_{1}\nu _{2}-\nu _{1}\mu _{2}\big\rangle $;

\item $\alpha _{1}\otimes _{_{\mathrm{Lin}}}\alpha _{2}=\big\langle \mu
_{1}+\mu _{2}-\mu _{1}\mu _{2}-\mu _{1}\nu _{2}-\nu _{1}\mu _{2},\eta
_{1}+\eta _{2}-\eta _{1}\eta _{2},\nu _{1}+\nu _{2}-\nu _{1}\nu
_{2}\big\rangle $;

\item $\lambda _{_{\mathrm{Lin}}}\cdot \alpha =\big\langle 1-(1-\mu
)^{\lambda },(1-\mu )^{\lambda }-(1-\mu -\eta )^{\lambda },(1-\mu )^{\lambda
}-(1-\mu -\nu )^{\lambda }\big\rangle $;

\item $\alpha ^{\lambda _{\mathrm{Lin}}}=\big\langle (1-\nu )^{\lambda
}-(1-\mu -\nu )^{\lambda },1-(1-\eta )^{\lambda },1-(1-\nu )^{\lambda
}\big\rangle $;
\end{enumerate}

where $\lambda\in \mathbb{N}$.
\end{definition}

\begin{theorem}[{\textrm{\protect\cite[Theorems~1 and 2]{LLCFL2021}}}]
\label{LLCFL-Thm} Let $\alpha_1$ and $\alpha_2$ be two PFNs. Then, $%
\alpha_1\oplus_{_{\mathrm{Lin}}}\alpha_2$ and $\alpha_1\otimes_{_{\mathrm{Lin%
}}}\alpha_2$ are PFNs.
\end{theorem}

\begin{example}
\label{LLCFL>1} Let $\alpha_1=\alpha_2=\left\langle 0, \frac{1}{2}, \frac{1}{%
2}\right\rangle$. For $\lambda\geq 2$, by Definition~\ref{LLCFL-Def}, we
have
$
\alpha_1\oplus_{_{\mathrm{Lin}}} \alpha_2= \alpha_1\otimes_{_{\mathrm{Lin}}}
\alpha_2=\left\langle 0, \frac{3}{4}, \frac{3}{4}\right\rangle
$
and
$
\lambda_{_{\mathrm{Lin}}}\cdot \alpha_1=(\alpha_1)^{\lambda_{_{\mathrm{Lin}%
}}}=\left\langle 0, 1-\frac{1}{2^{\lambda}}, 1-\frac{1}{2^{\lambda}}
\right\rangle.
$
Clearly, $\alpha_1\otimes_{_{\mathrm{Lin}}} \alpha_2\notin \mathscr{P}$ and $%
\lambda_{_{\mathrm{Lin}}}\cdot \alpha_1\notin \mathscr{P}$ since $0+\frac{3}{%
4}+\frac{3}{4}=\frac{3}{2}>1$ and $0+(1-\frac{1}{2^{\lambda}}) +(1-\frac{1}{%
2^{\lambda}})\geq \frac{3}{2}>1$ ($\lambda\geq 2$). Thus, none of the operations in Definition~\ref{LLCFL-Def} (\cite[%
Definition~10]{LLCFL2021}) is closed in $\mathscr{P}$. Therefore, Theorem~%
\ref{LLCFL-Thm} (\cite[Theorems~1 and 2]{LLCFL2021}) does not hold.
\end{example}

\subsection{PFMM operator in \protect\cite%
{XSWWZLX2019,WWGW2019}}

Based on the Muirhead mean operator \cite{M1902}, Xu et al.~\cite{XSWWZLX2019}
introduced the following PFMM operator.

\begin{definition}[{\textrm{\protect\cite[Definition~6]{XSWWZLX2019}}}]
\label{PFMM} Let $\alpha_{j}\in \mathscr{P}$ ($j=1, \ldots, n$)
and $P=(p_1, \ldots, p_n)\in \mathbb{R}^{n}$ be a vector of parameters.
Define the \textit{Picture fuzzy Muirhead mean} PFMM operator by
\begin{small}
$$
\mathrm{PFMM}^{P}(\alpha_1, \ldots, \alpha_n)= \Bigg(\frac{1}{n!}%
\bigoplus_{\vartheta\in S_{n}}\bigotimes_{j=1}^{n}
\alpha_{\vartheta(j)}^{p_j}\Bigg)^{\frac{1}{\sum\limits_{j=1}^{n}p_{j}}},
$$
\end{small}
where $(\vartheta(1), \ldots, \vartheta(n))$
 is any permutation of $(1, \ldots, n)$, $S_{n}$ is the set
 of all permutations of $(1, \ldots, n)$, and $\oplus$ and $\otimes$ are given by Definition \ref%
{Wei-Def}.
\end{definition}

According to Definition~\ref{PFMM}, the following formula for PFMM was
obtained by \cite{XSWWZLX2019}.

\begin{theorem}[{\textrm{\protect\cite[Theorem~1]{XSWWZLX2019}}}]
\label{XSWWZLX-Thm}
Let $\alpha_{j}\in \mathscr{P}$ ($j=1, \ldots, n$)
and $P=(p_1, \ldots, p_n)\in \mathbb{R}^{n}$.
Then, $\mathrm{PFMM}^{P}(\alpha_1, \ldots, \alpha_n)\in \mathscr{P}$ and
\begin{small}
\begin{equation}
\label{PFMM-eq-1}
\begin{split}
&\mathrm{PFMM}^{P}(\alpha_1, \ldots, \alpha_n) \\
=&\left\langle \Bigg(1-\prod_{\vartheta\in
S_{n}}\Bigg(1-\prod_{j=1}^{n}\mu_{\vartheta(j)}^{p_j}
\Bigg)^{\frac{1}{n!}}\Bigg)^{\frac{1}{\sum\limits_{j=1}^{n}p_{j}}},\right. \\
& 1- \Bigg(1-\prod_{\vartheta\in
S_{n}}\Bigg(1-\prod_{j=1}^{n}(1-\eta_{\vartheta(j)})^{p_j}
\Bigg)^{\frac{1}{n!}}\Bigg)^{\frac{1}{\sum\limits_{j=1}^{n}p_{j}}}, \\
& \left. 1- \Bigg(1-\prod_{\vartheta\in
S_{n}}\Bigg(1-\prod_{j=1}^{n}(1-\nu_{\vartheta(j)})^{p_j}
\Bigg)^{\frac{1}{n!}}\Bigg)^{\frac{1}{\sum\limits_{j=1}^{n}p_{j}}} \right\rangle.
\end{split}%
\end{equation}
\end{small}
\end{theorem}

\begin{example}
\label{Exm-3.7} Take $\alpha _{1}=\langle 0, 0.9, 0.1 \rangle$, $\alpha
_{2}=\langle 0, 0.1, 0.9 \rangle\in \mathscr{P}$, and $P=(0.5,0.5)\in
\mathbb{R}^{2}$. By formula~\eqref{PFMM-eq-1}, we have
$\mathrm{PFMM}^{P}(\alpha _{1},\alpha _{2})
=\langle0,0.7,0.7 \rangle\notin \mathscr{P}$.
This means that \cite[Theorem~1]{XSWWZLX2019} does not hold.
\end{example}

Taking into account the attribute's weight, Wang et al.~\cite{WWGW2019}
introduced the PFWMM operator as follows.

\begin{definition}[{\textrm{\protect\cite[Definition~6]{WWGW2019}}}]
\label{PFWMM} Let $\omega =(\omega_{1}, \ldots ,\omega
_{n})^{\top }$ be the weight vector such that $\omega_{i}\in (0,1]$ and $%
\sum_{j=1}^{n}\omega _{j}=1$, $\alpha_{j}\in \mathscr{P}$ ($j=1, \ldots, n$),
and $P=(p_{1}, \ldots, p_{n})\in \mathbb{R}^{n}$ be
a vector of parameters. Define the \textit{picture fuzzy weighted Muirhead
mean} (PFWMM) operator by
\begin{small}
$$
\mathrm{PFWMM}_{n\omega }^{P}(\alpha _{1}, \ldots ,\alpha
_{n})=\left[\frac{1}{n!}\bigoplus_{\vartheta \in
S_{n}}\bigotimes_{j=1}^{n}\left( n\omega _{\vartheta (j)}\alpha _{\vartheta
(j)}\right) ^{p_{j}}\right]^{\frac{1}{\sum\limits_{j=1}^{n}p_{j}}},
$$\end{small}
where $(\vartheta(1), \ldots,$ $\vartheta(n))$
 is any permutation of $(1, \ldots, n)$, $S_{n}$ is the set
 of all permutations of $(1, \ldots, n)$, and $\oplus$ and $\otimes$ are given by Definition \ref%
{Wei-Def}.
\end{definition}

\begin{example}[\textrm{Continuation of Example~\protect\ref{Exm-3.7}}]
\label{Exm-3.7.1} Given $\alpha _{1}$, $\alpha _{2}$, and $P$ as in Example~%
\ref{Exm-3.7}, we choose $\omega =(0.5,0.5)^{\top }$. By direct calculation and
Example~\ref{Exm-3.7}, we have $\mathrm{PFWMM}^{P}_{n\omega}(\alpha_1,
\alpha_2)= \mathrm{PFMM}^{P}(\alpha_1, \alpha_2)=\langle 0, 0.7, 0.7%
\rangle\notin \mathscr{P}$.
\end{example}

\begin{remark}
(1) Example~\ref{Exm-3.7} also shows that \cite[Theorem~2]{WWGW2019} does
not hold.

(2) Example~\ref{Exm-3.7.1} shows that \cite[Theorem~6]{WWGW2019} does not
hold.
\end{remark}

\subsection{PFBM operator in \protect\cite{AA2020-1}%
}

Ate\c{s} and Akay~\cite{AA2020-1} introduced the following
Bonferroni mean operator for PFNs.

\begin{definition}[{\textrm{\protect\cite[Definition~3.1]{AA2020-1}}}]
\label{AA-Def}
Let $\alpha _{j}\in \mathscr{P}$ ($j=1, \ldots ,n$).
Then, for any $p$, $q\geq 0$, define the \textit{picture fuzzy Bonferroni mean} (PFBM)
operator $\mathrm{PFB}^{p,q}$ by
\begin{small}
$$
\mathrm{PFB}^{p, q}(\alpha_1, \ldots, \alpha_n)
=\left[\frac{1}{%
n(n-1)}\Bigg(\bigoplus_{i, j=1,\atop  i\neq j}^{n}(\alpha_{i}^{p}\otimes
\alpha_{j}^{q}) \Bigg)\right]^{\frac{1}{p+q}},
$$
\end{small}
where $\oplus$ and $\otimes$ are given by Definition \ref{Wei-Def}.
\end{definition}

According to Definition~\ref{AA-Def}, the following formula for PFBM was
obtained by \cite{AA2020-1}.

\begin{theorem}[{\textrm{\protect\cite[Theorem~1]{AA2020-1}}}]
\label{AA-Thm} Let $\alpha_{j}=\langle \mu_j, \eta_j, \nu_j\rangle\in \mathscr{P}$ (%
$j=1, \ldots, n$) and $p$, $q\geq 0$. Then,
$\mathrm{PFBM}^{p, q}(\alpha_1, \ldots, \alpha_n)\in \mathscr{P}$ and
\begin{small}
\begin{equation}
\label{AA-eq-1}
\begin{split}
&\mathrm{PFBM}^{p, q}(\alpha_1, \ldots, \alpha_n) \\
= & \left\langle \Bigg(1-\prod_{i, j=1,\atop  i\neq j}^{n}(1-\mu_i^{p}\mu_{j}^{q}) ^{%
\frac{1}{n(n-1)}}\Bigg)^{\frac{1}{p+q}},\right. \\
& 1-\Bigg(1-\prod_{i, j=1,\atop  i\neq
j}^{n}\left(1-(1-\eta_i)^{p}(1-\eta_j)^{q} \right)^{\frac{1}{n(n-1)}%
}\Bigg)^{\frac{1}{p+q}}, \\
& \left.1-\Bigg(1-\prod_{i, j=1,\atop  i\neq
j}^{n}\left(1-(1-\nu_i)^{p}(1-\nu_j)^{q} \right)^{\frac{1}{n(n-1)}}\Bigg)^{%
\frac{1}{p+q}} \right \rangle.
\end{split}%
\end{equation}
\end{small}
\end{theorem}

\begin{example}
\label{Exm-3.8} Take $\alpha_1=\langle 0, 0.9, 0.1 \rangle$, $%
\alpha_2=\langle 0, 0.1, 0.9 \rangle
\in \mathscr{P}$, and $p=q=0.5$. By formula~\eqref{AA-eq-1}, we have
$\mathrm{PFBM}^{p, q}(\alpha_1, \alpha_2)=\langle 0, 0.7, 0.7 \rangle\notin \mathscr{P}$.
This means that \cite[Theorem~1]{AA2020-1} does not hold.
\end{example}

\subsection{PFNWBM operator in \protect\cite{ZTLJZ2020}}

Based on the operations in Definition~\ref{Wei-Def}, Zhang et al.~\cite%
{ZTLJZ2020} introduced the following normalized weighted
Bonferroni mean operator for PFNs.

\begin{definition}[{\textrm{\protect\cite[Definition~8]{ZTLJZ2020}}}]
\label{ZTLJZ-Def} Let $\alpha _{j}\in \mathscr{P}$ ($j=1, \ldots ,n$),
and $\omega _{j}$ ($j=1, \ldots ,n$) be the associated weight of $%
\alpha _{j}$ such that $\omega _{j}\in (0,1]$ and $\sum_{j=1}^{n}\omega
_{j}=1$. Then, for any $p$, $q>0$, define the \textit{picture fuzzy normalized weighted
Bonferroni mean} (PFNWBM) operator $\mathrm{PFB}^{p,q}$ by
\begin{small}
$$
\mathrm{PFNWBM}_{\omega }^{p,q}(\alpha _{1}, \ldots ,\alpha
_{n})=\left[\bigoplus_{i,j=1, \atop i\neq j}^{n}\left( \frac{\omega _{i}\omega _{j}%
}{1-\omega _{i}}(\alpha _{i}^{p}\otimes \alpha _{j}^{q})\right) \right]^{%
\frac{1}{p+q}},
$$
\end{small}
where the operators $\oplus $ and $\otimes $ are given by Definition \ref%
{Wei-Def}.
\end{definition}

According to Definition~\ref{ZTLJZ-Def}, the following formula for PFBM was
obtained by \cite{ZTLJZ2020}.

\begin{theorem}[{\textrm{\protect\cite[Theorem~2]{ZTLJZ2020}}}]
\label{AA-Thm} Let $p$, $q> 0$ and $\alpha_{j}=\big\langle \mu_j, \eta_j, \nu_j%
\big\rangle\in \mathscr{P}$ ($j=1, \ldots, n$), whose weight
vector is $\omega=(\omega_1, \ldots, \omega_n)^{\top}$ such that $%
\omega_{j}\in (0, 1]$ and $\sum_{j=1}^{n}\omega_{j}=1$. Then,
$\mathrm{PFNWBM}_{\omega}^{p, q}(\alpha_1, \ldots, \alpha_n)\in \mathscr{P}$ and
\begin{small}
\begin{equation}
\label{ZTLJZ-eq-1}
\begin{split}
& \mathrm{PFNWBM}_{\omega}^{p, q}(\alpha_1, \ldots, \alpha_n)
\\
= & \left\langle \Bigg(1-\prod_{i, j=1, \atop i\neq j}^{n}(1-\mu_i^{p}\mu_{j}^{q}) ^{%
\frac{\omega_i\omega_j}{1-\omega_i}}\Bigg)^{\frac{1}{p+q}},\right. \\
&1-\Bigg(1-\prod_{i, j=1, \atop i\neq
j}^{n}\left(1-(1-\eta_i)^{p}(1-\eta_j)^{q} \right)^{\frac{\omega_i\omega_j}{%
1-\omega_i}}\Bigg)^{\frac{1}{p+q}}, \\
&\left.1-\Bigg(1-\prod_{i, j=1, \atop i\neq
j}^{n}\left(1-(1-\nu_i)^{p}(1-\nu_j)^{q} \right)^{\frac{\omega_i\omega_j}{%
1-\omega_i}}\Bigg)^{\frac{1}{p+q}} \right\rangle.
\end{split}%
\end{equation}
\end{small}
\end{theorem}

\begin{example}
\label{Exm-3.8-1} Take $\alpha_1=\langle 0, 0.9, 0.1 \rangle$, $%
\alpha_2=\langle 0, 0.1, 0.9 \rangle
\in \mathscr{P}$, $p=q=0.5$, and $\omega=(\omega_1, \omega_2)^{\top}= (0.5,
0.5)^{\top}$. By formula~\eqref{ZTLJZ-eq-1}, we have
$\mathrm{PFNWBM}_{\omega}^{p, q}(\alpha_1, \alpha_2)=\langle 0, 0.7, 0.7\rangle\notin \mathscr{P}$.
This means that \cite[Theorem~2]{ZTLJZ2020} does not hold.
\end{example}

\section{Interactional operational laws for PFNs}
\label{S-5}

Inspired by the interactional operations proposed by Wang et al.~\cite[%
Definition 3.1]{WZTT2017} and Ju et al.~\cite[Definition~4]{JJGGW2019}, we
introduce the following interaction operational laws for PFNs.

%%\begin{definition}
%%\label{Def-ope}
%%Let $\alpha_1=\langle\mu_1, \eta_1, \nu_1\rangle$, $\alpha_2=\langle\mu_2, \eta_2, \nu_2\rangle$,
%%and $\alpha=\langle \mu, \eta, \nu\rangle$ be three PFNs,
%%$\tau$ be an AG of a strict t-norm $T$, and $\zeta(x)=\tau(1-x)$.
%%Define
%%\begin{enumerate}[(i)]
%%  \item $\alpha_1\tilde{\oplus} \alpha_2=\left\langle \zeta^{-1}(\zeta(\mu_1)+\zeta(\mu_2)),
%%  \tau^{-1}(\tau(\eta_1)+\tau(\eta_2)), \tau^{-1}(\tau(\eta_1+\nu_1)+\tau(\eta_2+\nu_2))-
%%  \tau^{-1}(\tau(\eta_1)+\tau(\eta_2))\right\rangle$;
%%  \item $\alpha_1\tilde{\otimes} \alpha_2=\left\langle \tau^{-1}(\tau(\mu_1+\eta_1)
%%  +\tau(\mu_2+\eta_2))-\tau^{-1}(\tau(\eta_1)+\tau(\eta_2)),
%%  \tau^{-1}(\tau(\eta_1)+\tau(\eta_2)), \zeta^{-1}(\zeta(\nu_1)+\zeta(\nu_2))\right\rangle$;
%%  \item $\tilde{\lambda} \alpha=\left\langle \zeta^{-1}(\lambda\zeta(\mu)),
%%  \tau^{-1}(\lambda \tau(\eta)), \tau^{-1}(\lambda\tau(\eta+\nu))-\tau^{-1}(\lambda \tau(\eta))
%%  \right\rangle$, $\lambda>0$;
%%  \item $\alpha^{\tilde{\lambda}}=\left\langle \tau^{-1}(\lambda\tau(\mu)), \tau^{-1}
%%  (\lambda\tau(\eta+\mu))-\tau^{-1}(\lambda\tau(\mu)),
%%  \zeta^{-1}(\lambda\zeta(\nu))\right\rangle$, $\lambda>0$.
%%\end{enumerate}
%%\end{definition}

\begin{definition}
\label{Def-Wu}
Let $\tau $ be an AG of a strict t-norm $T$ and $\zeta $ be an AG of
the dual t-conorm $S$ of $T;$ i.e., $\zeta (u)=\tau (1-u)$.
For $\alpha _{1}=\langle \mu _{1},\eta _{1},\nu _{1}\rangle $%
, $\alpha _{2}=\langle \mu _{2},\eta _{2},\nu _{2}\rangle $, $\alpha
=\langle \mu ,\eta ,\nu \rangle \in \mathscr{P}$, define
\begin{enumerate}[(i)]
\item $\alpha ^{\complement }=\left\langle \nu ,\eta ,\mu \right\rangle $;

\item $\alpha _{1}\oplus _{_{T}}\alpha _{2}=\langle S(\mu _{1},\mu
_{2}),T(\eta _{1}+\nu _{1},\eta _{2}+\nu _{2})-T(\nu _{1},\nu _{2}),T(\nu
_{1},\nu _{2})\rangle $;

\item $\alpha _{1}\otimes _{_{T}}\alpha _{2}=
\langle T(\mu _{1},\mu_{2}),T(\eta _{1}+\mu _{1},\eta _{2}
+\mu _{2})-T(\mu _{1},\mu _{2}),S(\nu
_{1},\nu _{2})\rangle $;

\item $\lambda _{_{T}}\cdot \alpha =
\langle \zeta ^{-1}(\lambda \cdot
\zeta (\mu )),\tau ^{-1}(\lambda \cdot \tau (\eta +\nu ))-\tau ^{-1}(\lambda
\cdot \tau (\nu )),\tau ^{-1}(\lambda \cdot \tau (\nu ))\rangle$, $\lambda >0$;

\item $\alpha ^{\lambda _{_{T}}}=\langle \tau ^{-1}(\lambda \cdot \tau
(\mu )),\tau ^{-1}(\lambda \cdot \tau (\eta +\mu ))-\tau ^{-1}(\lambda \cdot
\tau (\mu )),\zeta ^{-1}(\lambda \cdot \zeta (\nu ))\rangle $, $%
\lambda >0$.
\end{enumerate}
%%where $\zeta(x)=\tau(1-x)$, which is an AG of $S$.
\end{definition}

\begin{remark}
(i) When $\eta _{1}=\eta _{2}=0$ and $T=T_{\textbf{P}}$ (product t-norm), the operational
laws in Definition~\ref{Def-Wu} are consistent with the usual operations on
ordinary intuitionistic fuzzy numbers (see \cite[Definition 1.1]{Ata1999}
and \cite[Definition~1.2.2]{XC2012}).

(ii) When $T=T_{\textbf{P}}$, the operational laws $\alpha _{1}\oplus
_{_{T_{\textbf{P}}}}\alpha _{2}$ and $\lambda _{_{T_{\textbf{P}}}}\cdot
\alpha $ in Definition~\ref{Def-Wu} reduce to the one in \cite[Definition~8]%
{PWC2021}.
\end{remark}

The following Theorem \ref{Closed-Thm} demonstrates that all operational
laws defined in Definition~\ref{Def-Wu} are closed in $\mathscr{P}$.

\begin{theorem}
\label{Closed-Thm} Let $\alpha_1$, $\alpha_2$, $\alpha\in \mathscr{P}$
and $T$ be a strict t-norm. Then, for any $%
\lambda>0$, $\alpha_1\oplus_{_T} \alpha_2$, $\alpha_1\otimes_{_T}
\alpha_2$, $\lambda_{_T}\cdot \alpha$, and $\alpha^{\lambda_{_T}}$ are PFNs.
\end{theorem}

\begin{proof}
(i) By Definition~\ref{Def-Wu}, we have
$\alpha_1\oplus_{_{T}} \alpha_2=\big\langle S(\mu_1, \mu_2), T(\eta_1+\nu_1,
\eta_2+\nu_2)-T(\nu_1, \nu_2), T(\nu_1, \nu_2)\big\rangle$.
This, together with $\mu_1\leq 1-(\eta_1+\nu_1)$ and $\mu_2\leq
1-(\eta_2+\nu_2)$, since $S$ is increasing, implies that
$S(\mu_1, \mu_2)+\left[T(\eta_1+\nu_1, \eta_2+\nu_2)-T(\nu_1, \nu_2)%
\right] +T(\nu_1, \nu_2)=S(\mu_1, \mu_2)+T(\eta_1+\nu_1, \eta_2+\nu_2)
\leq S(1-(\eta_1+\nu_1), 1-(\eta_2+\nu_2))+T(\eta_1+\nu_1, \eta_2+\nu_2) =1.
$
Thus $\alpha_1\oplus_{_{T}} \alpha_2\in \mathscr{P}$. Similarly, we can
prove that $\alpha_1\otimes_{_{T}} \alpha_2\in \mathscr{P}$.

(ii) Fix $\lambda >0$. By Definition~\ref{Def-Wu}, we have
$\lambda_{_{T}}\cdot \alpha=\langle \zeta^{-1}(\lambda\cdot \zeta(\mu)),
\tau^{-1}(\lambda\cdot \tau(\eta+\nu))-\tau^{-1}(\lambda\cdot \tau(\nu)),
\tau^{-1}(\lambda\cdot \tau(\nu))\rangle$.
Since $\tau $ is decreasing and $1-\mu \geq \eta +\nu $, we get
$\zeta^{-1}(\lambda\cdot \zeta(\mu))+\left[ \tau^{-1}(\lambda\cdot
\tau(\eta+\nu))-\tau^{-1}(\lambda \cdot\tau(\nu))\right]+
\tau^{-1}(\lambda\cdot \tau(\nu))
=\zeta^{-1}(\lambda\cdot \zeta(\mu))+ \tau^{-1}(\lambda\cdot
\tau(\eta+\nu))=1-\tau^{-1}(\lambda\cdot \tau(1-\mu))
+\tau^{-1}(\lambda\cdot \tau(\eta+\nu))
\leq 1-\tau^{-1}(\lambda\cdot \tau(1-\mu))+\tau^{-1}(\lambda\cdot
\tau(1-\mu))=1$.
Thus, $\lambda _{_{T}}\cdot \alpha \in \mathscr{P}$. Similarly, we can prove
that $\alpha ^{\lambda _{_{T}}}\in \mathscr{P}$.
\end{proof}

\begin{theorem}
\label{Pro-Them}
Let $\alpha$, $\beta$, $\gamma\in \mathscr{P}$ and $T$ be a
strict t-norm. Then, for any $\lambda$, $\xi >0$, we have
\begin{enumerate}[{\rm (i)}]

\item $\beta \oplus _{_T}\alpha =\alpha \oplus_{_{T}}\beta$;

\item $\beta \otimes _{_T}\alpha=\alpha \otimes _{_{T}}\beta$;

\item $(\alpha \oplus _{_{T}}\beta )\oplus _{_{T}}\gamma =\alpha \oplus
_{_{T}}(\beta \oplus _{_{T}}\gamma )$;

\item $(\alpha \otimes _{_{T}}\beta )\otimes _{_{T}}\gamma =\alpha \otimes
_{_{T}}(\beta \otimes _{_{T}}\gamma )$;

\item $(\xi _{_{T}}\alpha )\oplus _{_{T}}(\lambda _{_{T}}\alpha )=(\xi
+\lambda )_{_{T}}\alpha $;

\item $(\alpha ^{\xi _{_{T}}})\otimes _{_{T}}(\alpha ^{\lambda
_{_{T}}})=\alpha ^{(\xi +\lambda )_{_{T}}}$;

\item $\lambda _{_{T}}(\alpha \oplus _{_{T}}\beta )=(\lambda _{_{T}}\alpha
)\oplus _{_{T}}(\lambda _{_{T}}\beta )$;

\item $(\alpha \otimes _{_{T}}\beta )^{\lambda _{_{T}}}=\alpha ^{\lambda
_{_{T}}}\otimes _{_{T}}\beta ^{\lambda _{_{T}}}$;

\item $\xi _{_{T}}(\lambda _{_{T}}\alpha )=(\lambda \cdot \xi )_{_{T}}\alpha
$;

\item $(\alpha^{\lambda_{_T}})^{\xi_{_T}}=\alpha^{(\lambda\cdot \xi)_{_T}}$.
\end{enumerate}
\end{theorem}

\begin{theorem}
\label{N-Thm} Let $\alpha_j=\langle\mu_j, \eta_j, \nu_j\rangle\in \mathscr{P}
$ ($j=1, \ldots, n$) and $T$ be a strict t-norm. Then,
\begin{equation}
\label{eq-Wu-1.1}
\begin{split}
&\quad \alpha_1\oplus_{_{T}}\alpha_2\oplus_{_{T}}\cdots \oplus_{_{T}}\alpha_n\\
&=\left\langle S^{(n)}(\mu_1, \ldots, \mu_n), T^{(n)}(\eta_1+\nu_1,
\ldots, \eta_n+\nu_n)\right. \\
&\quad \quad \left.-T^{(n)}(\nu_1, \ldots, \nu_n), T^{(n)}(\nu_1,
\ldots, \nu_n)\right\rangle,
\end{split}%
\end{equation}
and
\begin{equation}
\label{eq-Wu-1.2}
\begin{split}
&\quad \alpha_1\otimes_{_{T}}\alpha_2\otimes_{_{T}}\cdots
\otimes_{_{T}}\alpha_n \\
&=\left\langle T^{(n)}(\mu_1, \ldots, \mu_n), T^{(n)}(\eta_1+\mu_1,
\ldots, \eta_n+\mu_n)\right. \\
&\quad \quad \left.-T^{(n)}(\mu_1, \ldots, \mu_n), S^{(n)}(\nu_1,
\ldots, \nu_n)\right\rangle.
\end{split}%
\end{equation}
\end{theorem}

\section{Picture fuzzy interactional aggregation operators}
\label{S-6}

Based on the interactional operational laws defined in Section~\ref%
{S-5}, this section is devoted to aggregating the picture fuzzy
information.

\subsection{Picture fuzzy interactional weighted average (geometric) operator%
}

\begin{definition}
\label{PFWA-Def} Let $\omega=(\omega_1, \ldots, \omega_n)^{\top}$
be the weight vector such that $\omega_{j}\in (0, 1]$ and $%
\sum_{j=1}^{n}\omega_j=1$ and $T$ be a strict t-norm. Define the \textit{%
picture fuzzy interactional weighted average operator} $\mathrm{PFIWA}_{T,
\omega}$ and \textit{picture fuzzy interactional weighted geometric operator}
$\mathrm{PFIWG}_{T, \omega}$ induced by $T$ as
\begin{equation}
\label{eq-PFWA-Ope}
\begin{split}
\mathrm{PFIWA}_{T, \omega}: \mathscr{P}^{n}&\rightarrow \mathscr{P} \\
(\alpha_1, \ldots, \alpha_n)&\mapsto
(\omega_1)_{_T}\alpha_1\oplus_{_T} \cdots
\oplus_{_T} (\omega_n)_{_T}\alpha_n,
\end{split}%
\end{equation}%
and
\begin{equation}
\label{eq-PFWG-Ope}
\begin{split}
\mathrm{PFIWG}_{T, \omega}: \mathscr{P}^{n}&\rightarrow \mathscr{P} \\
(\alpha_1, \ldots, \alpha_n)&\mapsto
\alpha_1^{(\omega_1)_{_T}}\otimes_{_T}\cdots \otimes_{_T} \alpha_n^{(\omega_n)_{_T}},
\end{split}%
\end{equation}%
respectively.
\end{definition}

\begin{theorem}
\label{PFWA-PFGA-Thm} Let $\alpha_{j}=(\mu_{j}, \eta_{j}, \nu_{j})\in \mathscr{P}$ ($j=1,
\ldots, n$), $\omega=(\omega_1, \ldots, \omega_n)^{\top}$ be the weight vector such that $\omega_{j}\in (0, 1]$ and
$\sum_{j=1}^{n}\omega_j=1$, and $T$ be a strict t-norm with an AG $\tau$. Then,
\begin{equation}
\label{eq-PFWA-1}
\begin{split}
&\quad \mathrm{PFIWA}_{T, \omega} (\alpha_1, \ldots, \alpha_n) \\
&= \left\langle \zeta^{-1}(\omega_1\cdot \zeta(\mu_1)+ \cdots +\omega_n\cdot
\zeta(\mu_n)), \right.\\
&\quad \quad \tau^{-1}(\omega_1 \cdot \tau(\eta_1+\nu_1)+\cdots +\omega_n
\cdot \tau(\eta_n+\nu_n))\\
&\quad \quad -\tau^{-1}(\omega_1\cdot \tau(\nu_1)+\cdots
+\omega_n\cdot \tau(\nu_n)), \\
&\quad \quad \left.\tau^{-1}(\omega_1\cdot \tau(\nu_1)+\cdots +\omega_n\cdot
\tau(\nu_n)) \right\rangle,
\end{split}
\end{equation}
and
\begin{equation}
\label{eq-PFWG-2}
\begin{split}
&\quad \mathrm{PFIWG}_{T, \omega} (\alpha_1, \ldots, \alpha_n) \\
&= \left\langle \tau^{-1}(\omega_1\cdot \tau(\mu_1)+ \cdots +\omega_n\cdot
\tau(\mu_n)),\right. \\
&\quad \quad \tau^{-1}(\omega_1 \cdot \tau(\eta_1+\mu_1)+\cdots +\omega_n
\cdot \tau(\eta_n+\mu_n)) \\
&\quad \quad -\tau^{-1}(\omega_1\cdot \tau(\mu_1)+\cdots
+\omega_n\cdot \tau(\mu_n)), \\
&\quad \quad \left.\zeta^{-1}(\omega_1\cdot \zeta(\nu_1)+\cdots +\omega_n\cdot
\zeta(\nu_n)) \right\rangle,
\end{split}
\end{equation}
where $\zeta(u)=\tau(1-u)$.
\end{theorem}

If we assign specific forms of t-norms to Theorem~\ref{PFWA-PFGA-Thm}, we
have the following results.

(1) If we take $T=T_{\textbf{P}}$ (product t-norm), then

i) $\mathrm{PFIWA}_{T_{\textbf{P}},\omega }(\alpha
_{1},\ldots ,\alpha _{n})=\Big\langle 1-\prod_{j=1}^{n}(1-\mu _{j})^{\omega
_{j}},\prod_{j=1}^{n}(\eta _{j}+\nu _{j})^{\omega _{j}}-\prod_{j=1}^{n}\nu
_{j}^{\omega _{j}},\prod_{j=1}^{n}\nu _{j}^{\omega _{j}} \Big\rangle ,$

ii) $\mathrm{PFIWG}_{T_{\textbf{P}},\omega }(\alpha
_{1},\ldots ,\alpha _{n})
=\Big\langle \prod_{j=1}^{n}\mu _{j}^{\omega _{j}},\prod_{j=1}^{n}(\eta
_{j}+\mu _{j})^{\omega _{j}}-\prod_{j=1}^{n}\mu _{j}^{\omega
_{j}},1-\prod_{j=1}^{n}(1-\nu _{j})^{\omega _{j}}\Big\rangle.$

(2) If we take $T$ as Schweizer-Sklar t-norms $T_{\gamma }^{\textbf{SS}}$ (%
$\gamma \in (-\infty ,0)$), then

i) $\mathrm{PFIWA}_{T_{\gamma }^{\textbf{SS}},\omega
}(\alpha _{1},\ldots ,\alpha _{n})\newline
=\Big\langle 1-\Big( \sum\limits_{j=1}^{n}\omega _{j}(1-\mu
_{j})^{\gamma }\Big) ^{1/\gamma },\Big( \sum\limits_{j=1}^{n}\omega _{j}(\eta
_{j}+\nu _{j})^{\gamma }\Big) ^{1/\gamma }-\Big( \sum\limits_{j=1}^{n}\omega
_{j}\nu _{j}^{\gamma }\Big) ^{1/\gamma },\Big( \sum\limits_{j=1}^{n}\omega
_{j}\nu _{j}^{\gamma }\Big) ^{1/\gamma }\Big\rangle,$

ii) $\mathrm{PFIWG}_{T_{\gamma }^{\textbf{SS}},\omega
}(\alpha _{1},\ldots ,\alpha _{n})\newline
=
\Big\langle \Big( \sum\limits_{j=1}^{n}\omega _{j}\mu _{j}^{\gamma
}\Big) ^{1/\gamma },\Big( \sum\limits_{j=1}^{n}\omega _{j}(\eta _{j}+\mu
_{j})^{\gamma }\Big) ^{1/\gamma }-\Big( \sum\limits_{j=1}^{n}\omega _{j}\mu
_{j}^{\gamma }\Big) ^{1/\gamma },1-\Big( \sum\limits_{j=1}^{n}\omega _{j}(1-\nu
_{j})^{\gamma }\Big) ^{1/\gamma }\Big\rangle.$

(3) If we take $T$ as Hamacher t-norms $T_{\gamma }^{\textbf{H}}$ ($\gamma
\in (0,+\infty )$), then

i) $\mathrm{PFIWA}_{T_{\gamma }^{\textbf{H}},\omega
}(\alpha _{1},\ldots ,\alpha _{n})\newline
=\Bigg\langle\frac{\prod\limits_{j=1}^{n}\left( \frac{\gamma +(1-\gamma )(1-\mu
_{j})}{1-\mu _{j}}\right) ^{\omega _{j}}-1}{\prod\limits_{j=1}^{n}\left( \frac{%
\gamma +(1-\gamma )(1-\mu _{j})}{1-\mu _{j}}\right) ^{\omega _{j}}-1+\gamma }%
,\frac{\gamma }{\prod\limits_{j=1}^{n}\left( \frac{\gamma +(1-\gamma )(\eta
_{j}+\nu _{j})}{\eta _{j}+\nu _{j}}\right) ^{\omega _{j}}-1+\gamma }\newline
~~~~ -\frac{%
\gamma }{\prod\limits_{j=1}^{n}\left( \frac{\gamma +(1-\gamma )\nu _{j}}{\nu _{j}}%
\right) ^{\omega _{j}}-1+\gamma },
\frac{\gamma }{\prod\limits_{j=1}^{n}\left( \frac{\gamma +(1-\gamma )\nu _{j}}{%
\nu _{j}}\right) ^{\omega _{j}}-1+\gamma }\Bigg\rangle,$

ii) $\mathrm{PFIWG}_{T_{\gamma }^{\textbf{H}},\omega
}(\alpha _{1},\ldots ,\alpha _{n})\newline
=\Bigg\langle\frac{\gamma }{\prod\limits_{j=1}^{n}\left( \frac{\gamma +(1-\gamma
)\mu _{j}}{\mu _{j}}\right) ^{\omega _{j}}-1+\gamma },
\frac{\gamma }{\prod\limits_{j=1}^{n}\left( \frac{\gamma +(1-\gamma )(\eta
_{j}+\mu _{j})}{\eta _{j}+\mu _{j}}\right) ^{\omega _{j}}-1+\gamma }\newline
~~~~ -\frac{%
\gamma }{\prod\limits_{j=1}^{n}\left( \frac{\gamma +(1-\gamma )\mu _{j}}{\mu _{j}}%
\right) ^{\omega _{j}}-1+\gamma },
\frac{\prod\limits_{j=1}^{n}\left( \frac{\gamma +(1-\gamma )(1-\nu _{j})}{1-\nu
_{j}}\right) ^{\omega _{j}}-1}{\prod\limits_{j=1}^{n}\left( \frac{\gamma +(1-\gamma
)(1-\nu _{j})}{1-\nu _{j}}\right) ^{\omega _{j}}-1+\gamma }\Bigg\rangle.$

(4) If we take $T$ as Frank t-norms $T_{\gamma}^{\textbf{F}}$ ($\gamma\in
(0, 1)\cup (1, +\infty)$), then

i) $\mathrm{PFIWA}_{T_{\gamma }^{\textbf{F}},\omega
}(\alpha _{1},\ldots ,\alpha _{n})
=\Big\langle1-\log _{\gamma }\Big( \prod\limits_{j=1}^{n}(\gamma ^{1-\mu
_{j}}-1)^{\omega _{j}}+1\Big) ,
\log _{\gamma }\Big( \prod\limits_{j=1}^{n}\Big( \gamma ^{\eta _{j}+\nu
_{j}}-1\Big) ^{\omega _{j}}+1\Big) -\log _{\gamma }\Big(
\prod\limits_{j=1}^{n}\Big( \gamma ^{\nu _{j}}-1\Big) ^{\omega _{j}}+1\Big) ,%
\log _{\gamma }\Big( \prod\limits_{j=1}^{n}\Big( \gamma ^{\nu _{j}}-1\Big)
^{\omega _{j}}+1\Big) \Big\rangle,$

ii) $\mathrm{PFIWG}_{T_{\gamma }^{\textbf{F}},\omega
}(\alpha _{1},\ldots ,\alpha _{n})
=\Big\langle\log _{\gamma }\Big( \prod\limits_{j=1}^{n}\Big( \gamma ^{\mu
_{j}}-1\Big) ^{\omega _{j}}+1\Big) ,
\log _{\gamma }\Big( \prod\limits_{j=1}^{n}\Big( \gamma ^{\eta _{j}+\mu
_{j}}-1\Big) ^{\omega _{j}}+1\Big) -\log _{\gamma }\Big(
\prod\limits_{j=1}^{n}\Big( \gamma ^{\mu _{j}}-1\Big) ^{\omega _{j}}+1\Big) ,%
1-\log _{\gamma }\Big( \prod\limits_{j=1}^{n}(\gamma ^{1-\nu _{j}}-1)^{\omega
_{j}}+1\Big) \Big\rangle.$

(5) If we take $T$ as Dombi t-norms $T_{\gamma}^{\textbf{D}}$ ($\gamma\in
(0, +\infty)$), then

i) $\mathrm{PFIWA}_{T_{\gamma }^{\textbf{D}},\omega
}(\alpha _{1},\ldots ,\alpha _{n})\newline
=\Bigg\langle\frac{\sqrt[\gamma ]{\sum\limits_{j=1}^{n}\omega _{j}\left( \frac{\mu
_{j}}{1-\mu _{j}}\right) ^{\gamma }}}{1+\sqrt[\gamma ]{\sum\limits_{j=1}^{n}\omega
_{j}\left( \frac{\mu _{j}}{1-\mu _{j}}\right) ^{\gamma }}},
\frac{1}{\sqrt[\gamma ]{\sum\limits_{j=1}^{n}\omega _{j}\left( \frac{1-\eta
_{j}-\nu _{j}}{\eta _{j}+\nu _{j}}\right) ^{\gamma }}+1}\newline
~~~~ -\frac{1}{\sqrt[%
\gamma ]{\sum\limits_{j=1}^{n}\omega _{j}\left( \frac{1-\nu _{j}}{\nu _{j}}\right)
^{\gamma }}+1},\frac{1}{\sqrt[\gamma ]{\sum\limits_{j=1}^{n}\omega _{j}\left( \frac{1-\nu _{j}%
}{\nu _{j}}\right) ^{\gamma }}+1}\Bigg\rangle,$

ii) $\mathrm{PFIWG}_{T_{\gamma }^{\textbf{D}},\omega
}(\alpha _{1},\ldots ,\alpha _{n})\newline
=\Bigg\langle\frac{1}{\sqrt[\gamma ]{\sum\limits_{j=1}^{n}\omega _{j}\left( \frac{%
1-\mu _{j}}{\mu _{j}}\right) ^{\gamma }}+1},\frac{1}{\sqrt[\gamma ]{\sum\limits_{j=1}^{n}\omega _{j}\left( \frac{1-\eta
_{j}-\mu _{j}}{\eta _{j}+\mu _{j}}\right) ^{\gamma }}+1}\newline
~~~~ -\frac{1}{\sqrt[%
\gamma ]{\sum\limits_{j=1}^{n}\omega _{j}\left( \frac{1-\mu _{j}}{\mu _{j}}\right)
^{\gamma }}+1},\frac{\sqrt[\gamma ]{\sum\limits_{j=1}^{n}\omega _{j}\left( \frac{\nu _{j}}{%
1-\nu _{j}}\right) ^{\gamma }}}{1+\sqrt[\gamma ]{\sum\limits_{j=1}^{n}\omega
_{j}\left( \frac{\nu _{j}}{1-\nu _{j}}\right) ^{\gamma }}}\Bigg\rangle.$

(6) If we take $T$ as Acz\'{e}l-Alsina t-norms $T_{\gamma}^{\textbf{AA}}$ (%
$\gamma\in (0, +\infty)$), then

i) $\mathrm{PFIWA}_{T_{\gamma }^{\textbf{AA}},\omega
}(\alpha _{1},\ldots ,\alpha _{n})\newline
=\Bigg\langle1-e^{-\sqrt[\gamma ]{\sum\limits_{j=1}^{n}\omega _{j}( -\log
(1-\mu _{j})) ^{\gamma }}},\newline
~~~~e^{-\sqrt[\gamma ]{\sum\limits_{j=1}^{n}\omega _{j}( -\log (\eta _{j}+\nu
_{j})) ^{\gamma }}}-e^{-\sqrt[\gamma ]{\sum\limits_{j=1}^{n}\omega _{j}(
-\log \nu _{j}) ^{\gamma }}},\newline
~~~~e^{-\sqrt[\gamma ]{\sum\limits_{j=1}^{n}\omega _{j}( -\log \nu _{j})
^{\gamma }}}\Bigg\rangle,$

ii) $\mathrm{PFIWG}_{T_{\gamma }^{\textbf{AA}},\omega
}(\alpha _{1},\ldots ,\alpha _{n})\newline
=\Bigg\langle e^{-\sqrt[\gamma ]{\sum\limits_{j=1}^{n}\omega _{j}\left( -\log \mu
_{j}\right) ^{\gamma }}},\newline
~~~~e^{-\sqrt[\gamma ]{\sum\limits_{j=1}^{n}\omega _{j}\left( -\log (\eta _{j}+\mu
_{j})\right) ^{\gamma }}}-e^{-\sqrt[\gamma ]{\sum\limits_{j=1}^{n}\omega _{j}\left(
-\log \mu _{j}\right) ^{\gamma }}},\newline
~~~~1-e^{-\sqrt[\gamma ]{\sum\limits_{j=1}^{n}\omega _{j}\left( -\log (1-\nu
_{j})\right) ^{\gamma }}}\Bigg\rangle.$

In particular, if we take $\eta_{\alpha_j}=0$
for each $\alpha_j$, i.e., each $\alpha_j$ reduces to an
intuitionistic fuzzy number, then
$\mathrm{PFIWA}_{T_{\gamma }^{\textbf{AA}},\omega
}(\alpha _{1},\ldots ,\alpha _{n})
=\big\langle1-e^{-\sqrt[\gamma ]{\sum_{j=1}^{n}\omega _{j}( -\log
(1-\mu _{j})) ^{\gamma }}}, 0,
e^{-\sqrt[\gamma ]{\sum_{j=1}^{n}\omega _{j}( -\log \nu _{j})
^{\gamma }}}\big\rangle$,
which is consistent with \cite[Theorem~2]{SCY2022}. Applying
Theorems~\ref{Mono-Thorem}--\ref{Boundedness-Thm}, we know that
\cite[Theorem~2--5]{SCY2022} are direct corollaries of our results.

The operators $\mathrm{PFIWA}_{_{T, \omega}}$ and $\mathrm{PFIWG}_{_{T,
\omega}}$ have the following desirable properties.

\begin{theorem}[\textrm{Monotonicity}]
\label{Mono-Thorem} Let $\alpha_{j}=\left\langle\mu_{\alpha_j},
\eta_{\alpha_j}, \nu_{\alpha_j}\right\rangle$ ($j=1, \ldots, n$) and $%
\beta_{j}=\left\langle\mu_{\beta_j}, \eta_{\beta_j},
\nu_{\beta_j}\right\rangle$ ($j=1, \ldots, n$) be two collections of PFNs
such that $\mu_{\alpha_j}\leq \mu_{\beta_j}$, $\eta_{\alpha_j}\leq
\eta_{\beta_j}$, and $\nu_{\alpha_j}\geq \nu_{\beta_j}$; i.e., $%
\alpha_{j}\subseteq \beta_{j}$. Then
$$
\mathrm{PFIWA}_{_{T, \omega}}(\alpha_1, \ldots, \alpha_n)
\preceq_{_{\mathrm{W}}} \mathrm{PFIWA}_{_{T, \omega}} (\beta_1,
\ldots, \beta_n).
$$
\end{theorem}

\begin{proof}
By Theorem~\ref{PFWA-PFGA-Thm}, we have
$\mathrm{PFIWA}_{T, \omega} (\alpha_1, \ldots, \alpha_n)
= \langle \zeta^{-1}(\omega_1\cdot \zeta(\mu_{\alpha_1})+ \cdots
+\omega_n\cdot \zeta(\mu_{\alpha_n})), \tau^{-1}(\omega_{1} \cdot
\tau(\eta_{\alpha_1}+\nu_{\alpha_1})+\cdots +\omega_n \cdot
\tau(\eta_{\alpha_n}+\nu_{\alpha_n}))-\tau^{-1}(\omega_{1}\cdot
\tau(\nu_{\alpha_1})+\cdots +\omega_n\cdot \tau(\nu_{\alpha_n})),
\tau^{-1}(\omega_1\cdot \tau(\nu_{\alpha_1})+\cdots
+\omega_n\cdot \tau(\nu_{\alpha_n}))\rangle
\triangleq \langle \tilde{\mu}, \tilde{\eta}, \tilde{\nu}\rangle$,
and
$\mathrm{PFIWA}_{T, \omega} (\beta_1, \ldots, \beta_n)
= \big\langle \zeta^{-1}(\omega_1\cdot \zeta(\mu_{\beta_1})+ \cdots
+\omega_n\cdot \zeta(\mu_{\beta_n})), \tau^{-1}(\omega_{1} \cdot
\tau(\eta_{\beta_1}+\nu_{\beta_1})+\cdots +\omega_n \cdot
\tau(\eta_{\beta_n}+\nu_{\beta_n}))-\tau^{-1}(\omega_{1}\cdot
\tau(\nu_{\beta_1})+\cdots +\omega_n\cdot \tau(\nu_{\beta_n})),
\tau^{-1}(\omega_1\cdot \tau(\nu_{\beta_1})+\cdots
+\omega_n\cdot \tau(\nu_{\beta_n})) \big\rangle
\triangleq \big\langle \hat{\mu}, \hat{\eta}, \hat{\nu}\big\rangle$.

Since $\zeta$ is increasing, it
follows from $\mu_{\alpha_j}\leq \mu_{\beta_j}$ that $\tilde{\mu}\leq \hat{\mu}$.
Similarly, since $\tau$ is decreasing, it follows from $\nu_{\alpha_{j}}\geq \nu_{\beta_j}$
that $\tilde{\nu}%
\geq \hat{\nu}$. To prove $\big\langle \tilde{\mu}, \tilde{\eta}, \tilde{\nu}%
\big\rangle
\preceq_{_{\mathrm{W}}} \big\langle \hat{\mu}, \hat{\eta}, \hat{\nu}%
\big\rangle$, we consider the following:

(1) If there exists $1\leq j_0\leq n$ such that $\mu_{\alpha_{j_0}}<\mu_{%
\beta_{j_0}}$, noting that $\zeta$ is strictly increasing and $%
\zeta(1)=+\infty$, from $\mu_{\alpha_{j}}\leq\mu_{\beta_{j}}$ ($j=1, 2,
\ldots, n$), it follows that $\tilde{\mu}=\zeta^{-1}(\omega_1\cdot
\zeta(\mu_{\alpha_1})+ \cdots +\omega_n\cdot
\zeta(\mu_{\alpha_n}))<\zeta^{-1}(\omega_1\cdot \zeta(\mu_{\beta_1})+ \cdots
+\omega_n\cdot \zeta(\mu_{\beta_n}))=\hat{\mu}$. This, together with $\tilde{%
\nu}\geq \hat{\nu}$, implies that $S\left(\langle \tilde{\mu}, \tilde{\eta},
\tilde{\nu}\rangle\right)= \tilde{\mu}-\tilde{\nu}<\hat{\mu}-\tilde{\nu}\leq
\hat{\mu}-\hat{\nu}= S\left(\langle \hat{\mu}, \hat{\eta}, \hat{\nu}%
\rangle\right)$. Thus
$
\mathrm{PFIWA}_{T, \omega} (\alpha_1, \ldots, \alpha_n)= %
\langle \tilde{\mu}, \tilde{\eta}, \tilde{\nu}\rangle \prec_{_{%
\mathrm{W}}}\langle \hat{\mu}, \hat{\eta}, \hat{\nu}\rangle =%
\mathrm{PFIWA}_{T, \omega} (\beta_1, \ldots, \beta_n).$

(2) If there exists $1\leq j_{0}\leq n$ such that $\nu _{\alpha
_{j_{0}}}>\nu _{\beta _{j_{0}}}$, noting that $\tau $ is strictly decreasing
and $\tau (0)=+\infty $, from $\nu _{\alpha _{j}}\geq \nu _{\beta _{j}}$ ($%
j=1,2,\ldots ,n$), it follows that $\tilde{\nu}=\tau ^{-1}(\omega _{1}\cdot
\tau (\nu _{\alpha _{1}})+\cdots +\omega _{n}\cdot \tau (\nu _{\alpha
_{n}}))>\tau ^{-1}({\omega _{1}\cdot \tau (\nu _{\beta
_{1}})+\cdots +\omega _{n}\cdot \tau (\nu _{\beta _{n}}}))=\hat{\nu}$. This,
together with $\tilde{\mu}\leq \hat{\mu}$, implies that $S\left( \langle
\tilde{\mu},\tilde{\eta},\tilde{\nu}\rangle \right) =\tilde{\mu}-\tilde{\nu}<%
\tilde{\mu}-\hat{\nu}\leq \hat{\mu}-\hat{\nu}=S\left( \langle \hat{\mu},\hat{%
\eta},\hat{\nu}\rangle \right) $. Thus
$
\mathrm{PFIWA}_{T,\omega }(\alpha _{1},\ldots ,\alpha _{n})=%
\langle\tilde{\mu},\tilde{\eta},\tilde{\nu}\rangle\prec _{_{\mathrm{W%
}}}\langle\hat{\mu},\hat{\eta},\hat{\nu}\rangle=\mathrm{PFIWA}%
_{T,\omega }(\beta _{1},\ldots ,\beta _{n}).
$

(3) If $\mu_{\alpha_j}=\mu_{\beta_j}$ and $\nu_{\alpha_j}=\nu_{\beta_{j}}$
for all $1\leq j\leq n$, then $\tilde{\mu}=\hat{\mu}$ and $\tilde{\nu}=\hat{%
\nu}$, and thus $S(\langle \tilde{\mu}, \tilde{\eta}, \tilde{\nu}\rangle)
=S(\langle \hat{\mu}, \hat{\eta}, \hat{\nu}\rangle)$ and $H_{1}(\langle
\tilde{\mu}, \tilde{\eta}, \tilde{\nu}\rangle) =H_{1}(\langle \hat{\mu},
\hat{\eta}, \hat{\nu}\rangle)$. Since $\tau$ is strictly decreasing, by $%
\eta_{\alpha_j}\leq\eta_{\beta_j}$ ($j=1, 2, \ldots, n$), we have
$
\tilde{\eta} =\tau^{-1}(\omega_{1} \cdot
\tau(\eta_{\alpha_1}+\nu_{\alpha_1})+\cdots +\omega_n \cdot
\tau(\eta_{\alpha_n}+\nu_{\alpha_n}))-\tau^{-1}(\omega_{1}\cdot
\tau(\nu_{\alpha_1})+\cdots +\omega_n\cdot \tau(\nu_{\alpha_n}))
=\tau^{-1}(\omega_{1} \cdot \tau(\eta_{\alpha_1}+\nu_{\beta_1})+\cdots
+\omega_n \cdot
\tau(\eta_{\alpha_n}+\nu_{\beta_n}))-\tau^{-1}(\omega_{1}\cdot
\tau(\nu_{\beta_1})+\cdots +\omega_n\cdot \tau(\nu_{\beta_n}))
\leq \tau^{-1}(\omega_{1} \cdot \tau(\eta_{\beta_1}+\nu_{\beta_1})+\cdots
+\omega_n \cdot
\tau(\eta_{\beta_n}+\nu_{\beta_n}))-\tau^{-1}(\omega_{1}\cdot
\tau(\nu_{\beta_1})+\cdots +\omega_n\cdot \tau(\nu_{\beta_n}))=\hat{\eta},
$
implying that
\begin{equation*}
H_{2}(\langle \tilde{\mu}, \tilde{\eta}, \tilde{\nu}\rangle) =\tilde{\mu}+%
\tilde{\eta}+\tilde{\nu}\leq \hat{\mu}+\hat{\eta}+\hat{\nu}=H_{2}(\langle
\hat{\mu}, \hat{\eta}, \hat{\nu}\rangle).
\end{equation*}
Therefore, by Definition~\ref{Order-Wu}, there holds that
$\mathrm{PFIWA}_{T, \omega} (\alpha_1, \ldots,
\alpha_n)=\langle \tilde{\mu}, \tilde{\eta}, \tilde{\nu}\rangle %
\preceq_{_{\mathrm{W}}} \langle \hat{\mu}, \hat{\eta}, \hat{\nu}%
\rangle
=\mathrm{PFIWA}_{T, \omega} (\beta_1, \ldots, \beta_n)$.
\end{proof}

\begin{theorem}[\textrm{Idempotency}]
\label{Idem-Thm} If $\alpha_{1}=\cdots =\alpha_{n}=
\langle\mu, \eta, \nu\rangle\in %
\mathscr{P}$, then
$
\mathrm{PFIWA}_{_{T, \omega}}(\alpha_1, \ldots, \alpha_n) =%
\big\langle\mu, \eta, \nu\big\rangle.
$
\end{theorem}

\begin{proof}
By Theorem~\ref{PFWA-PFGA-Thm}, we have
$\mathrm{PFIWA}_{T, \omega} (\alpha_1, \ldots, \alpha_n)
=\langle \zeta^{-1}(\omega_1\cdot \zeta(\mu)+ \cdots +\omega_n\cdot
\zeta(\mu)), \tau^{-1}(\omega_1 \cdot \tau(\eta+\nu)+\cdots +\omega_n \cdot
\tau(\eta+\nu))-\tau^{-1}(\omega_1\cdot \tau(\nu)+\cdots +\omega_n\cdot
\tau(\nu)), \tau^{-1}(\omega_1\cdot \tau(\nu)+\cdots +\omega_n\cdot
\tau(\nu)) \rangle =\langle \mu, \eta, \nu \rangle.$
\end{proof}

\begin{remark}
\cite[Theorems~2--5]{PWC2021} are direct corollaries of our Theorems~\ref%
{PFWA-PFGA-Thm}--\ref{Idem-Thm}.
\end{remark}

\begin{theorem}[\textrm{Boundedness}]
\label{Boundedness-Thm}
For $\alpha_{j}=\langle \mu_{j}, \eta_{j}, \nu_{j}\rangle
\in \mathscr{P}$ ($j=1, \ldots, n$),
we have
$\langle\min\limits_{1\leq j\leq n}\{\mu_{j}\}, \min\limits_{1\leq j\leq
n}\{\eta_{j}\}, \max\limits_{1\leq j\leq n}\{\nu_{j}\} \rangle$
 $\preceq_{_{\mathrm{W}}} \mathrm{PFIWA}_{_{T, \omega}}(\alpha_1, \ldots,
\alpha_n)\preceq_{_{\mathrm{W}}} \langle\max\limits_{1\leq j\leq n}\{\mu_j\},
1-(\max\limits_{1\leq j\leq n}\{\mu_j\}+ \min\limits_{1\leq j\leq
n}\{\nu_{j}\}), \min\limits_{1\leq j\leq n}\{\nu_{j}\}\rangle.
$
\end{theorem}

\begin{proof}
For convenience, denote
$\big\langle\min\limits_{1\leq j\leq n}\{\mu_{j}\},
\min\limits_{1\leq j\leq n}\{\eta_{j}\},$ $\max\limits_{1\leq j\leq n}\{\nu_{j}\} \big\rangle%
=\big\langle \tilde{\mu}_1, \tilde{\eta}_1, \tilde{\nu}_1\big\rangle,$
$\mathrm{PFIWA}_{_{T, \omega}}(\alpha_1, \alpha_2, \ldots, \alpha_n)=%
\big\langle \mu, \eta, \nu\big\rangle,$
and
$\big\langle\max\limits_{1\leq j\leq
n}\{\mu_j\}, 1-(\max\limits_{1\leq j\leq n}\{\mu_j\}+ \min\limits_{1\leq j\leq
n}\{\nu_{j}\}),$ $\min\limits_{1\leq j\leq n}\{\nu_{j}\}\big\rangle= %
\big\langle \tilde{\mu}_2, \tilde{\eta}_2, \tilde{\nu}_2\big\rangle.$
Clearly, $\big\langle \tilde{\mu}_1, \tilde{\eta}_1, \tilde{\nu}_1%
\big\rangle
$, $\big\langle \mu, \eta, \nu\big\rangle$, and $\big\langle \tilde{\mu}_2,
\tilde{\eta}_2, \tilde{\nu}_2\big\rangle$ are PFNs. Noting that $\big\langle
\tilde{\mu}_1, \tilde{\eta}_1, \tilde{\nu}_1\big\rangle \subseteq \alpha_{j}$
holds for all $1\leq j\leq n$, by Theorems~\ref{Mono-Thorem} and \ref%
{Idem-Thm}, we have $\big\langle \tilde{\mu}_1, \tilde{\eta}_1, \tilde{\nu}_1%
\big\rangle
=\mathrm{PFIWA}_{_{T, \omega}}(\big\langle \tilde{\mu}_1, \tilde{\eta}_1,
\tilde{\nu}_1\big\rangle) \preceq_{_{\mathrm{W}}} \mathrm{PFIWA}_{_{T,
\omega}}(\alpha_1, \ldots, \alpha_n)$.

Clearly, $\mu_j\leq \tilde{\mu}_2$ and $\nu_j\geq \tilde{\nu}_2$ holds for
all $1\leq j\leq n$. To prove $\langle \mu, \eta, \nu \rangle
\preceq_{_{\mathrm{W}}} \langle \tilde{\mu}_2, \tilde{\eta}_2, \tilde{\nu%
}_2 \rangle$, we consider the following cases:

(1) If there exists $1\leq j_0\leq n$ such that $\mu_{j_0}<\tilde{\mu}_2$,
since $\zeta$ is a strictly increasing function with $\zeta(1)=+\infty$ and $%
\mu_j\leq \tilde{\mu}_2$ ($j=1, 2, \ldots, n$), by Theorem~\ref%
{PFWA-PFGA-Thm}, we have
\begin{equation}\label{eq-Bound-1}
\begin{split}
\mu&=\zeta^{-1}(\omega_1\cdot \zeta(\mu_1)+ \cdots +\omega_n\cdot
\zeta(\mu_n))\\
&<\zeta^{-1}(\omega_1\cdot \zeta(\tilde{\mu}_2)+ \cdots
+\omega_n\cdot \zeta(\tilde{\mu}_2))=\tilde{\mu}_2.
\end{split}
\end{equation}
Meanwhile, since $\tau$ is a strictly decreasing function with $%
\tau(0)=+\infty$, by Theorem~\ref{PFWA-PFGA-Thm} and $\nu_{j}\geq \tilde{\nu}%
_2$, we have
\begin{equation}\label{eq-Bound-2}
\begin{split}
\nu&=\tau^{-1}(\omega_1\cdot \tau(\nu_1)+\cdots +\omega_n\cdot
\tau(\nu_n))\\
&\geq \tau^{-1}(\omega_1\cdot \tau(\tilde{\nu}_2)+\cdots
+\omega_n\cdot \tau(\tilde{\nu}_2))=\tilde{\nu}_2.
\end{split}
\end{equation}
This, together with formula~\eqref{eq-Bound-1}, implies that $S(\langle \mu,
\eta, \nu\rangle)=\mu-\nu<\tilde{\mu}_2-\nu\leq \tilde{\mu}_2-\tilde{\nu}_2=
S(\langle \tilde{\mu}_2, \tilde{\eta}_2, \tilde{\nu}_2\rangle)$. Thus $%
\langle \mu, \eta, \nu\rangle\prec_{_{\mathrm{W}}} \langle \tilde{\mu}_2,
\tilde{\eta}_2, \tilde{\nu}_2\rangle$.

(2) If there exists $1\leq j_0\leq n$ such that $\nu_{j_0}>\tilde{\nu}_2$,
by using similar arguments to the above proof, it can be verified that $S(\langle \mu, \eta,
\nu\rangle)< S(\langle \tilde{\mu}_2, \tilde{\eta}_2, \tilde{\nu}_2\rangle)$.
Thus $\langle \mu, \eta, \nu\rangle\prec_{_{\mathrm{W}}} \langle
\tilde{\mu}_2, \tilde{\eta}_2, \tilde{\nu}_2\rangle$.

(3) If $\mu_j= \tilde{\mu}_2$ and $\nu_j= \tilde{\nu}_2$ holds for all $%
1\leq j\leq n$, i.e., $\tilde{\mu}_2=\mu_{j}$ and $\tilde{\nu}_2 =\nu_{j}$ ($%
j=1, 2, \ldots, n$), from $\mu_j+\eta_j+\nu_j\leq 1$, it follows that $%
\max_{1\leq j\leq n}\{\eta_{j}\}\leq 1-(\tilde{\mu}_2+\tilde{\nu}_2) =\tilde{%
\eta}_2$. This means that $\alpha_{j}\subseteq \big\langle \tilde{\mu}_2,
\tilde{\eta}_2, \tilde{\nu}_2\big\rangle$ holds for all $1\leq j\leq n$. By
Theorems~\ref{Mono-Thorem} and \ref{Idem-Thm}, we have  $\mathrm{PFIWA}_{_{T,
\omega}}(\alpha_1, \alpha_2, \ldots, \alpha_n) \preceq_{_{\mathrm{W}}}
\mathrm{PFIWA}_{_{T, \omega}}(\langle \tilde{\mu}_2, \tilde{\eta}_2, \tilde{%
\nu}_2\rangle)= \big\langle \tilde{\mu}_2, \tilde{\eta}_2, \tilde{\nu}_2%
\big\rangle$.
\end{proof}

\begin{remark}
Many scholars obtained the boundedness for various picture fuzzy aggregation
operators by the following two incorrect forms:

(1) In \cite{KAA2019,Wei2017,Wei2018a,WLG2018}, the lower bound and upper
bound are equal to $\min_{\leq j\leq n}\{\alpha _{j}\}$ and $\max_{\leq j\leq
n}\{\alpha _{j}\}$, respectively. Since there is no total order for PFNs, the
minimum and maximum of $\alpha _{1},\alpha _{2},\ldots ,\alpha _{n}$ does
not need to exist.

(2) In \cite[Theorems 5, 10, and 15]{AA2020-1}, \cite[Property~3]{Ga2017},
\cite[Theorems~4, 8, 13, and 17]{JSPY2019}, the lower bound is equal to $%
\big\langle\min_{1\leq j\leq n}\{\mu _{\alpha _{j}}\},\max_{1\leq j\leq
n}\{\eta _{\alpha _{j}}\},\max_{1\leq j\leq n}\{\nu _{\alpha _{j}}\}%
\big\rangle$. It does not need to be a PFN since $\min_{1\leq j\leq n}\{\mu
_{\alpha _{j}}\}+\max_{1\leq j\leq n}\{\eta _{\alpha _{j}}\}+\max_{1\leq j\leq
n}\{\nu _{\alpha _{j}}\}$ can be larger than $1$.
\end{remark}

\begin{theorem}[\textrm{Shift Invariance}]
\label{Shift-invariance-Thm} Let $\alpha_{j}=\langle\mu_{j}, \eta_{j},
\nu_{j}\rangle\in \mathscr{P}$ ($j=1, \ldots, n$) and $%
\beta=\langle\mu_{\beta}, \eta_{\beta}, \nu_{\beta} \rangle\in %
\mathscr{P}$. Then,
$
\mathrm{PFIWA}_{_{T, \omega}}(\alpha_1\oplus_{_{T}}\beta,
\alpha_2\oplus_{_{T}}\beta, \ldots, \alpha_n\oplus_{_{T}}\beta) =\mathrm{%
PFIWA}_{_{T, \omega}}(\alpha_1, \alpha_2, \ldots,
\alpha_n)\oplus_{_{T}}\beta.$
\end{theorem}

\begin{proof}
By Definition~\ref{Def-Wu}, we have
$\alpha_{j}\oplus_{_{T}}\beta =\langle S(\mu_{j}, \mu_{\beta}),
T(\eta_j+\nu_j, \eta_{\beta}+\nu_{\beta}) -T(\nu_j, \nu_{\beta}), T(\nu_j,
\nu_{\beta})\rangle =\langle \zeta^{-1}(\zeta(\mu_{j})+\zeta(\mu_{\beta})),
\tau^{-1}(\tau(\eta_j+\nu_j)+\tau(\eta_{\beta}+\nu_{\beta}))
-\tau^{-1}(\tau(\nu_j)+\tau(\nu_{\beta})), \tau^{-1}(\tau(\nu_j)
+\tau(\nu_{\beta}))\rangle.
$
This, together with Theorem~\ref{PFWA-PFGA-Thm}, implies that
\begin{align*}
&\quad \mathrm{PFIWA}_{_{T, \omega}}(\alpha_1\oplus_{_{T}}\beta,
\alpha_2\oplus_{_{T}}\beta, \ldots, \alpha_n\oplus_{_{T}}\beta) \\
&= \left\langle \zeta^{-1}(\omega_1\cdot \zeta(\mu_1)+ \cdots +\omega_n\cdot
\zeta(\mu_n)+\zeta(\mu_{\beta})),\right. \\
&\quad \tau^{-1}(\omega_1 \cdot \tau(\eta_1+\nu_1)+\cdots +\omega_n
\cdot \tau(\eta_n+\nu_n)+\tau(\eta_{\beta}+\nu_{\beta})) \\
&\quad -\tau^{-1}(\omega_1\cdot \tau(\nu_1)+\cdots +\omega_n\cdot
\tau(\nu_n)+\tau(\nu_{\beta})), \\
&\quad \left.\tau^{-1}(\omega_1\cdot \tau(\nu_1)+\cdots +\omega_n\cdot
\tau(\nu_n)+\tau(\nu_{\beta})) \right\rangle\\
&= \left\langle \zeta^{-1}(\zeta(\zeta^{-1}(\omega_1\cdot \zeta(\mu_1)+
\cdots +\omega_n\cdot \zeta(\mu_n)))+\zeta(\mu_{\beta})),\right. \\
&\quad \tau^{-1}\big(\tau(\tau^{-1}(\omega_1 \cdot
\tau(\eta_1+\nu_1)+\cdots +\omega_n \cdot
\tau(\eta_n+\nu_n)))\\
&~~~~~~~~~~~~~~~~~~~~~~~~~~~~~~~~~~~~~~~~~~~~~~~~
+\tau(\eta_{\beta}+\nu_{\beta})\big) \\
&\quad -\tau^{-1}(\tau(\tau^{-1}(\omega_1\cdot \tau(\nu_1)+\cdots
+\omega_n\cdot \tau(\nu_n)))+\tau(\nu_{\beta})), \\
&\quad \left.\tau^{-1}(\tau(\tau^{-1}(\omega_1\cdot \tau(\nu_1)+\cdots
+\omega_n\cdot \tau(\nu_n)))+\tau(\nu_{\beta})) \right\rangle\\
&= \left\langle S(\zeta^{-1}(\omega_1\cdot \zeta(\mu_1)+ \cdots
+\omega_n\cdot \zeta(\mu_n)), \mu_{\beta}),\right. \\
&\quad  T(\tau^{-1}(\omega_1 \cdot \tau(\eta_1+\nu_1)+\cdots +\omega_n
\cdot \tau(\eta_n+\nu_n)), \eta_{\beta}+\nu_{\beta}) \\
&\quad  -T(\tau^{-1}(\omega_1\cdot \tau(\nu_1)+\cdots +\omega_n\cdot
\tau(\nu_n)), \nu_{\beta}), \\
&\quad  \left.T(\tau^{-1}(\omega_1\cdot \tau(\nu_1)+\cdots +\omega_n\cdot
\tau(\nu_n)), \nu_{\beta}) \right\rangle \\
&=\mathrm{PFIWA}_{_{T, \omega}}(\alpha_1, \alpha_2, \ldots,
\alpha_n)\oplus_{_{T}}\beta \quad \text{(by formula~\eqref{eq-PFWA-1})}.
\end{align*}
\end{proof}

\begin{theorem}[\textrm{Homogeneity}]
\label{Homogeneity-Thm} Let $\alpha_{j}=\langle\mu_{j}, \eta_{j},
\nu_{j}\rangle\in \mathscr{P}$ ($j=1, \ldots, n$) and $\lambda>0$.
Then,
$
\mathrm{PFIWA}_{_{T, \omega}}(\lambda_{_{T}}\cdot \alpha_1,
\ldots, \lambda_{_{T}}\cdot \alpha_n)
=\lambda_{_{T}}\cdot \mathrm{PFIWA}_{_{T, \omega}}(\alpha_1,
\ldots, \alpha_n).$
\end{theorem}

\begin{proof}
Noting that $\lambda_{_{T}}\cdot \alpha_j =\big\langle
\zeta^{-1}(\lambda\cdot \zeta(\mu_j)), \tau^{-1}(\lambda\cdot
\tau(\eta_j+\nu_j))-\tau^{-1}(\lambda\cdot \tau(\nu_j)),
\tau^{-1}(\lambda\cdot \tau(\nu_j)) \big\rangle$ ($j=1, 2, \ldots, n$), by
Definition~\ref{Def-Wu} and Theorem~\ref{PFWA-PFGA-Thm}, we have
\begin{align*}
&\quad \mathrm{PFIWA}_{_{T, \omega}}(\lambda_{_{T}}\cdot \alpha_1,
\lambda_{_{T}}\cdot \alpha_2, \ldots, \lambda_{_{T}}\cdot \alpha_n) \\
&= \left\langle \zeta^{-1}(\omega_1\cdot\lambda\cdot \zeta(\mu_1)+ \cdots
+\omega_n\cdot\lambda\cdot \zeta(\mu_n)),\right. \\
&\quad \quad \tau^{-1}(\omega_1 \cdot\lambda\cdot \tau(\eta_1+\nu_1)+\cdots
+\omega_n \cdot\lambda\cdot
\tau(\eta_n+\nu_n))\\
&\quad \quad -\tau^{-1}(\omega_1\cdot\lambda\cdot \tau(\nu_1)+\cdots
+\omega_n\cdot\lambda\cdot \tau(\nu_n)), \\
&\quad \quad \left.\tau^{-1}(\omega_1\cdot\lambda\cdot \tau(\nu_1)+\cdots
+\omega_n\cdot\lambda\cdot \tau(\nu_n)) \right\rangle,
\end{align*}
and
\begin{align*}
&\quad \lambda_{_{T}}\cdot \mathrm{PFIWA}_{_{T, \omega}}(\alpha_1, \alpha_2,
\ldots, \alpha_n) \\
&=\lambda_{_{T}}\cdot \left\langle \zeta^{-1}(\omega_1\cdot \zeta(\mu_1)+
\cdots +\omega_n\cdot \zeta(\mu_n)),\right. \\
&\quad \quad \tau^{-1}(\omega_1 \cdot \tau(\eta_1+\nu_1)+\cdots +\omega_n
\cdot \tau(\eta_n+\nu_n))\\
&\quad \quad -\tau^{-1}(\omega_1\cdot \tau(\nu_1)+\cdots
+\omega_n\cdot \tau(\nu_n)), \\
&\quad \quad \left.\tau^{-1}(\omega_1\cdot \tau(\nu_1)+\cdots +\omega_n\cdot
\tau(\nu_n)) \right\rangle\\
&= \left\langle \zeta^{-1}(\lambda(\omega_1\cdot \zeta(\mu_1)+ \cdots
+\omega_n\cdot \zeta(\mu_n))),\right. \\
&\quad \quad \tau^{-1}(\lambda(\omega_1 \cdot \tau(\eta_1+\nu_1)+\cdots
+\omega_n \cdot \tau(\eta_n+\nu_n)))\\
&\quad \quad -\tau^{-1}(\lambda(\omega_1\cdot
\tau(\nu_1)+\cdots +\omega_n\cdot \tau(\nu_n))), \\
&\quad \quad \left.\tau^{-1}(\lambda(\omega_1\cdot \tau(\nu_1)+\cdots
+\omega_n\cdot \tau(\nu_n))) \right\rangle,
\end{align*}
implying that $\mathrm{PFIWA}_{_{T, \omega}}(\lambda_{_{T}}\cdot \alpha_1,
\lambda_{_{T}}\cdot \alpha_2, \ldots, \lambda_{_{T}}\cdot \alpha_n)
=\lambda_{_{T}}\cdot \mathrm{PFIWA}_{_{T, \omega}}(\alpha_1, \alpha_2,
\ldots, \alpha_n)$.
\end{proof}

We have the following results analogously.

\begin{theorem}[\textrm{Monotonicity}]
\label{Mono-Thorem-2} Let $\alpha_{j}=\langle\mu_{\alpha_j},
\eta_{\alpha_j}, \nu_{\alpha_j}\rangle$ ($j=1, \ldots, n$) and $%
\beta_{j}=\langle\mu_{\beta_j}, \eta_{\beta_j},
\nu_{\beta_j}\rangle$ ($j=1, \ldots, n$) be two collections of PFNs
such that $\mu_{\alpha_j}\leq \mu_{\beta_j}$, $\eta_{\alpha_j}\leq
\eta_{\beta_j}$, and $\nu_{\alpha_j}\geq \nu_{\beta_j}$; i.e., $%
\alpha_{j}\subseteq \beta_{j}$. Then
\begin{equation*}
\mathrm{PFIWG}_{_{T, \omega}}(\alpha_1, \ldots, \alpha_n)
\preceq_{_{\mathrm{W}}} \mathrm{PFIWG}_{_{T, \omega}} (\beta_1,
\ldots, \beta_n).
\end{equation*}
\end{theorem}

\begin{theorem}[\textrm{Idempotency}]
\label{Idem-Thm-2} If $\alpha_{1}=\cdots =\alpha_n=
\langle\mu, \eta, \nu \rangle\in %
\mathscr{P}$, then
$
\mathrm{PFIWG}_{_{T, \omega}}(\alpha_1, \ldots, \alpha_n) =%
\big\langle\mu, \eta, \nu\big\rangle.
$
\end{theorem}

\begin{theorem}[\textrm{Boundedness}]
For $\alpha_{j}=\langle\mu_{j}, \eta_{j}, \nu_{j}\rangle\in
\mathscr{P}$ ($j=1, \ldots, n$), we have
$\langle\min\limits_{1\leq j\leq n}\{\mu_{j}\}, \min\limits_{1\leq j\leq
n}\{\eta_{j}\}, \max\limits_{1\leq j\leq n}\{\nu_{j}\} \rangle$
$\preceq_{_\mathrm{W}} \mathrm{PFIWG}_{_{T, \omega}}(\alpha_1, \ldots,
\alpha_n)\preceq_{_{\mathrm{W}}}\langle\max\limits_{1\leq j\leq n}\{\mu_j\},
1-(\max\limits_{1\leq j\leq n}\{\mu_j\}+ \min\limits_{1\leq j\leq
n}\{\nu_{j}\}), \min\limits_{1\leq j\leq n}\{\nu_{j}\}\rangle.
$
\end{theorem}

\begin{theorem}[\textrm{Shift Invariance}]
\label{Shift-invariance-Thm-1} Let $\alpha_{j}=\langle \mu_{j},
\eta_{j}, \nu_{j}\rangle\in \mathscr{P}$ ($j=1, \ldots, n$) and $%
\beta=\langle\mu_{\beta}, \eta_{\beta}, \nu_{\beta} \rangle\in %
\mathscr{P}$. Then,
$\mathrm{PFIWG}_{_{T, \omega}}(\alpha_1\otimes_{_{T}}\beta,
\alpha_2\otimes_{_{T}}\beta, \ldots, \alpha_n\otimes_{_{T}}\beta) =\mathrm{%
PFIWG}_{_{T, \omega}}(\alpha_1, \alpha_2, \ldots,
\alpha_n)\otimes_{_{T}}\beta.$
\end{theorem}

\begin{theorem}[\textrm{Homogeneity}]
\label{Homogeneity-Thm-1}
Let $\alpha_{j}=\langle\mu_{j}, \eta_{j}, \nu_{j}\rangle\in \mathscr{P}$
($j=1, \ldots, n$) and $\lambda>0$.
Then,
$\mathrm{PFIWG}_{_{T, \omega}}((\alpha_1)^{\lambda_{_{T}}},
\ldots, (\alpha_n)^{\lambda_{_{T}}}) =(
\mathrm{PFIWG}_{_{T, \omega}}(\alpha_1, \ldots,
\alpha_n))^{\lambda_{_{T}}}.$
\end{theorem}

\subsection{Picture fuzzy interactional ordered weighted average (geometric)
operator}

\begin{definition}
\label{PFWA-Def-2} Let $\omega=(\omega_1, \ldots, \omega_n)^{\top}$
be the weight vector such that $\omega_{j}\in (0, 1]$ and $%
\sum_{j=1}^{n}\omega_j=1$ and $T$ be a strict t-norm. Define the \textit{%
picture fuzzy interactional ordered weighted average operator} $\mathrm{%
PFIOWA}_{T, \omega}$ and \textit{picture fuzzy interactional ordered
weighted geometric operator} $\mathrm{PFIOWG}_{T, \omega}$ induced by $T$ as
\begin{equation*}
\label{eq-PFOWA-Ope}
\begin{split}
\mathrm{PFIOWA}_{T, \omega}: \mathscr{P}^{n}&\rightarrow \mathscr{P} \\
(\alpha_1, \ldots, \alpha_n)&\mapsto
(\omega_1)_{_T}\alpha_{\sigma(1)}\oplus_{_T} \cdots
\oplus_{_T} (\omega_n)_{_T}\alpha_{\sigma(n)},
\end{split}%
\end{equation*}%
and
\begin{equation*}  \label{eq-PFOWG-Ope}
\begin{split}
\mathrm{PFIOWG}_{T, \omega}: \mathscr{P}^{n}&\rightarrow \mathscr{P} \\
(\alpha_1, \ldots, \alpha_n)&\mapsto
\alpha_{\sigma(1)}^{(\omega_1)_{_T}} \otimes_{_T} \cdots
\otimes_{_T} \alpha_{\sigma(n)}^{(\omega_n)_{_T}},
\end{split}%
\end{equation*}%
respectively, where $({\sigma(1)}, \ldots, {\sigma(n)})$ is a
permutation of $(1, \ldots, n)$ such that $\alpha_{\sigma(j)} \preceq_{_{%
\mathrm{W}}} \alpha_{\sigma(j-1)}$ for all $2\leq j\leq n$, i.e., $%
\alpha_{\sigma(1)}\succeq_{_{\mathrm{W}}} \alpha_{\sigma(2)}\succeq_{_{%
\mathrm{W}}} \cdots \succeq_{_{\mathrm{W}}} \alpha_{\sigma(n)}$.
\end{definition}

Similarly to Theorem~\ref{PFWA-PFGA-Thm}, we have the following result.

\begin{theorem}
\label{PFOWA-PFOWG-Thm} Let $\alpha_{j}=(\mu_{j}, \eta_{j}, \nu_{j})\in \mathscr{P}$ ($j=1,
\ldots, n$), $\omega=(\omega_1, \ldots, \omega_n)^{\top}$ be the weight vector
of $\alpha_{j}$ ($j=1, \ldots, n$) such that $\omega_{j}\in (0, 1]$ and $\sum_{j=1}^{n}\omega_j=1$, and $T$ be
a strict t-norm with an AG $\tau$. Then,
\begin{equation*}  \label{eq-PFOWA-1}
\begin{split}
&\quad \mathrm{PFIOWA}_{T, \omega} (\alpha_1, \alpha_2, \ldots, \alpha_n) \\
&= \left\langle \zeta^{-1}(\omega_1\cdot \zeta(\mu_{\sigma(1)})+ \cdots
+\omega_n\cdot \zeta(\mu_{\sigma(n)})),\right. \\
&\quad \quad \tau^{-1}(\omega_1 \cdot
\tau(\eta_{\sigma(1)}+\nu_{\sigma(1)})+\cdots +\omega_n \cdot
\tau(\eta_{\sigma(n)}+\nu_{\sigma(n)}))\\
&\quad \quad - \tau^{-1}(\omega_1\cdot
\tau(\nu_{\sigma(1)})+\cdots +\omega_n\cdot \tau(\nu_{\sigma(n)})), \\
&\quad \quad \left.\tau^{-1}(\omega_1\cdot \tau(\nu_{\sigma(1)})+\cdots
+\omega_n\cdot \tau(\nu_{\sigma(n)})) \right\rangle,
\end{split}%
\end{equation*}
and
\begin{equation*}  \label{eq-PFOWG-2}
\begin{split}
&\quad \mathrm{PFIOWG}_{T, \omega} (\alpha_1, \alpha_2, \ldots, \alpha_n) \\
&= \left\langle \tau^{-1}(\omega_1\cdot \tau(\mu_{\sigma(1)})+ \cdots
+\omega_n\cdot \tau(\mu_{\sigma(n)})),\right. \\
&\quad \quad \tau^{-1}(\omega_1 \cdot
\tau(\eta_{\sigma(1)}+\mu_{\sigma(1)})+\cdots +\omega_n \cdot
\tau(\eta_{\sigma(n)}+\mu_{\sigma(n)}))\\
&\quad \quad -\tau^{-1}(\omega_1\cdot
\tau(\mu_{\sigma(1)})+\cdots +\omega_n\cdot \tau(\mu_{\sigma(n)})), \\
&\quad \quad \left.\zeta^{-1}(\omega_1\cdot \zeta(\nu_{\sigma(1)})+\cdots
+\omega_n\cdot \zeta(\nu_{\sigma(n)})) \right\rangle,
\end{split}%
\end{equation*}
where $\zeta(x)=\tau(1-x)$.
\end{theorem}

Meanwhile, it can be verified that the operators $\mathrm{PFIOWA}_{T,\omega }$
and $\mathrm{PFIOWG}_{T,\omega }$ have idempotency and boundedness
properties. Besides, we obtain that they have commutativity.

\begin{theorem}[\textrm{Commutativity}]
Let $\alpha_{j}=\left\langle\mu_{j}, \eta_{j}, \nu_{j}\right\rangle$ ($j=1,
\ldots, n$) and $\alpha_{j}^{\prime}=\left\langle\mu_{j}^{\prime},
\eta_{j}^{\prime}, \nu_{j}^{\prime}\right\rangle$ ($j=1, \ldots, n$) be
two collections of PFNs such that $(\alpha_{1}^{\prime},
\ldots, \alpha_{n}^{\prime})$ is any permutation of $%
(\alpha_1, \ldots, \alpha_n)$. Then
\begin{equation*}
\mathrm{PFIOWA}_{T, \omega} (\alpha_1, \ldots, \alpha_n)= \mathrm{%
PFIOWA}_{T, \omega}(\alpha_{1}^{\prime}, \ldots,
\alpha_{n}^{\prime}),
\end{equation*}
and
\begin{equation*}
\mathrm{PFIOWG}_{T, \omega} (\alpha_1, \ldots, \alpha_n)= \mathrm{%
PFIOWG}_{T, \omega}(\alpha_{1}^{\prime}, \ldots,
\alpha_{n}^{\prime}).
\end{equation*}
\end{theorem}

\section{A picture fuzzy MCDM method and a numerical example}
\label{S-7}
\subsection{A novel picture fuzzy MCDM method}
Consider an MCDM problem under the PF environment,
suppose $A=\{A_{1},A_{2},\ldots ,A_{m}\}$ is a set of alternatives to be
selected, and $G=\{G_{1},G_{2},\ldots ,G_{n}\}$ is a set of criteria to be
evaluated, whose weight vector is $\omega =(\omega _{1},\omega _{2},\ldots
,\omega _{n})^{\top }$ so that $\omega _{j}\in (0,1]$ and $%
\sum_{j=1}^{n}\omega _{j}=1$. Assume that the rating of an alternative $%
A_{i} $ ($i=1, \ldots ,m$) based on the criteria $G_{j}$ ($j=1, \ldots, n$) is
assessed by the decision-maker, forming PFNs $\alpha _{ij}=\langle \mu
_{ij},\eta _{ij},\nu _{ij}\rangle $, where $\mu _{ij}$ and $\eta _{ij}$
are the degrees of the positive and the neutral memberships,
respectively, and $\nu _{ij}$ is the degree that the alternative $A_{i}$
does not satisfy the attribute $G_{j}$. To rank alternatives, the following
steps are given:

Step~1: (Construct the decision matrix)
The decision-maker gives their preference by PFNs $\alpha _{ij}=\langle \mu _{ij},\eta
_{ij},\nu _{ij}\rangle $ towards the alternative $A_{i}$ and the criteria $%
G_{j}$. Hence, a picture fuzzy decision matrix $D=(\alpha
_{ij})_{m\times n}$ is constructed as below:
\begin{table}[H]
\centering
  \caption{The picture fuzzy decision matrix $D$}
  \label{Tab-**}
\scalebox{0.75}{
	\begin{tabular}{cccccc}
		\toprule
		$$ & $G_{1}$ &  $G_{2}$ & $\cdots$ & $G_{n}$ \\
		\midrule
		$A_{1}$ & $\langle \mu_{11}, \eta_{11}, \nu_{11}\rangle$ &
        $\langle \mu_{12}, \eta_{12}, \nu_{12}\rangle$ &
           $\cdots$ & $\langle \mu_{1n}, \eta_{1n}, \nu_{1n}\rangle$\\
		$A_{2}$ & $\langle \mu_{21}, \eta_{21}, \nu_{21}\rangle$ &
        $\langle \mu_{22}, \eta_{22}, \nu_{22}\rangle$ &
           $\ldots$ & $\langle \mu_{2n}, \eta_{2n}, \nu_{2n}\rangle$ \\
        $\vdots$ & $\vdots$ & $\vdots$ &
           $\ddots$ & $\vdots$ \\
        $A_{m}$ & $\langle \mu_{m1}, \eta_{m1}, \nu_{m1}\rangle$ &
        $\langle \mu_{m2}, \eta_{m2}, \nu_{m2}\rangle$ &
           $\ldots$ & $\langle \mu_{mn}, \eta_{mn}, \nu_{mn}\rangle$ \\
        \bottomrule
	\end{tabular}
      }
\end{table}

Step~2: (Normalize the decision matrix)
The matrix $D=(\alpha_{ij})_{m\times n}$ is transformed into the
normalized matrix $R=(r_{ij})_{m\times n}$ for the picture fuzzy decision as below:
\begin{equation*}
r_{ij}=
\begin{cases}
\alpha_{ij}, & \text{for benefit criteria } G_{j}, \\
\alpha_{ij}^{\complement}, & \text{for cost criteria } G_{j}.%
\end{cases}%
\end{equation*}
%%where $\alpha_{ij}^{\complement}$ is the complement of $\alpha_{ij}$.

Step~3: (Compute the aggregated value) Based on the normalized
decision matrix $R$, compute the overall aggregated value of
every alternative $A_{i}$ ($i=1,\ldots ,m$), under the different criteria $%
G_{1},G_{2},\ldots ,G_{n}$, by using $\mathrm{PFIWA}_{T,\omega }$, $%
\mathrm{PFIWG}_{T,\omega },$ $\mathrm{PFIOWA}_{T,\omega },$ or $\mathrm{%
PFIOWG}_{T,\omega }$ for some strict t-norm $T$. Hence, it is taken the
collective value $r_{i}$ for each alternative $A_{i}$.

Step~4: (Rank the alternatives) The alternatives $%
A_{1},A_{2},\ldots ,A_{m}$ by means of the total order given by Definition~%
\ref{Order-Wu} are ranked and their most desirable is selected.

\subsection{An illustrative example}

In an attempt to increase incomes, a commercial company in China plans to
make the financial strategy and investment for the coming year. Based
on the company's strategic plan, they have the following six alternative
schemes after the strict screening: (1) $A_{1}$: to invest in
\textquotedblleft Hunan Province"; (2) $A_{2}$: to invest in
\textquotedblleft Guangdong Province"; (3) $A_{3}$: to invest in
\textquotedblleft Hubei Province"; (4) $A_{4}$: to invest in
\textquotedblleft Sichuan Province"; (5) $A_{5}$: to invest in
\textquotedblleft Hebei Province"; (6) $A_{6}$: to invest in
\textquotedblleft Anhui Province". The six alternatives are judged by four attributes: (1) the \textquotedblleft growth
analysis" denoted by $G_{1}$; (2) the \textquotedblleft risk analysis" denoted by $G_{2}$; (3) the \textquotedblleft social-political impact analysis", denoted by $G_{3}$; (4)
the \textquotedblleft environmental impact analysis", denoted by $G_{4}$, where the weight
vector is $\omega =(0.2,0.3,0.1,0.4)^{\top }$.

Step~1: By normalization, the normalized decision
matrix $R=(r_{ij})_{6\times 4}$ transformed by the decision matrix from
the expert is listed in Table~\ref{Tab-4}.
\begin{table}[H]
\caption{The normalized picture fuzzy decision matrix $R$}
\label{Tab-4}\centering
\scalebox{0.75}{
	\begin{tabular}{cccccc}
		\toprule
		$$ & $G_{1}$ &  $G_{2}$ & $G_{3}$ & $G_{4}$ \\
		\midrule
		$A_{1}$ & $\langle0.6, 0.1, 0.2\rangle$ & $\langle0.5, 0.3, 0.1\rangle$ &
           $\langle0.5, 0.1, 0.3\rangle$ & $\langle 0.2, 0.3, 0.4\rangle$\\
		$A_{2}$ & $\langle0.4, 0.4, 0.1\rangle$ & $\langle0.6, 0.3, 0.1\rangle$ &
           $\langle0.5, 0.2, 0.2\rangle$ & $\langle0.7, 0.1, 0.2\rangle$ \\
        $A_{3}$ & $\langle0.2, 0.2, 0.3\rangle$ & $\langle0.6, 0.2, 0.1\rangle$ &
           $\langle0.4, 0.1, 0.3\rangle$ & $\langle0.4, 0.3, 0.3\rangle$ \\
        $A_{4}$ & $\langle0.5, 0.1, 0.2\rangle$ & $\langle0.4, 0.2, 0.1\rangle$ &
           $\langle0.2, 0.2, 0.5\rangle$ & $\langle0.3, 0.2, 0.2\rangle$ \\
        $A_{5}$ & $\langle0.2, 0.2, 0.2\rangle$ & $\langle0.5, 0.2, 0.1\rangle$ &
           $\langle0.3, 0.2, 0.3\rangle$ & $\langle0.6, 0.2, 0.1\rangle$ \\
        $A_{6}$ & $\langle0.6, 0.1, 0.3\rangle$ & $\langle0.1, 0.2, 0.6\rangle$ &
           $\langle0.1, 0.3, 0.5\rangle$ & $\langle0.2, 0.3, 0.2\rangle$ \\
        \bottomrule
	\end{tabular}
      }
\end{table}

Step~2: If we take $T=T_{\textbf{P}}$, by the PFIWA operator $\mathrm{%
PFIWA}_{T_{\textbf{P}},\omega }$, the overall rating values of
each alternative $A_{i}$ are as below:
\begin{small}
\begin{align*}
r_{1}& =\big \langle1-0.4^{0.2}\cdot 0.5^{0.3}\cdot 0.5^{0.1}\cdot 0.8^{0.4},
\\
& 0.3^{0.2}\cdot 0.4^{0.3}\cdot 0.4^{0.1}\cdot
0.7^{0.4}-0.2^{0.2}\cdot 0.1^{0.3}\cdot 0.3^{0.1}\cdot 0.4^{0.4}, \\
& \quad \quad 0.2^{0.2}\cdot 0.1^{0.3}\cdot 0.3^{0.1}\cdot 0.4^{0.4}\big
\rangle \\
& =\langle 0.4229,0.2492,0.2232\rangle ,
\end{align*}%
\end{small}
similarly, $r_{2}=\langle 0.6046, 0.2314, 0.1414\rangle$, $r_{3}=\langle
0.4373, 0.2355, 0.2158\rangle$, $r_{4}=\langle 0.3667, 0.1883, 0.1780\rangle$,
$r_{5}=\langle 0.4804, 0.2062, 0.1282\rangle$, $r_{6}=\langle
0.2700, 0.2466, 0.3305\rangle $.

Step~3: By the direct calculation, the score values of the aggregated
values $r_{1},r_{2},\ldots ,r_{6}$ are $S(r_{1})=0.1997$, $S(r_{2})=0.4632$,
$S(r_{3})=0.2215$, $S(r_{4})=0.1887$, $S(r_{5})=0.3522$, and $%
S(r_{6})=-0.0605$. Thus, we have $r_{2}\succ _{_{\mathrm{W}}}r_{5}\succ _{_{%
\mathrm{W}}}r_{3}\succ _{_{\mathrm{W}}}r_{1}\succ _{_{\mathrm{W}}}r_{4}\succ
_{_{\mathrm{W}}}r_{6}$ by Definition~\ref{Order-Wu}. It means that the best
financial strategy is investing in the Eastern Asian market; i.e., $A_{2}$.

To study the changing tendency of the scores and the rankings for $A_1$,
$A_2$, $A_3$, $A_4$, $A_5$, and  $A_6$ with the variation of the t-norm $T$ and the parameter $\gamma $,
we use the following figures for illustrating these subjects.
\begin{figure}
\centering
\subfigure[Scores for alternatives derived by $\mathrm{PFIWA}_{T_{\protect%
\gamma}^{\textbf{SS}}, \protect\omega}$]{\includegraphics[height=4.1cm,width=4.35cm]{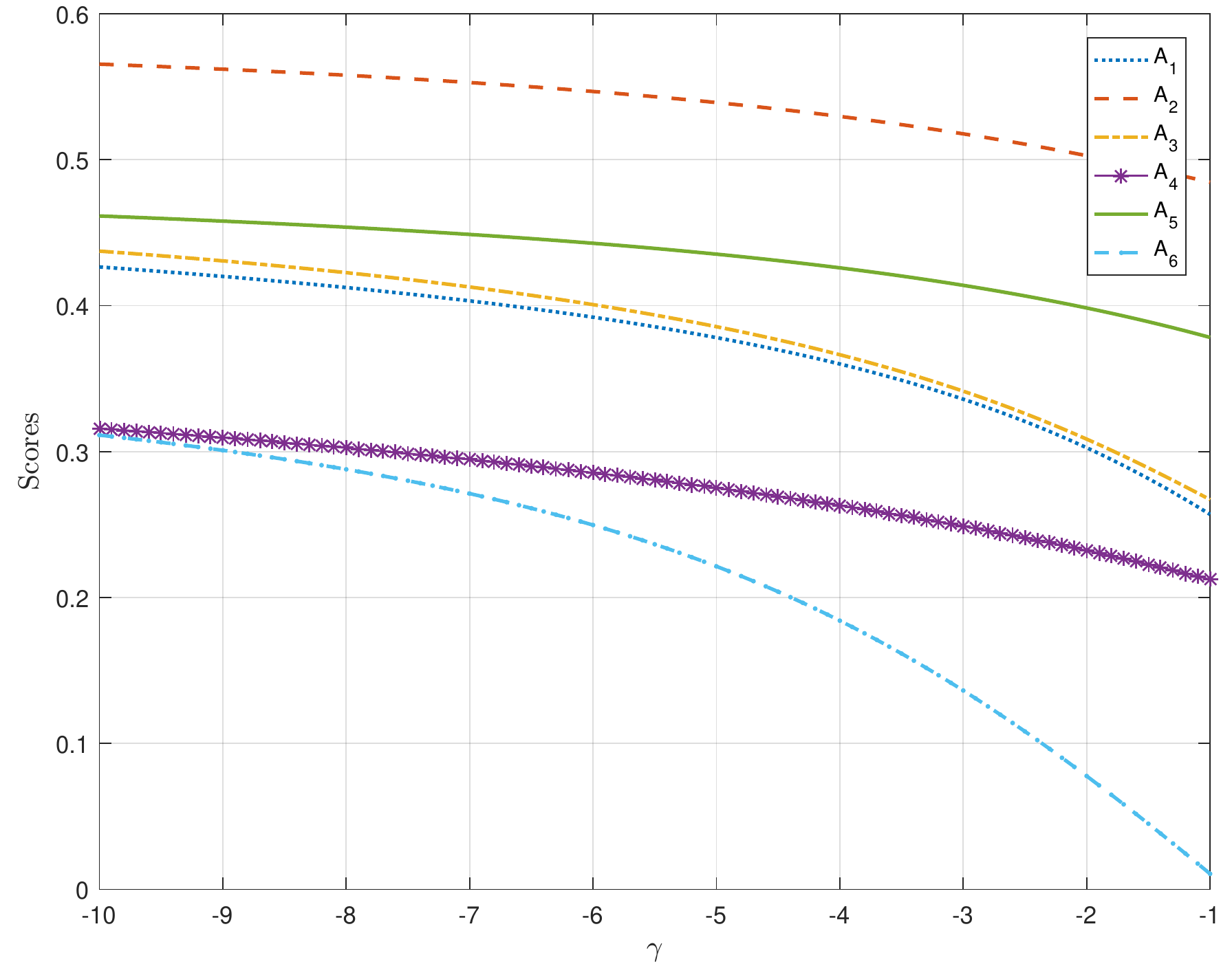}}
\subfigure[Scores for alternatives derived by $\mathrm{PFIWA}_{T_{\protect%
\gamma}^{\textbf{H}}, \protect\omega}$]{\includegraphics[height=4.1cm,width=4.35cm]{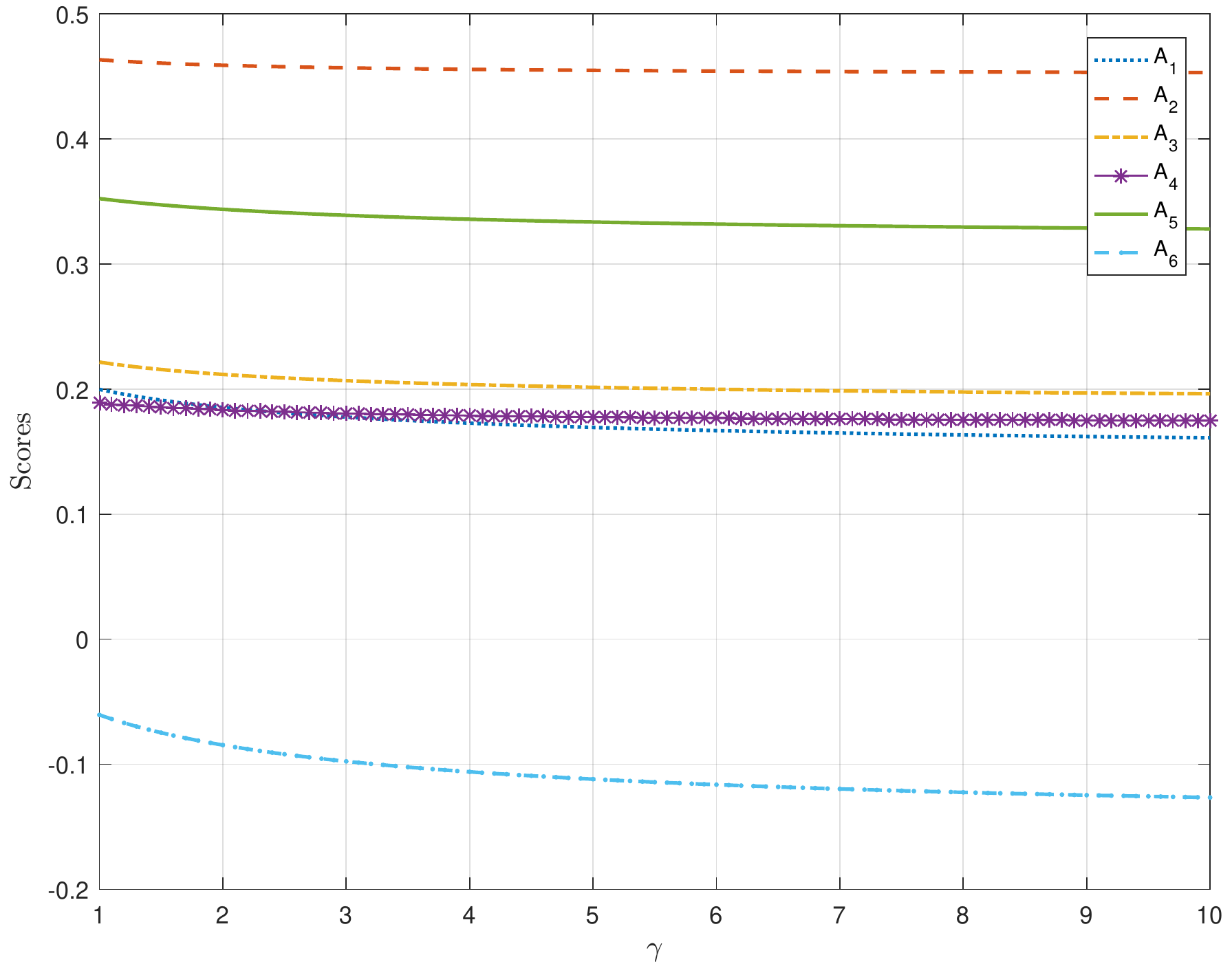}}
\subfigure[Scores for alternatives derived by $\mathrm{PFIWG}_{T_{\protect%
\gamma}^{\textbf{F}}, \protect\omega}$]{\includegraphics[height=4.1cm,width=4.35cm]{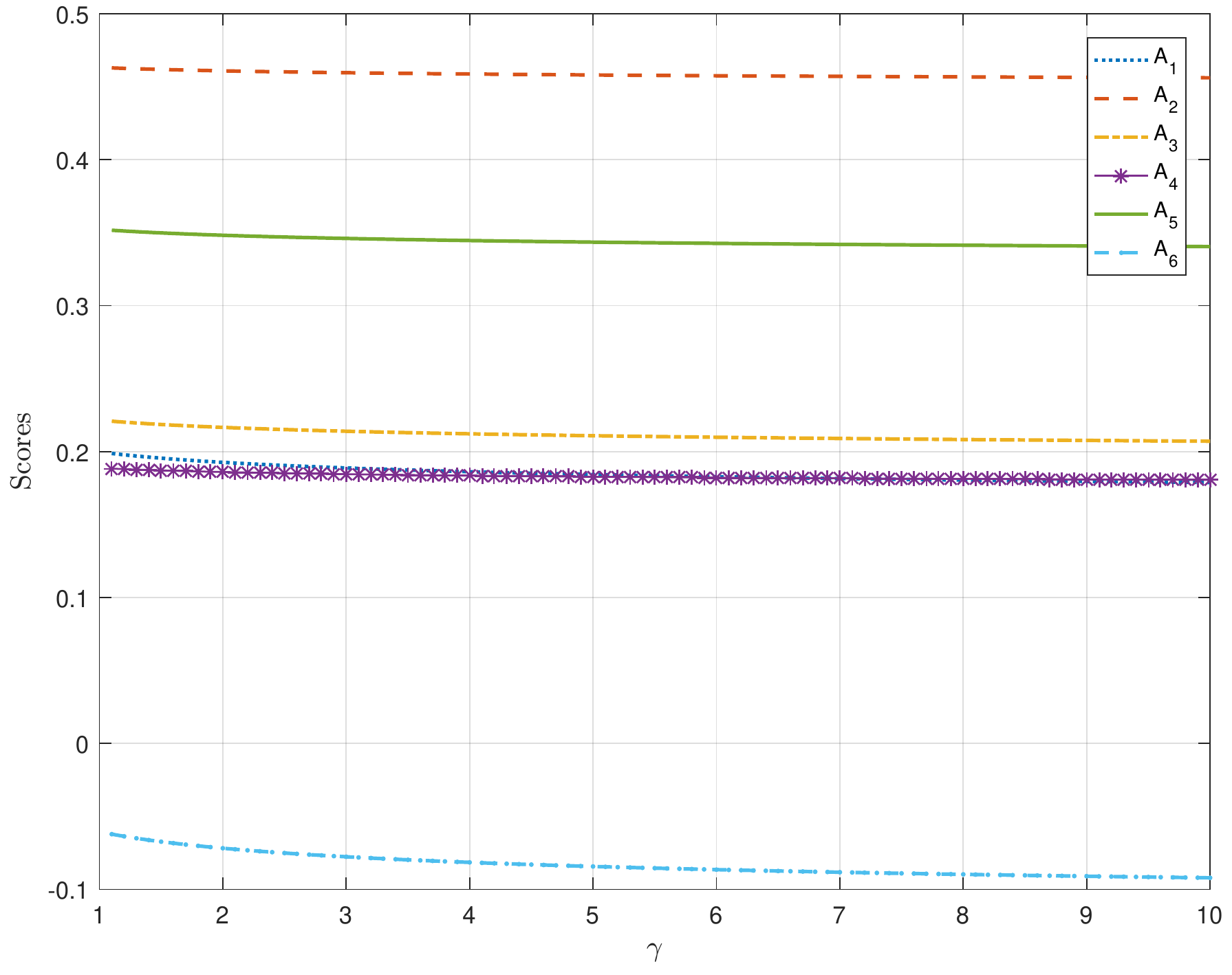}}
\subfigure[Scores for alternatives derived by $\mathrm{PFIWG}_{T_{\protect%
\gamma}^{\textbf{D}}, \protect\omega}$]{\includegraphics[height=4.1cm,width=4.35cm]{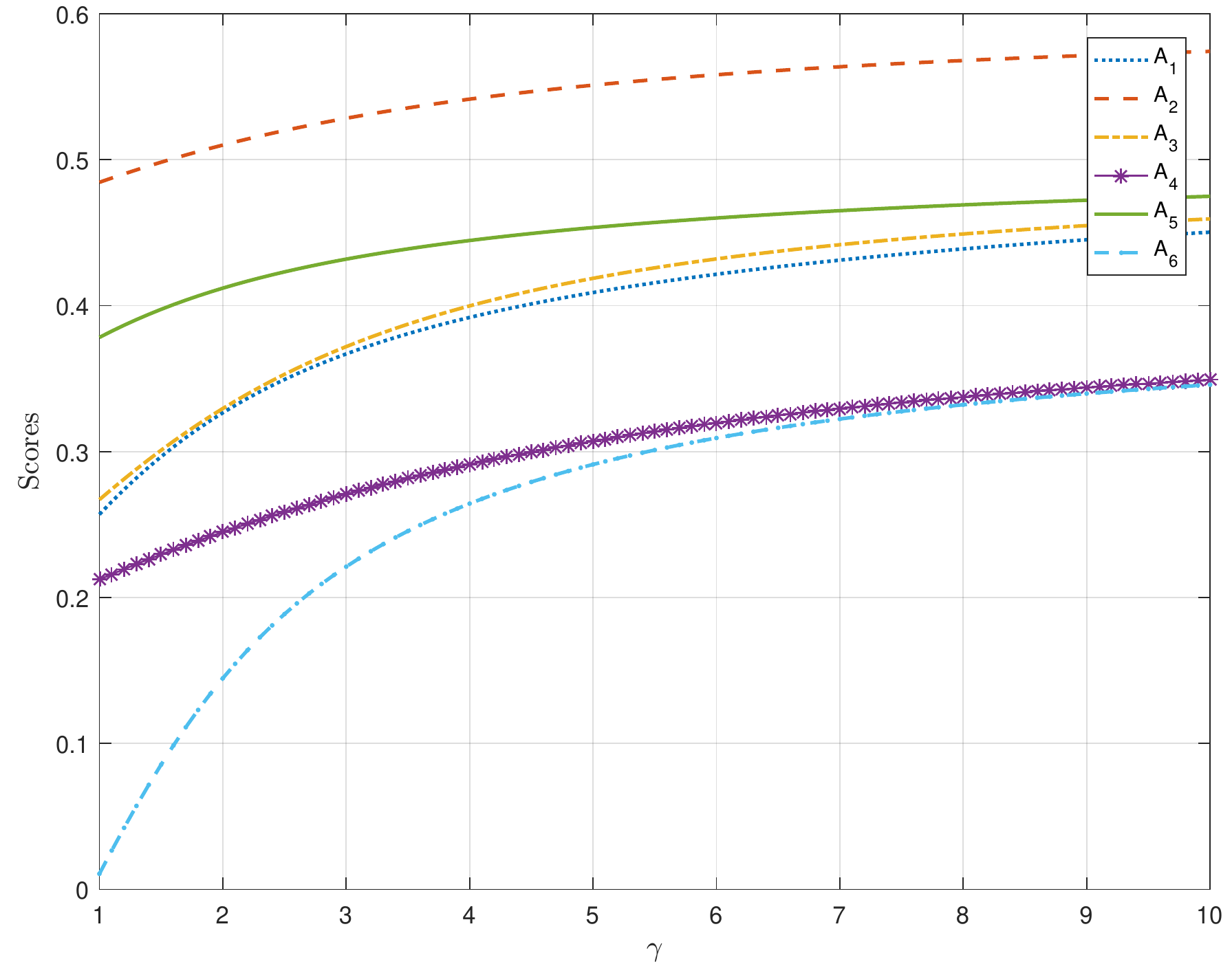}}
\subfigure[Scores for alternatives derived by $\mathrm{PFIWA}_{T_{\protect%
\gamma}^{\textbf{AA}}, \protect\omega}$]{\includegraphics[height=4.1cm,width=4.35cm]{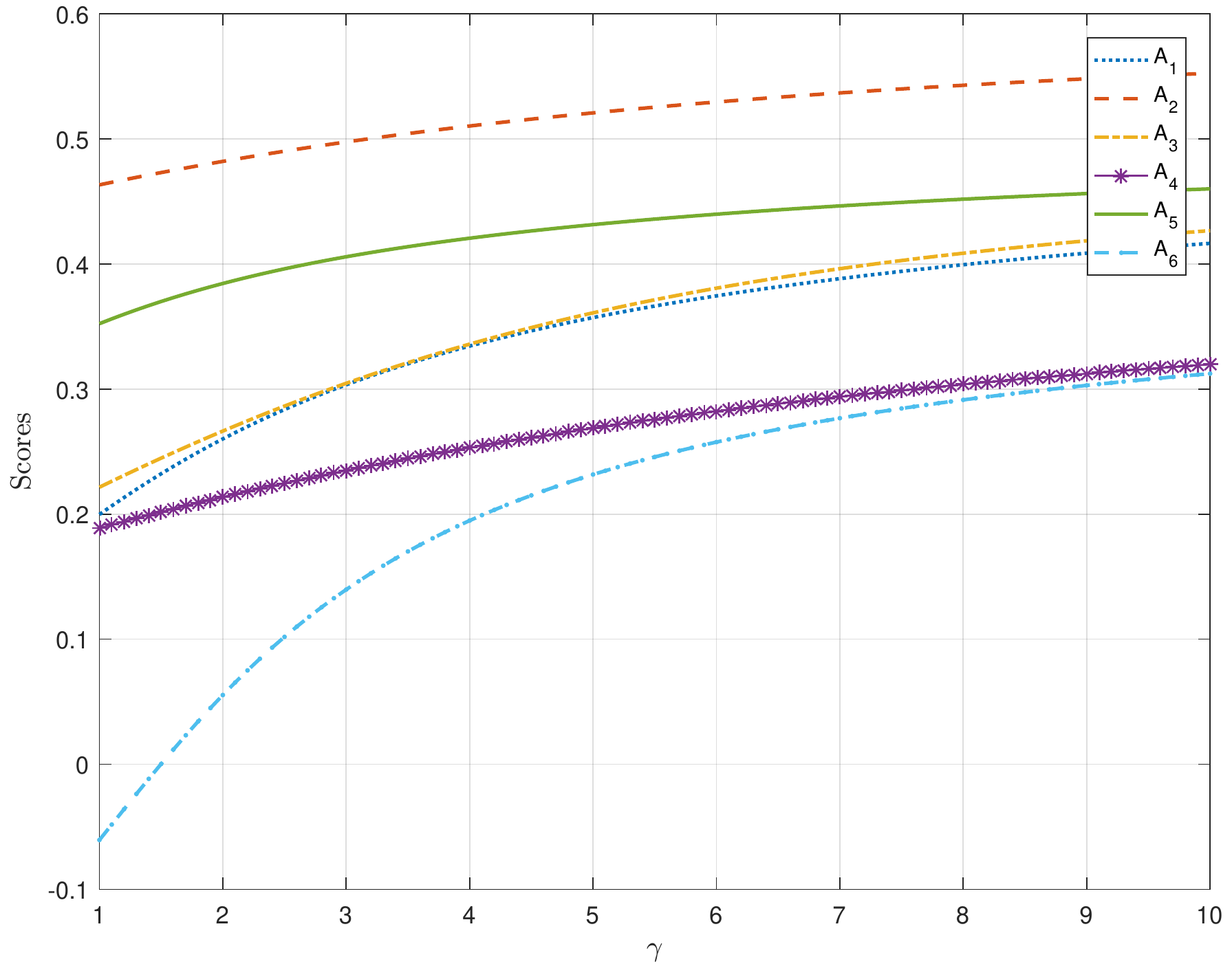}}
\caption{Scores for alternatives}
\label{FIG:2}
\end{figure}

Fig.~\ref{FIG:2} (a) draws the scores for the alternatives derived by the
operator $\mathrm{PFIWA}_{T_{\gamma }^{\textbf{SS}},\omega }$, induced by
Schweizer-Sklar t-norm $T_{\gamma }^{\textbf{SS}},$ as the parameter $\gamma
$ changes from $-10$ to $-1$. Then, we observe that (1) the score functions
of all alternatives decrease monotonously; (2)
$A_{2}$ is the best financial strategy.

Fig.~\ref{FIG:2} (b) draws the scores for the alternatives derived by the
operator $\mathrm{PFIWA}_{T_{\gamma }^{\textbf{H}},\omega }$, induced by
Hamacher t-norm $T_{\gamma }^{\textbf{H}},$ as the parameter $\gamma $
changes from $1$ to $10$. Then, we observe that (1) the score functions
of all alternatives decrease monotonously; (2)
$A_{2}$ is the best financial strategy.

Fig.~\ref{FIG:2} (c) draws the scores for the alternatives derived by the
operator $\mathrm{PFIWA}_{T_{\gamma }^{\textbf{F}},\omega }$, induced by
Frank t-norm $T_{\gamma }^{\textbf{F}},$ as the parameter $\gamma $
changes from $1$ to $10$. Then, we observe that (1) the score functions
of all alternatives decrease monotonously; (2)
$A_{2}$ is the best financial strategy.

Fig.~\ref{FIG:2} (d) draws the scores for the alternatives derived by the
operator $\mathrm{PFIWG}_{T_{\gamma }^{\textbf{D}},\omega }$, induced by
Dombi t-norm $T_{\gamma }^{\textbf{D}},$ as the parameter $\gamma $
changes from $1$ to $10$. Then, we observe that (1) the score functions
of all alternatives increase monotonously; (2)
$A_{2}$ is the best financial strategy.

Fig.~\ref{FIG:2} (e) draws the scores for the alternatives derived by the
operator $\mathrm{PFIWA}_{T_{\gamma }^{\textbf{AA}},\omega }$, induced by Acz%
\'{e}l-Alsina t-norm $T_{\gamma }^{\textbf{AA}},$ as the parameter $\gamma $
changes from $1$ to $10$. Then, we observe that (1) the score functions
of all alternatives increase monotonously; (2)
$A_{2}$ is the best financial strategy.

In general, from the above analysis, the parameter $\gamma $ can be regarded
as a reflection of the decision makers' preferences as the parameter $\gamma
$ changes in a certain range; however, the scores and the rankings of the
alternatives are different. Then, the best financial strategy is always $%
A_{2}$. Therefore, we have sufficient reasons to conclude that this company's best
financial strategy is investing in \textquotedblleft {\bf Guangdong Province}".

\section{Conclusion}
In summary, we introduced an total order on $\mathscr{P},$ denoted by $%
\preceq _{_{\mathrm{W}}},$ by using the score function, and the first
and the second accuracy functions. We also showed that $\preceq _{_{%
\mathrm{W}}}$ is an admissible order, with which $\mathscr{P}$ has the
structure of a complete lattice. Moreover, we discussed the operational laws
in \cite{AMAK2019,AA2020-1,Ga2017,JSPY2019,KAA2019,LLCFL2021,
WWGW2019,Wei2017,Wei2018a,XSWWZLX2019}. We demonstrated that these
operational laws might violate the closedness property
on $\mathscr{P}$, which implies that that decision procedure might not
be convergent. To clear up the problems of these
operational laws, we proposed new interactive operational laws for PFNs
via a strict t-norm. If assigning specific forms of t-norms to our
laws, some attracted operators were obtained. It should be pointed out that
our picture fuzzy operators satisfy the closedness property in comparison
with the above-mentioned ones. We also described the main characteristics of
these operators. Through the instrument of new operators, we defined the
picture fuzzy interactional weighted average (PFIWA) operator and the
picture fuzzy interactional weighted geometric (PFIWG) operator. Both of
them get possession of monotonicity, idempotency, boundedness,
shift-invariance, and homogeneity. Furthermore, by applying PFIWA and PFIWG
operators, a novel MCDM method was developed with the picture fuzzy
environment. We finally provided a practical example of an
investment plan that illustrates the efficiency of the developed method, and
the recent results are conceivable and creditable. Therefore, the research
about this new method may be challenging and meaningful.

\section*{Appendix}
\subsection{Proof of Theorem~\ref{Total-Order-Thm}:}
\begin{proof}
Now, we show that, for any given two PFNs $
\alpha =\langle \mu _{\alpha },\eta _{\alpha },\nu _{\alpha }\rangle $ and $%
\beta =\langle \mu _{\beta },\eta _{\beta },\nu _{\beta }\rangle $, if we have the equations: $%
S(\alpha )-S(\beta )=0$, $H_{1}(\alpha )-H_{1}(\beta )=0$, and $H_{2}(\alpha
)-H_{2}(\beta )=0$, then we must have $\alpha =\beta $. Namely,
if the following equalities hold,
\begin{equation*}
\begin{cases}
\mu _{\alpha }-\nu _{\alpha }=\mu _{\beta }-\nu _{\beta }, \\
\mu _{\alpha }+\nu _{\alpha }=\mu _{\beta }+\nu _{\beta }, \\
\mu _{\alpha }+\eta _{\alpha }+\nu _{\alpha }=\mu _{\beta }+\eta _{\beta
}+\nu _{\beta },%
\end{cases}
\label{eq-Wu-2.1}
\end{equation*}%
i.e.,
\begin{equation*}
\begin{bmatrix}
1 & 0 & -1 \\
1 & 0 & 1 \\
1 & 1 & 1%
\end{bmatrix}%
\begin{bmatrix}
\mu _{\alpha }-\mu _{\beta } \\
\eta _{\alpha }-\eta _{\beta } \\
\nu _{\alpha }-\nu _{\beta }%
\end{bmatrix}%
=%
\begin{bmatrix}
0 \\
0 \\
0%
\end{bmatrix}%
,
\end{equation*}%
then $\mu _{\alpha }=\mu _{\beta }$, $\eta _{\alpha }=\eta _{\beta }$, and $%
\nu _{\alpha }=\nu _{\beta }$. This holds trivially since $%
\begin{vmatrix}
1 & 0 & -1 \\
1 & 0 & 1 \\
1 & 1 & 1%
\end{vmatrix}%
=-2$.

This, together with the fact that
$\preceq _{_{\mathrm{W}}}$ is a partial order on $\mathscr{P}$, implies the conclusion.
\end{proof}

\subsection{Proof of Theorem~\ref{Adm-Order-Thm}:}
\begin{proof}
It follows from Theorem \ref{Total-Order-Thm} that $\preceq_{_{\mathrm{W}}}$ is a total
order on $\mathscr{P}$. For two PFNs $\alpha =\langle \mu_{\alpha},
\eta_{\alpha}, \nu_{\alpha}\rangle $ and $\beta =\langle \mu_{\beta},
\eta_{\beta}, \nu_{\beta}\rangle ,$ let $\alpha \subseteq
\beta$. By Definition~\ref{sub-order}, there holds $\mu_{\alpha}-
\mu_{\beta}\leq0$, $\eta_{\alpha}-\eta_{\beta}\leq0$, and $\nu_{\alpha}-
\nu_{\beta}\geq0$. Then, we have
$S(\alpha )-S(\beta )
=(\mu_{\alpha}-\nu_{\alpha})-(\mu_{\beta}-\nu_{\beta})
=(\mu_{\alpha}-\mu_{\beta})+(\nu_{\beta}-\nu_{\alpha})\leq 0$.

Case 1. For $S(\alpha )-S(\beta )<0$, we have $S(\alpha )<S(\beta)$.
From Definition \ref{Order-Wu}, it follows that $\alpha \prec_{_{\mathrm{W}}}\beta$.

Case 2. For $S(\alpha )-S(\beta )=0$, we have
$(\mu_{\alpha}-\mu_{\beta})+(\nu_{\beta}-\nu_{\alpha})=0$.
This, together with $\mu_{\alpha}\leq \mu_{\beta}$ and $\nu_{\alpha}\geq
\nu_{\beta}$, implies that $\mu_{\alpha}=\mu_{\beta}$ and
$\nu_{\alpha}=\nu_{\beta}$, and thus $H_{1}(\alpha)=\mu_{\alpha}+\nu_{\alpha}
=\mu_{\beta}+\nu_{\beta}=H_{1}(\beta)$. By $\eta_{\alpha}\leq \eta_{\beta}$,
we have $H_{2}(\alpha)=\mu_{\alpha}+\eta_{\alpha}+\nu_{\alpha}
\leq \mu_{\beta}+\eta_{\beta}+\nu_{\beta}=H_{2}(\beta)$. Therefore,
by Definition~\ref{Order-Wu}, we have $\alpha \preceq_{_{\mathrm{W}}}\beta$.
\end{proof}

\subsection{Proof of Theorem~\ref{Pro-Them}:}
\begin{proof}
Let $S$ be the dual t-conorm of $T$. The statements (i)--(ii) follow
directly from the commutativity of $T$ and $S$.

(iii)
$(\alpha \oplus _{_{T}}\beta )\oplus _{_{T}}\gamma
=\big\langle S(\mu _{\alpha },\mu _{\beta }),T(\eta _{\alpha }+\nu
_{\alpha },\eta _{\beta }+\nu _{\beta })-T(\nu _{\alpha },\nu _{\beta
}),T(\nu _{\alpha },\nu _{\beta })\big\rangle\oplus _{_{T}}\gamma
=\big\langle S(S(\mu _{\alpha },\mu _{\beta }),\mu _{\gamma
}),T(T(\eta _{\alpha }+\nu _{\alpha },\eta _{\beta }+\nu _{\beta }),\eta
_{\gamma }+\nu _{\gamma })-T( T(\nu _{\alpha },\nu _{\beta }),\nu
_{\gamma }) ,T(T(\nu _{\alpha },\nu _{\beta }),\nu _{\gamma })
\big\rangle =\big\langle S(\mu _{\alpha },S(\mu _{\beta },\mu _{\gamma
})),T(\eta _{\alpha }+\nu _{\alpha },T(\eta _{\beta }+\nu _{\beta },\eta
_{\gamma }+\nu _{\gamma }))-T(\nu _{\alpha },T(\nu _{\beta },\nu _{\gamma
}),T(\nu _{\alpha },T(\nu _{\beta },\nu _{\gamma }))\big\rangle
=\alpha \oplus _{_{T}}\big\langle S(\mu _{\beta },\mu _{\gamma }),T(\eta
_{\beta }+\nu _{\beta },\eta _{\gamma }+\nu _{\gamma })-T(\nu _{\beta },\nu
_{\gamma }),T(\nu _{\beta },\nu _{\gamma })\big\rangle
 =\alpha \oplus _{_{T}}(\beta \oplus _{_{T}}\gamma ).
$

(v) Notice that $T(x,y)=\tau ^{-1}(\tau (x)+\tau (y))$ and $S(x,y)=\zeta
^{-1}(\zeta (x)+\zeta (y))$. Then, we get that
$(\xi_{_T}\alpha)\oplus_{_T}(\lambda_{_T} \alpha)
=\big\langle \zeta^{-1}(\xi\cdot \zeta(\mu_{\alpha})), \tau^{-1}(\xi\cdot
\tau(\eta_{\alpha}+\nu_{\alpha})) -\tau^{-1}(\xi\cdot \tau(\nu_{\alpha})),
\tau^{-1}(\xi\cdot \tau(\nu_{\alpha}))\big\rangle
 \oplus_{_{T}} \big\langle \zeta^{-1}(\lambda\cdot
\zeta(\mu_{\alpha})), \tau^{-1}(\lambda \cdot
\tau(\eta_{\alpha}+\nu_{\alpha})) -\tau^{-1}(\lambda\cdot
\tau(\nu_{\alpha})), \tau^{-1}(\lambda\cdot \tau(\nu_{\alpha}))\big\rangle
=\big\langle S(\zeta^{-1}(\xi\cdot \zeta(\mu_{\alpha})),
\zeta^{-1}(\lambda\cdot \zeta(\mu_{\alpha}))), T(\tau^{-1}(\xi\cdot
\tau(\eta_{\alpha}+\nu_{\alpha})), \tau^{-1}(\lambda\cdot
\tau(\eta_{\alpha}+\nu_{\alpha})))-T(\tau^{-1}(\xi\cdot \tau(\nu_{\alpha})),
\tau^{-1}(\lambda\cdot \tau(\nu_{\alpha}))), T(\tau^{-1}(\xi\cdot
\tau(\nu_{\alpha})), \tau^{-1}(\lambda\cdot \tau(\nu_{\alpha})))\big\rangle
=\big\langle\zeta^{-1}((\xi+\lambda)\cdot \zeta(\mu_{\alpha})),
\tau^{-1}((\xi+\lambda)\cdot \tau(\eta_{\alpha}+\nu_{\alpha}))
-\tau^{-1}((\xi+\lambda)\cdot \tau(\nu_{\alpha})),
\tau^{-1}((\xi+\lambda)\cdot \tau(\nu_{\alpha}))\big\rangle
=(\xi+\lambda)_{_T}\alpha.$

(vii)
$(\lambda_{_T}\alpha)\oplus_{_T}(\lambda_{_T}\beta)
=\big\langle \zeta^{-1}(\lambda\cdot \zeta(\mu_{\alpha})),
\tau^{-1}(\lambda\cdot \tau(\eta_{\alpha}+\nu_{\alpha}))
-\tau^{-1}(\lambda\cdot \tau(\nu_{\alpha})), \tau^{-1}(\lambda\cdot
\tau(\nu_{\alpha}))\big\rangle \oplus_{_{T}} \big\langle \zeta^{-1}(\lambda\cdot
\zeta(\mu_{\beta})), \tau^{-1}(\lambda\cdot \tau(\eta_{\beta}+\nu_{\beta}))
-\tau^{-1}(\lambda\cdot \tau(\nu_{\beta})), \tau^{-1}(\lambda\cdot
\tau(\nu_{\beta}))\big\rangle
=\big\langle S(\zeta^{-1}(\lambda\cdot \zeta(\mu_{\alpha})),
\zeta^{-1}(\lambda\cdot \zeta(\mu_{\beta}))), T(\tau^{-1}(\lambda\cdot
\tau(\eta_{\alpha}+\nu_{\alpha})), \tau^{-1}(\lambda\cdot
\tau(\eta_{\beta}+\nu_{\beta})))
-T(\tau^{-1}(\lambda\cdot \tau(\nu_{\alpha})),
\tau^{-1}(\lambda\cdot \tau(\nu_{\beta}))), T(\tau^{-1}(\lambda\cdot
\tau(\nu_{\alpha})), \tau^{-1}(\lambda\cdot \tau(\nu_{\beta})))\big\rangle
=\big\langle\zeta^{-1}(\lambda\cdot
(\zeta(\mu_{\alpha})+\zeta(\mu_{\beta}))), \tau^{-1}(\lambda\cdot
(\tau(\eta_{\alpha}+\nu_{\alpha})+\tau(\eta_{\beta}+\nu_{\beta})))
-\tau^{-1}(\lambda\cdot
(\tau(\nu_{\alpha})+\tau(\nu_{\beta}))), \tau^{-1}(\lambda \cdot
(\tau(\nu_{\alpha})+\tau(\nu_{\beta}))) \big\rangle
=\big\langle\zeta^{-1}(\lambda\cdot \zeta\circ \zeta^{-1}
(\zeta(\mu_{\alpha})+\zeta(\mu_{\beta}))), \tau^{-1}(\lambda\cdot \tau \circ
\tau^{-1} (\tau(\eta_{\alpha}+\nu_{\alpha})+\tau(\eta_{\beta}+\nu_{\beta})))
-\tau^{-1}(\lambda\cdot \tau\circ \tau^{-1}
(\tau(\nu_{\alpha})+\tau(\nu_{\beta}))), \tau^{-1}(\lambda \cdot \tau\circ
\tau^{-1} (\tau(\nu_{\alpha})+\tau(\nu_{\beta}))) \big\rangle
=\big\langle\zeta^{-1}(\lambda\cdot S(\mu_{\alpha}, \mu_{\beta})),
\tau^{-1}(\lambda\cdot T(\eta_{\alpha}+\nu_{\alpha},
\eta_{\beta}+\nu_{\beta})) -\tau^{-1}(\lambda\cdot T(\nu_{\alpha},
\nu_{\beta})), \tau^{-1}(\lambda\cdot T(\nu_{\alpha}, \nu_{\beta})) %
\big\rangle =\lambda_{_T}(\alpha\oplus_{_T} \beta)$.

(ix)
$\xi _{_{T}}(\lambda _{_{T}}\alpha )
=\xi _{_{T}}\cdot \big\langle\zeta ^{-1}(\lambda \zeta (\mu _{\alpha
})),\tau ^{-1}(\lambda \tau (\eta _{\alpha }+\nu _{\alpha }))-\tau
^{-1}(\lambda \tau (\nu _{\alpha })),\tau ^{-1}(\lambda \tau (\nu _{\alpha
}))\big\rangle=\big\langle\zeta ^{-1}((\xi \cdot \lambda )\cdot \zeta (\mu
_{\alpha })),\tau ^{-1}((\xi \cdot \lambda )\cdot \tau (\eta _{\alpha }+\nu
_{\alpha }))-\tau ^{-1}((\xi \cdot \lambda )\cdot \tau (\nu _{\alpha
})),\tau ^{-1}((\xi \cdot \lambda )\cdot \tau (\nu _{\alpha }))\big\rangle
=(\lambda \cdot \xi )_{_{T}}\alpha$.

By using similar arguments to the proofs of (iii), (v), (vii), and (ix), it
is verified the statements (iv), (vi), (viii), and (x).
\end{proof}

\subsection{Proof of Theorem~\ref{N-Thm}:}
\begin{proof}We only prove the formula~\eqref{eq-Wu-1.1} in view of the mathematical
induction on $n$. The formula~\eqref{eq-Wu-1.2} is proved analogously.

(1) When $n=2$, it is easy to check that $\alpha _{1}\oplus _{_{T}}\alpha
_{2}=\langle S(\mu _{1},\mu _{2}),T(\eta _{1}+\nu _{1},\eta _{2}+\nu
_{2})-T(\nu _{1},\nu _{2}),T(\nu _{1},\nu _{2})\rangle $.

(2) Suppose that the formula~\eqref{eq-Wu-1.1} holds when $n=k$; i.e.,
\begin{align*}
&\quad \alpha_1\oplus_{_{T}}\alpha_2\oplus_{_{T}}\cdots \oplus_{_{T}}\alpha_k
\\
&=\left\langle S^{(k)}(\mu_1, \mu_2, \ldots, \mu_k),\right. \\
&\quad T^{(k)}(\eta_1+\nu_1, \eta_2+\nu_2, \ldots, \eta_k+\nu_k)
-T^{(k)}(\nu_1, \nu_2, \ldots, \nu_k), \\
&\quad \left.T^{(k)}(\nu_1, \nu_2, \ldots, \nu_k)\right\rangle,
\end{align*}
When $n=k+1$, by Definition~\ref{Def-Wu} and Theorem~\ref{Pro-Them} (iii),
we get
\begin{align*}
&\quad \alpha_1\oplus_{_{T}}\alpha_2\oplus_{_{T}}\cdots
\oplus_{_{T}}\alpha_{k+1}=\left(\alpha_1\oplus_{_{T}}
\cdots \oplus_{_{T}}\alpha_{k}\right)\oplus_{_{T}}
\alpha_{k+1} \\
&=\left\langle S^{(k)}(\mu_1, \ldots, \mu_k),\right. \\
&\quad \quad T^{(k)}(\eta_1+\nu_1, \ldots, \eta_k+\nu_k)
-T^{(k)}(\nu_1, \ldots, \nu_k), \\
&\quad \quad \left.T^{(k)}(\nu_1, \ldots, \nu_k)\right\rangle \oplus_{_{T}} %
\left\langle\mu_{k+1}, \eta_{k+1}, \nu_{k+1}\right\rangle \\
&=\left\langle S(S^{(k)}(\mu_1, \ldots, \mu_k), \mu_{k+1}),\right. \\
& \quad \quad T(T^{(k)}(\eta_1+\nu_1, \ldots, \eta_k+\nu_k),
\eta_{k+1}+\nu_{k+1})\\
&\quad \quad  -T(T^{(k)}(\nu_1, \ldots, \nu_k), \nu_{k+1}), \\
&\quad \quad \left.T(T^{(k)}(\nu_1, \ldots, \nu_k), \nu_{k+1})\right\rangle,\\
&=\left\langle S^{(k+1)}(\mu_1, \ldots, \mu_{k+1}), \right.\\
&\quad \quad T^{(k+1)}(\eta_1+\nu_1, \ldots,
\eta_{k+1}+\nu_{k+1})\\
&\quad \quad  -T^{(k+1)}(\nu_1, \ldots, \nu_{k+1}), \\
&\quad\quad \left.T^{(k+1)}(\nu_1, \ldots, \nu_{k+1})\right\rangle,
\end{align*}
and thus formula~\eqref{eq-Wu-1.1} holds when $n=k+1$.

Therefore, by (1) and (2), we obtain that formula~\eqref{eq-Wu-1.1} holds for
all $n$.
\end{proof}

\subsection{Proof of Theorem~\ref{PFWA-PFGA-Thm}:}
\begin{proof}
By the formulas~\eqref{eq-PFWA-Ope} and \eqref{eq-PFWG-Ope}, it follows from
Theorem~\ref{N-Thm} and Definition~\ref{Def-Wu} that
\begin{align*}
& \quad \mathrm{PFIWA}_{T,\omega }(\alpha _{1},\ldots ,\alpha
_{n}) \\
& =\left\langle \zeta ^{-1}(\omega _{1}\cdot \zeta (\mu _{1})),\tau
^{-1}(\omega _{1}\cdot \tau (\eta _{1}+\nu _{1}))-\tau ^{-1}(\omega
_{1}\cdot \tau(\nu _{1})),\right.\\
&\quad \left.\tau^{-1}(\omega _{1}\cdot \tau (\nu
_{1}))\right\rangle \oplus _{_{T}}\cdots \oplus _{_{T}}\left\langle \zeta ^{-1}(\omega
_{n}\cdot \zeta (\mu _{n})),\right.\\
&\quad \left.\tau ^{-1}(\omega _{n}\cdot \tau (\eta _{n}+\nu
_{n}))-\tau ^{-1}(\omega _{n}\cdot \tau (\nu _{n})),\tau ^{-1}(\omega
_{n}\cdot \tau (\nu _{n}))\right\rangle\\
%%\end{align*}
%%\begin{align*}
& =\left\langle S^{(n)}(\zeta ^{-1}(\omega _{1}\cdot \zeta (\mu _{1})),\ldots
,\zeta ^{-1}(\omega _{n}\cdot \zeta (\mu _{n}))),\right. \\
& \quad \quad T^{(n)}(\tau ^{-1}(\omega _{1}\cdot \tau (\eta _{1}+\nu
_{1})),\ldots ,\tau ^{-1}(\omega _{n}\cdot \tau (\eta _{n}+\nu
_{n}))) \\
&\quad \quad \quad \quad -T^{(n)}(\tau ^{-1}(\omega _{1}\cdot \tau (\nu _{1})),\ldots ,\tau
^{-1}(\omega _{n}\cdot \tau (\nu _{n}))),\\
&\quad \quad \left.
T^{(n)}(\tau ^{-1}(\omega _{1}\cdot \tau (\nu _{1})),\ldots
,\tau ^{-1}(\omega _{n}\cdot \tau (\nu _{n})))\right\rangle\\
%%\end{align*}
%%\begin{align*}
& =\left\langle\zeta ^{-1}(\omega _{1}\cdot \zeta (\mu _{1})+\cdots +\omega
_{n}\cdot \zeta (\mu _{n})),\right. \\
& \quad \quad \tau ^{-1}(\omega _{1}\cdot \tau (\eta _{1}+\nu _{1})+\cdots
+\omega _{n}\cdot \tau (\eta _{n}+\nu _{n}))\\
&\quad \quad \quad \quad -\tau ^{-1}(\omega _{1}\cdot
\tau (\nu _{1})+\cdots +\omega _{n}\cdot \tau (\nu _{n})), \\
& \quad \quad \left.\tau ^{-1}(\omega _{1}\cdot \tau (\nu _{1})+\cdots +\omega
_{n}\cdot \tau (\nu _{n}))\right\rangle,
\end{align*}%
and
\begin{align*}
& \quad \mathrm{PFIWG}_{T,\omega }(\alpha _{1},\alpha _{2},\ldots ,\alpha
_{n}) \\
& =\left\langle \tau ^{-1}(\omega _{1}\cdot \tau (\mu _{1})),\tau
^{-1}(\omega _{1}\cdot \tau (\eta _{1}+\mu _{1}))-\tau ^{-1}(\omega
_{1}\cdot \tau (\mu _{1})),\right.\\
&\quad \left.\zeta ^{-1}(\omega _{1}\cdot \zeta (\nu
_{1}))\right\rangle \otimes _{_{T}}\cdots \otimes _{_{T}}\left\langle \tau ^{-1}(\omega
_{n}\cdot \tau (\mu _{n})),\right.\\
&\quad \left.\tau ^{-1}(\omega _{n}\cdot \tau (\eta _{n}+\mu
_{n}))-\tau ^{-1}(\omega _{n}\cdot \tau (\mu _{n})),\zeta ^{-1}(\omega
_{n}\cdot \tau (\nu _{n}))\right\rangle \\
& =\left\langle T^{(n)}(\tau ^{-1}(\omega _{1}\cdot \tau (\mu _{1})),\ldots
,\tau ^{-1}(\omega _{n}\cdot \tau (\mu _{n}))),\right. \\
& \quad \quad T^{(n)}(\tau ^{-1}(\omega _{1}\cdot \tau (\eta _{1}+\mu
_{1})),\ldots ,\tau ^{-1}(\omega _{n}\cdot \tau (\eta _{n}+\mu
_{n})))\\
&\quad \quad \quad \quad -T^{(n)}(\tau ^{-1}(\omega _{1}\cdot \tau (\mu _{1})),\ldots ,\tau
^{-1}(\omega _{n}\cdot \tau (\mu _{n}))), \\
& \quad \quad \left.S^{(n)}(\zeta ^{-1}(\omega _{1}\cdot \zeta (\nu _{1})),\ldots
,\zeta ^{-1}(\omega _{n}\cdot \zeta (\nu _{n})))\right\rangle \\
& =\left\langle\tau ^{-1}(\omega _{1}\cdot \tau (\mu _{1})+\cdots +\omega
_{n}\cdot \tau (\mu _{n})),\right. \\
& \quad \quad \tau ^{-1}(\omega _{1}\cdot \tau (\eta _{1}+\mu _{1})+\cdots
+\omega _{n}\cdot \tau (\eta _{n}+\mu _{n}))\\
&\quad \quad \quad \quad -\tau ^{-1}(\omega _{1}\cdot
\tau (\mu _{1})+\cdots +\omega _{n}\cdot \tau (\mu _{n})), \\
& \quad \quad \left.\zeta ^{-1}(\omega _{1}\cdot \zeta (\nu _{1})+\cdots +\omega
_{n}\cdot \zeta (\nu _{n}))\right\rangle.
\end{align*}
\end{proof}

\end{document}